\documentclass[twoside]{article} 

\usepackage[accepted]{aistats 2020}

\usepackage[round]{natbib}

\renewcommand{\bibsection}

\usepackage[utf8]{inputenc} 
\usepackage[T1]{fontenc}    
\usepackage{hyperref}       
\usepackage{url}            
\usepackage{booktabs}       
\usepackage{amsfonts, amsthm, amsmath}       
\usepackage{amssymb}        
\usepackage{nicefrac}       
\usepackage{microtype}      
\usepackage{enumerate}
\usepackage{subfig}
\usepackage{graphicx}
\usepackage{xcolor}
\usepackage{multirow}
\usepackage{enumitem}
\usepackage{stfloats}
\usepackage{multicol}
\usepackage{titlesec}
\usepackage{pdf14}
\usepackage{listings}
\usepackage{array}

\definecolor{asparagus}{rgb}{0.53, 0.66, 0.42}

\newtheorem{theorem}{Theorem}
\newtheorem{lem}{Lemma}

\newtheorem{defin}{Definition}
\newtheorem*{obs}{Remark}

\usepackage{pgfplots}
\usetikzlibrary{external}
\hypersetup{
  colorlinks = True,
  citecolor=magenta
}

\newcommand{\reusefigure}{} 
 
\newcommand{\includepgfplots}[1]{%
  \ifx\reusefigure\undefined
      \tikzsetnextfilename{#1}%
      \input{./src/#1.tex}%
  \else
      \includegraphics{./img/#1}
  \fi
}

\usepackage{tikz}
\usetikzlibrary{patterns}

\setitemize{itemsep=0.01cm}
\setenumerate{itemsep=0.01cm}
\titlespacing*{\section}{0pt}{0.18cm}{0.1cm}
\titlespacing*{\subsection}{0pt}{0.1cm}{0.07cm}

\pgfplotsset{compat=1.14}

\newcommand{\mytitle}{Beyond exploding and vanishing gradients: analysing RNN training using attractors and smoothness}
\newcommand{\myauthors}{Ant\^{o}nio H. Ribeiro \And Koen Tiels \And  Luis A. Aguirre \And Thomas B. Sch\"on}
\newcommand{\myaddress}{Federal University of\\ Minas Gerais  \And  Eindhoven University\\ of Technology \And Federal University of\\ Minas Gerais \And Uppsala University}

\begin{document}

\setlength{\abovedisplayskip}{2.2pt}
\setlength{\belowdisplayskip}{2.2pt}

\twocolumn[
\vspace*{-0.5cm}
\aistatstitle{\mytitle}
\vspace*{-0.7cm}
\aistatsauthor{\myauthors}
\aistatsaddress{\myaddress}
\vspace*{-0.4cm}
]

\begin{abstract}
\vspace*{-0.4cm}
The exploding and vanishing gradient problem has been the major conceptual principle behind most architecture and training improvements in recurrent neural networks (RNNs) during the last decade.  In this paper, we argue that this principle, while powerful, might need some refinement to explain recent developments. We refine the concept of exploding gradients by reformulating the problem in terms of the cost function smoothness, which gives insight into higher-order derivatives and the existence of regions with many close local minima. We also clarify the distinction between vanishing gradients and the need for the RNN to learn attractors to fully use its expressive power. Through the lens of these refinements, we shed new light on recent developments in the RNN field, namely stable RNN and unitary (or orthogonal) RNNs.

\end{abstract}
\vspace*{-0.4cm}

\begin{figure}[b]
\noindent\rule[0.5ex]{\linewidth}{1pt}

\noindent\textbf{\footnotesize Please cite the conference proceedings paper:}
\begin{lstlisting}
@inproceedings{ribeiro_beyond_2020,
author = {Ribeiro, Ant\^{o}nio H. and Tiels, Koen and Aguirre, Luis A. and Sch\"on, Thomas B. },
title = {Beyond exploding and vanishing  gradients: analysing RNN training using attractors and smoothness},
year = {2020},
booktitle = {Proceedings of the 23rd International Conference on Artificial Intelligence and Statistics (AISTATS).},
publisher = {PMLR},
volume = {108}
}
\end{lstlisting}
\end{figure}

\section{Introduction}

Training recurrent neural networks can be challenging. The problem in training these recurrent models is usually stated in terms of the so-called  \textit{exploding and vanishing gradient problem}~\citep{hochreiter_long_1997, pascanu_difficulty_2013, bengio_learning_1994}. This problem is easy to understand and has motivated many techniques, including the use of gating mechanisms~\citep{hochreiter_long_1997, cho_properties_2014}, gradient clipping~\citep{pascanu_difficulty_2013}, non-saturating activation functions~\citep{chandar_nonsaturating_2019} and the manipulation of the propagation path of gradients~\citep{kanuparthi_hdetach_2019}.

A recently proposed family of methods based on the same principle are the so-called unitary and orthogonal RNNs~\citep{mhammedi_efficient_2017, vorontsov_orthogonality_2017, lezcano-casado_cheap_2019, helfrich_orthogonal_2018, arjovsky_unitary_2016, jing_tunable_2017, maduranga_complex_2019, wisdom_fullcapacity_2016, lezcano-casado_trivializations_2019}. In these architectures, the eigenvalues of the hidden-to-hidden weight matrix are fixed to one to simultaneously avoid exploding gradients (that might appear when the eigenvalues are larger than one) and vanishing gradients (that might appear when they are less than one). Similar, but less strict constraints were proposed by~\citet{zhang_stabilizing_2018, kerg_nonnormal_2019}. 
A different line of study goes in the opposite direction and suggests that using RNNs constrained to be stable can provide as good performance on many tasks as unconstrained RNNs. \citet{miller_stable_2019} discuss examples for which projecting all eigenvalues to values less than one does not affect the performance. In this case, RNNs can be truncated with arbitrarily small error and hence could be replaced by feedforward structures. Indeed, feedforward structures, such as transformer-based architectures~\citep{vaswani_attention_2017} and convolutional networks~\citep{bai_empirical_2018}, have recently matched or outperformed RNNs in many tasks. These feedforward architectures yield impressive results in language and music modeling~\citep{bai_empirical_2018, oord_wavenet_2016, dauphin_language_2017, radford_improving_2018, radford_language_2019}, text-to-speech conversion~\citep{oord_wavenet_2016}, machine translation~\citep{kalchbrenner_neural_2016, gehring_convolutional_2017} and other sequential learning tasks for which RNNs have been the default choice until recently~\citep{bai_empirical_2018}.

Our aim in this paper is to improve the present understanding of recurrent neural networks. We believe the moment is propitious for such analysis, and it might help in explaining this shift of winning paradigm for sequence modeling and the lack of consensus on how to deal with the eigenvalues. On the one hand, we complement the vanishing gradient interpretation with an analysis of the attractors of the RNN for storing long-term information. While the condition for dynamic attractors (other than a single fixed point) to appear is the same as for the gradients to vanish, these two notions are distinct. The numerical examples explore the training mechanism and the attractors of the stable~\citep{miller_stable_2019} and orthogonal RNNs~\citep{lezcano-casado_cheap_2019}. On the other hand, we study the \textit{smoothness} of the cost function, which builds on and improves the notion of exploding gradients, since it also takes into account higher-order derivatives. As a result of our analysis we have that not only ``walls'' might be formed in the cost function (as suggested by~\citet{pascanu_difficulty_2013}), but also regions with very intricate behavior, full of local minima. These regions are thus very hard for the optimization algorithms to navigate (cf.~Figure~\ref{fig:chaotic_LSTM_1d}).

\begin{figure*}[t]
  \vspace{-0.4cm}
\centering
\subfloat[]{\includegraphics[width=0.5\textwidth]{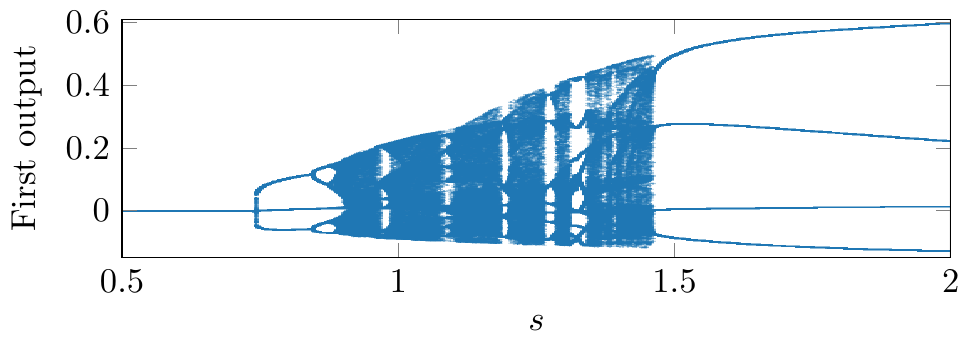}}
\subfloat[]{\includegraphics[width=0.5\textwidth]{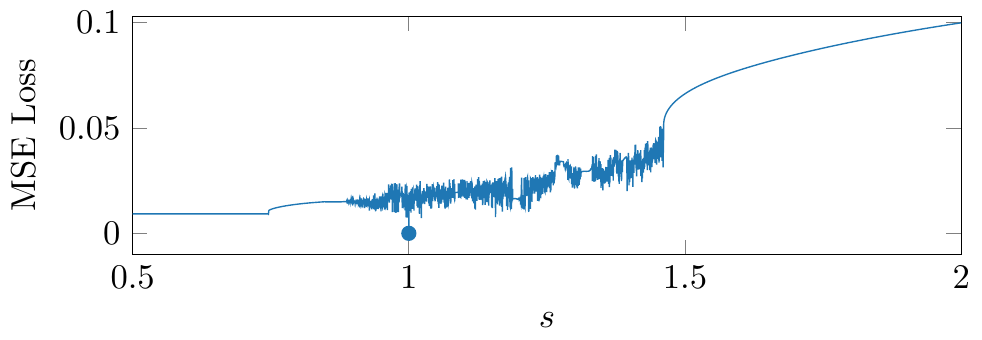}}
\vspace{-0.1cm}
\caption{\textbf{Chaotic LSTM.} Display: a)~Bifurcation diagram; and b)~cost function (mean-square error) for LSTM models with parameter vectors $\theta(s) = s \theta_{\mathrm{true}}$.}
\label{fig:chaotic_LSTM_1d}
\vspace{-0.3cm}
\end{figure*}

\section{Recurrent neural networks}
\label{sec:setup}

\begin{figure}[tpb]
  \vspace{-0.1cm}
\centering
\subfloat[Training]{\includegraphics[width=0.22\textwidth]{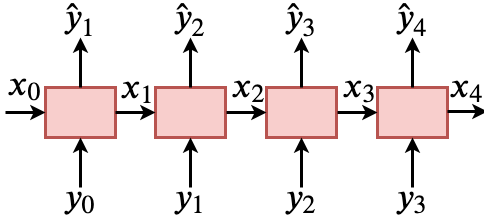}}
\hspace{0.04\textwidth}\subfloat[Inference]{\includegraphics[width=0.22\textwidth]{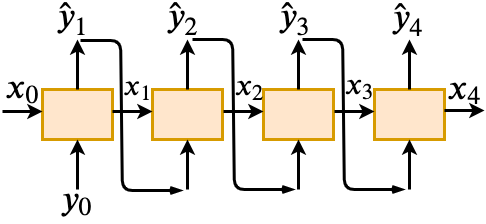}}
  \vspace{-0.05cm}
\caption{\textbf{Teacher forcing.} 
Illustration of a recurrent neural network in two different modes. In (a), observed previous values of the output $\mathbf{y}_t$ are used as input to the model. In (b), the neural network  outputs are instead fed back as inputs.
}
\vspace{-0.25cm}
\label{fig:rnn_modes}
\end{figure}

RNNs are nonlinear discrete-time dynamical systems, represented by the expression:
\begin{subequations}
    \label{eq:nlss}
    \begin{align}
        \mathbf{x}_{t+1} &= \mathbf{f}(\mathbf{x}_t, \mathbf{z}_t; \boldsymbol{\theta}),\\
        \hat{\mathbf{y}}_t &= \mathbf{g}(\mathbf{x}_t, \mathbf{z}_t; \boldsymbol{\theta}), 
    \end{align}
\end{subequations}
where ${\mathbf{x}_t \in \mathbb{R}^{N_x}}$ is the hidden state  and ${\hat{\mathbf{y}}_t\in \mathbb{R}^{N_y}}$, the predicted output. Training the model amounts to finding a parameter vector $\boldsymbol{\theta}$  that makes the model suitable for a given task.  Eq.~\eqref{eq:nlss} is sufficiently general to capture vanilla RNNs, LSTM~\citep{hochreiter_long_1997}, GRU~\citep{cho_properties_2014}, and stacked layers of these units.

When modeling a sequence $\{\mathbf{y}_t\}$, the previously observed outputs are often used as input to the RNN during training i.e., $\mathbf{z}_t = \mathbf{y}_{t-1}$.  This technique is  called \textit{teacher forcing} and it is commonly used in training language models~\citep{peters_deep_2018}.

The mode that is used during training does not have to be used during inference. Furthermore, after training, instead of the observed values, we can successively use the last neural network predicted value to generate many-steps-ahead predictions. The two modes are illustrated in Figure~\ref{fig:rnn_modes}. The RNN representation in Eq.~(\ref{eq:nlss}) is general enough to account for the inference mode by defining an extended hidden state  $(\mathbf{x}_{t}, \hat{\mathbf{y}}_t)$ and new transition and output functions,  $\mathbf{f}$ and $\mathbf{g}$, respectively. Hence, teacher forcing yields transition and output functions during inference that are not the same as those used during training.

When the problem being solved has a clearly defined input (c.f.  Sections~\ref{sec:sine-wave-generation} and~\ref{sec:symb-class}), the input sequence $\mathbf{u}_t$ should be fed to the neural network. In this case, the neural network input during training can be defined either as $\mathbf{z}_t = \mathbf{u}_t$ or, using teacher forcing, as $\mathbf{z}_t = (\mathbf{y}_{t-1}, \mathbf{u}_t)$.

The model parameters are estimated, i.e., the RNN is trained, by minimizing the cost function:
\begin{equation}
  \label{eq:cost_function}
    V(\boldsymbol{\theta}) = \frac{1}{N} \sum_{t=1}^N l(\mathbf{y}_t, \hat{\mathbf{y}}_t),
\end{equation}
where $l$ is the loss function and $N$ is the sequence length. For regression and classification problems the squared error and cross entropy loss are common choices. In the case of multiple independent training sequences, the parameters might be estimated by minimizing a weighted average of several $V$s defined as above.

\section{Information in a recurrent network}
\label{sec:rnn-information}

It is common to associate the challenges of storing information about a distant past in an RNN  with the \textit{vanishing gradients} problem during training~\citep{hochreiter_long_1997, pascanu_difficulty_2013}. That is, for long sequences, as the error gradients are backpropagated through the RNN, they might shrink exponentially to zero, making it harder to learn long-term dependencies.

This interpretation has motivated successful strategies for training recurrent neural networks, such as gating mechanisms~\citep{hochreiter_long_1997}. While useful, this notion alone does not give the whole picture, and the presence or absence of dynamic attractors plays an important role in the RNN capability of storing information.

\subsection{Entropy of the internal states}

We start our analysis by studying the amount of \textit{information} stored by an RNN and how this amount of information changes over time.  The discussion applies both to training and inference, and the consequences to each case will be detailed later on.

Assume that at time $t$ the system state $\mathbf{x}_t$ is distributed according to $p_t(\mathbf{x}_t)$. The entropy associated with this probability distribution provides a way of quantifying how much information we would obtain in measuring the state. For the set $\Omega_x$, the entropy $H_t$ at time $t$ can be computed as: 
\begin{equation}
    \label{ref:entropy_state}
    H_t = -\int_{\Omega_x} p_t(\mathbf{x}_t) \log p_t(\mathbf{x}_t) \textrm{ d} \mathbf{x}_t.
\end{equation}
Under mild assumptions, we have (see Section~\ref{sec:proofs} in the supplementary material) established the inequality: 
\begin{equation}
\label{eq:entropy_increment_upper_bound}
   H_{t} + N_x \log L_f  \leq H_{t+1},
\end{equation}
where $N_x$ is the dimension of $\mathbf{x}_t$ and $L_f$ is the Lipschitz constant of $\mathbf{f}$ in Eq.~(\ref{eq:nlss}).

The Lipschitz constant $L_f$  is related to the dynamical behaviour of the RNN, and Eq.~(\ref{eq:entropy_increment_upper_bound}) establishes that this constant gives a lower bound on the entropy evolution over time. The lower bound is illustrated in Figure~\ref{fig:entropy_increment_upper_bound} and shows that:
\vspace{-0.4cm}
\begin{enumerate}
    \item for a system with $L_f < 1$, the entropy might decay over time towards zero. The larger $L_f$ is, the slower the decay of information retention can be. 
    \item for a system with $L_f = 1$, the entropy can stay constant if the bound in Eq.~\eqref{eq:entropy_increment_upper_bound} is tight. Hence, such a system might retain the information;
    \item for a system with $L_f > 1$, the entropy increase over time. This is the case for chaotic systems, for instance.
\end{enumerate}   
\vspace{-0.4cm}
By the above discussion, the region of the parameter space for which $L_f < 1$, is the region where the model may not be able to retain information indefinitely. Simultaneously, $L_f < 1$ is a sufficient condition for the gradient to vanish, making it harder for the model to escape from the corresponding regions of the parameter space during training. Hence, information retention and the vanishing gradient problem are related; nevertheless, the two concepts are distinct.

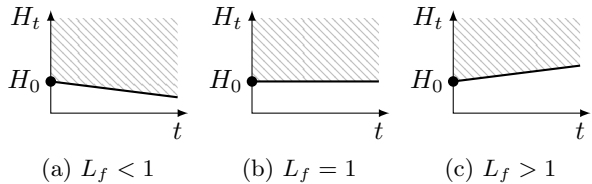
\begin{figure}[t]
  \vspace*{-0.4cm}
    \centering
    \subfloat[$L_f < 1$]{\begin{tikzpicture}[>=latex, x=0.8em, y=0.6em]
	\path[pattern=north west lines, pattern color=lightgray] (0,2) -- (6,1) -- (6,6) -- (0,6) -- (0,2);
	\draw[<->] (0,6.5) |- (6.5,0);
	\draw[-, thick] (0,2) -- (6,1);
	\path[fill] (0,2) circle[radius=0.2em];
	\node[below] at (6,0) {$t$};
	\node[left] at (0,6) {$H_t$};
	\node[left] at (0,2) {$H_0$};
\end{tikzpicture}}
    \subfloat[$L_f = 1$]{\begin{tikzpicture}[>=latex, x=0.8em, y=0.6em]
	\path[pattern=north west lines, pattern color=lightgray] (0,2) -- (6,2) -- (6,6) -- (0,6) -- (0,2);
	\draw[<->] (0,6.5) |- (6.5,0);
	\draw[-, thick] (0,2) -- (6,2);
	\path[fill] (0,2) circle[radius=0.2em];
	\node[below] at (6,0) {$t$};
	\node[left] at (0,6) {$H_t$};
	\node[left] at (0,2) {$H_0$};
\end{tikzpicture}}
    \subfloat[$L_f > 1$]{\begin{tikzpicture}[>=latex, x=0.8em, y=0.6em]
	\path[pattern=north west lines, pattern color=lightgray] (0,2) -- (6,3) -- (6,6) -- (0,6) -- (0,2);
	\draw[<->] (0,6.5) |- (6.5,0);
	\draw[-, thick] (0,2) -- (6,3);
	\path[fill] (0,2) circle[radius=0.2em];
	\node[below] at (6,0) {$t$};
	\node[left] at (0,6) {$H_t$};
	\node[left] at (0,2) {$H_0$};
\end{tikzpicture}}
    \caption{\textbf{Information in RNNs.} Illustration of the lower bound on the entropy $H_t$ obtained in Eq.  \eqref{eq:entropy_increment_upper_bound}. Starting from $H_0$, the entropy $H_t$ can only take values in the shaded region. The entropy may: (a)~ decay; (b)~remain constant; or, (c)~increase over time.
      \vspace*{-0.3cm}
    \label{fig:entropy_increment_upper_bound}}
\end{figure}

\subsection{Contractive \textit{vs} non-contractive systems}
\label{sec:contr-text-non}

A necessary condition for a dynamical system to have  $L_f < 1$ is that the system is \textit{contractive}\footnote{If $\tilde{\mathbf{u}} \in \Omega_{\mathbf{u}}$ it follows that $L_f \geq L_{\tilde{\mathbf{u}}}$. Hence, if the system is not contractive then $L_f \geq L_{\tilde{\mathbf{u}}} \geq 1$.}. 

\begin{defin}[\textbf{contractive system}]
  For a given input sequence $\{\tilde{\mathbf{u}}_t\}$, a dynamical system~(\ref{eq:nlss}) is said to be \textit{contractive} in $\Omega_{\mathbf{x}}$ if it satisfies, for some $L_{\tilde{\mathbf{u}}} < 1$ and for all $\mathbf{x}$ and $\mathbf{w}$ in $\Omega_{\mathbf{x}}$:
\begin{equation}
    \|\mathbf{f}(\mathbf{x}, \tilde{\mathbf{u}}_t; \boldsymbol{\theta}) - \mathbf{f}(\mathbf{w}, \tilde{\mathbf{u}}_t; \boldsymbol{\theta})\| < L_{\bar{\mathbf{u}}} \|\mathbf{x} -\mathbf{w}\|.
\end{equation}
\end{defin}

All contractive systems have a unique fixed point inside the contractive region $\Omega_{\mathbf{x}}$, and all trajectories converge to such a fixed point~\cite[Theorem 9.23]{rudin_principles_1964}. Hence, for contractive systems, the distribution $p_t(\mathbf{x}_t)$ will converge to a point mass at the fixed point, and thus zero entropy. Systems with richer nonlinear dynamic behaviors, such as multiple fixed points, limit cycles, and chaotic attractors, are all \textit{non-contractive}.

\subsection{Long-term memory}

In what follows, the concept of long-term memory will be closely related to the definition of an attractor of a dynamical system.

\begin{defin}[\textbf{attractor of a dynamical system}]
Let $\mathbf{u}_t$ be equal to a constant $\bar{\mathbf{u}}$ for every $t$. An \textbf{attractor} of the dynamical system $\mathbf{x}_{t+1} = \mathbf{f}(\mathbf{x}_t, \bar{\mathbf{u}}; \boldsymbol{\theta})$ is a subset $\mathcal{A} \subset \mathbb{R}^{N_x}$ for which:
\vspace{-0.4cm}
\begin{enumerate}
    \item if $\mathbf{x}_{t_0} \in \mathcal{A}$ then $\mathbf{x}_t \in  \mathcal{A}$ for all $t \ge t_0$;
    \item there exists a neighborhood of $\mathcal{A}$, called the \textbf{basin of attraction} $\mathcal{B}_{\mathcal{A}}$, such that: if $\mathbf{x}_{t_0}\in \mathcal{B}_{\mathcal{A}}$ then $\mathbf{x}_{t} \rightarrow \mathcal{A}$ as $t \rightarrow \infty$;
    \item there is no proper (non-empty) subset of $\mathcal{A}$ having the first two properties.
\end{enumerate}
\vspace{-0.4cm}
\end{defin}

Property $1$ is related to the concept of long-term memory. If at a given time instant $t_0$, $\mathbf{x}_{t_0} \in \mathcal{A}$, the system state will stay in $\mathcal{A}$ for all future time points, that is, it will ``remember'' this set $\mathcal{A}$. Systems that do not respect this property will just leave (i.e.~``forget'') set $\mathcal{A}$ at some time $t > t_0$.

Property $2$ is related to the \textit{robustness} of the long-term memory. So, for $\mathbf{x}_{t} \in \mathcal{A}$, we can disturb the system by applying a non-zero finite input sequence ${\mathbf{u}_t}$ and, if this sequence is sufficiently small so that the state remains in $\mathcal{B}_{\mathcal{A}}$, then the system\footnote{for $\mathbf{f}$ continuous in $\mathbf{u}$.} will converge back to $\mathcal{A}$. In order words, the system will not ``forget'' $\mathcal{A}$ even in the presence of some (sufficiently small) disturbance.

The idea of defining long-term memory as attractors of a dynamical system is not new.  Similar approaches have been pursued in~\citet{bengio_problem_1993, bengio_learning_1994}. Examples of attractors of after-training RNNs are studied by~\citet{sussillo_opening_2013} and by~\citet{maheswaranathan_reverse_2019}, for toy problems and sentiment analysis, respectively.

In this paper, we will build on this idea and show the RNN may go through bifurcations during training for tasks that require long-term memory and learn attractors to solve the problem. We will also show how stable RNNs and orthogonal RNNs use different mechanisms. For this analysis we will use the bifurcation diagram to visualize the attractors of the RNN. These diagrams show the values visited in steady-state behavior (when a constant input is applied to the system) as a function of some bifurcation parameter $s$. See Figure~\ref{fig:chaotic_LSTM_1d}(a).

As mentioned before, it is possible to have different state transition and output functions for training and inference, respectively. See Figure~\ref{fig:rnn_modes}. Hence, the attractors of a recurrent model might be different during training and inference. This is the case for the numerical example in Section~\ref{sec:word-level-language}, where we include the analysis of the attractor and bifurcation diagram in both scenarios. Before undertaking this analysis we will study the influence that the internal dynamics has on the smoothness of the cost function.

\section{Smoothness of the cost function}
\label{sec:smoothness}

We will now establish connections between the \textit{smoothness} of the cost function and the internal dynamics of the RNN \textit{during training}. These connections allow us to extend the concept of exploding gradient with the analysis of the smoothness of the cost function. This analysis is based on the Lipschitz constant of the cost function and of its gradient (sometimes called $\beta$-smoothness). Both constants play a crucial role in optimization, see~\citet{nesterov_introductory_1998}, and can be understood as qualitative measurements of how \textit{smooth} the cost function is. Lower values imply that the cost function is less intricate and that optimization algorithms can still converge while taking larger steps. It also provides an upper bound on how different the performance of two close local minima may be.

This smoothness analysis suggests that higher-order derivatives (which contain information about the curvature) might  explode in some regions of the parameter space. This indicates that not only walls are formed, but also entire regions of the parameter space exhibit intricate behavior, full of undesirable local minima. This goes beyond what many papers suggest about exploding gradients. For instance, \citet{pascanu_difficulty_2013} formulate the hypothesis that when gradients explode, they do so in some specific direction, creating \textit{walls} of high curvature.  ~\citet{doya_bifurcations_1993} speculates that bifurcations might cause the jumps in the cost function.

\subsection{Example: chaotic LSTM}
\label{sec:chaotic-lstm}

We consider an LSTM model of dimension two without the bias terms. In Figure~\ref{fig:chaotic_LSTM_1d}(a) we show the bifurcation diagram and, in Figure~\ref{fig:chaotic_LSTM_1d}(b), the cost function. These are given as a function of $s$: the weights of the RNN are $\theta(s) = s \theta_\mathrm{true}$ and $\theta_\mathrm{true}$ denotes the true data generating weights. See Section~\ref{sec:chaotic-lstm-details} in the supplementary material for additional details.

The bifurcation diagram depicts, for each parameter, the steady-state behavior of the system. It was generated using a simulation of 200 samples for which the first 100 samples were discarded to remove the transient. In this bifurcation diagram we can observe a region with rich nonlinear and chaotic behavior for $0.9 \lesssim s \lesssim 1.45$. This region corresponds to the region where the cost function (Figure~\ref{fig:chaotic_LSTM_1d}(b)) is intricate and full of local minima. This connection between internal dynamics and smoothness of the cost function is formalized in the next subsection.

\subsection{Theoretical results}
\label{sec:theoretical_results}

In this work, we phrase the exploding gradient problem in terms of the smoothness of the cost function.
Here, we quantify smoothness using the Lipschitz constant of the cost function $V$ and of its gradient, the so-called~$\beta$-smoothness,
and establish a connection between these constants, the simulation length~$N$ and the state-transition Lipschitz constant $L_f$. 
The formulation is quite natural and allows us to include higher-order derivatives in the analysis. 
We \citep{ribeiro_smoothness_2019} have previously studied this result in a different context. Here it is generalized to RNNs. 

The Lipschitz constant of $V$ and of its gradient provide upper bounds on $\|\nabla V\|$ and $\|\nabla^2 V\|$,  respectively, cf.~\citet[Lemma 3.1]{khalil_nonlinear_2002}.  Hence, the first part of the theorem below can indeed be seen as a formalization of the exploding gradient problem. The second part gives information about the explosion of second-order derivatives and curvature and, to the best of our knowledge, is novel. For the case $L_f < 1$, similar results have been presented by~\citet{miller_stable_2019}, but not for the case that interests us the most: $L_f > 1$, for which there might be an explosion in the curvature.

\begin{theorem}[\textbf{Lipschitz constants of $V$ and $\nabla V$}]
  \label{thm:lipschitz}
  Let $\mathbf{f}(\mathbf{x}, \mathbf{u}; \boldsymbol{\theta})$ and $\mathbf{g}(\mathbf{x},  \mathbf{u}; \boldsymbol{\theta})$ in Eq.~(\ref{eq:nlss}) be  Lipschitz in ${(\mathbf{x}, \boldsymbol{\theta})}$ with constants $L_f$ and $L_g$ on a compact and convex set $\Omega = (\Omega_{\mathbf{x}},\Omega_{\mathbf{u}}, \Omega_{\boldsymbol{\theta}})$. Assume the loss function is either $l(\hat{\mathbf{y}}, \mathbf{y}) = \|\mathbf{y} -  \hat{\mathbf{y}}\|^2$ or $l(\hat{\mathbf{y}}, \mathbf{y}) = -\mathbf{y}^T \log (\sigma(\hat{\mathbf{y}})) - (1-\mathbf{y})^T \log (1 - \sigma(\hat{\mathbf{y}}))$. Let $\{\mathbf{u}_t\}_{k=1}^N \subseteq \Omega_{\mathbf{u}}$ and $(\Omega_{\mathbf{x}}, \Omega_{\boldsymbol{\theta}})\subseteq \mathbb{R}^{N_\theta}$. If there exists at least one choice of ${(\mathbf{x}_0, \boldsymbol{\theta})}$ for which there is an invariant set contained in $\Omega$, then, for trajectories and parameters confined within  $\Omega$:
  \vspace{-0.4cm}
  \begin{enumerate}
      \item 
       The cost function $V$ defined in~(\ref{eq:cost_function}) is Lipschitz with constant:\footnote{Here $\mathcal{O}$ is the big O notation. It should be read as: $L(N) = \mathcal{O}(g(N))$ if and only if there exist positive integers $M$ and $N_0$ such that ${ |L(N)|\leq \;Mg(N)\text{ for all }N\geq N_{0}}$.}
      \begin{equation}
      \label{eq:asymptotic_L}
      L_V = 
          \begin{cases}
          \mathcal{O}(L_f^{2N}) & \text{if } L_f > 1, \\
          \mathcal{O}(N) & \text{if } L_f = 1, \\
          \mathcal{O}(1) & \text{if } L_f < 1.
          \end{cases}
      \end{equation}
      \item
      If the Jacobian matrices of $\mathbf{f}$ and $\mathbf{g}$ are also Lipschitz with respect to ${(\mathbf{x}, \boldsymbol{\theta})}$ on $\Omega$, then the gradient of the cost function $\nabla V$ is Lipschitz with constant:
          \begin{equation}
      \label{eq:asymptotic_L'}
      L_{V}' = 
          \begin{cases}
          \mathcal{O}(L_f^{3 N}) & \text{if } L_f > 1, \\
          \mathcal{O}\left(N^3\right) & \text{if } L_f = 1, \\
          \mathcal{O}\left(1\right) & \text{if } L_f < 1.
          \end{cases}
      \end{equation}
    \end{enumerate}
    \vspace{-0.4cm}
\end{theorem}
\vspace*{-0.3cm}
\begin{proof}
Section~\ref{sec:proofs} in the supplementary material.
\end{proof}
\vspace*{-0.2cm}

\begin{obs}[\textbf{other loss functions}]
    The above theorem was stated for two different loss functions (squared error and average cross-entropy preceded by the sigmoid function). The theorem still holds for any loss function for which the equivalent of Lemma~\ref{thm:condition_on_the_loss} (in the supplementary material) remains true.
\end{obs}
\begin{obs}[\textbf{relation between $L_f$ and eigenvalues}]
Let $A$ be the Jacobian matrix of $\mathbf{f}$ with respect to $\mathbf{x}$ for a fixed input $\bar{\mathbf{u}}$. The constant $L_{\bar{\mathbf{u}}}$  in the set $\Omega$ is equal to the largest possible eigenvalue of $A$ inside the set. Furthermore, $L_{\bar{\mathbf{u}}}$ is a lower bound on $L_f$, hence the necessary condition for exploding gradients to appear given in~\cite{pascanu_difficulty_2013} (i.e.~largest eigenvalue bigger than 1) also follows from Theorem~\ref{thm:lipschitz}. The necessary condition for the second derivative to explode could also be stated in terms of the largest eigenvalue: it needs to be bigger than 1.
\end{obs}

It is important to highlight that the proof of this more general result follows a different strategy compared to existing proofs of the exploding gradient problem (such as in~\cite{pascanu_difficulty_2013}). Rather than backpropagating the derivatives (which is more natural in the context of neural networks), we use forward propagation (similarly to what is done by~\citet{williams_experimental_1989}) which allows us to arrive at this more general result.

If the system is non-contractive we have $L_f \geq 1$ (cf. Section~\ref{sec:contr-text-non}). According to Theorem~\ref{thm:lipschitz}, the Lipschitz constants and $\beta$-smoothness for all these systems might blow up exponentially (or polynomially for some limit cases) with the maximum simulation length. This  might yield very intricate cost functions, see Figure~\ref{fig:chaotic_LSTM_1d}(b).

\section{Learning attractors during training}

Let us now study the mechanism employed by RNNs in learning a task that requires long-term memory. We focus on the LSTM~\citep{hochreiter_long_1997}; on the stable LSTM (sLSTM) proposed by~\cite{miller_stable_2019},  which projects, at each iteration, the parameters into the region of the parameter space that yields contractive models; and, on the orthogonal RNN (oRNN) proposed by~\citet{lezcano-casado_cheap_2019}.

\subsection{Sine-wave generation}
\label{sec:sine-wave-generation}

\begin{figure*}[t]
  \vspace*{-0.85cm}
  \centering
  \subfloat[LSTM]{
    \includegraphics[trim={5cm 3cm 0.3cm 3cm},clip,width=0.3\textwidth]{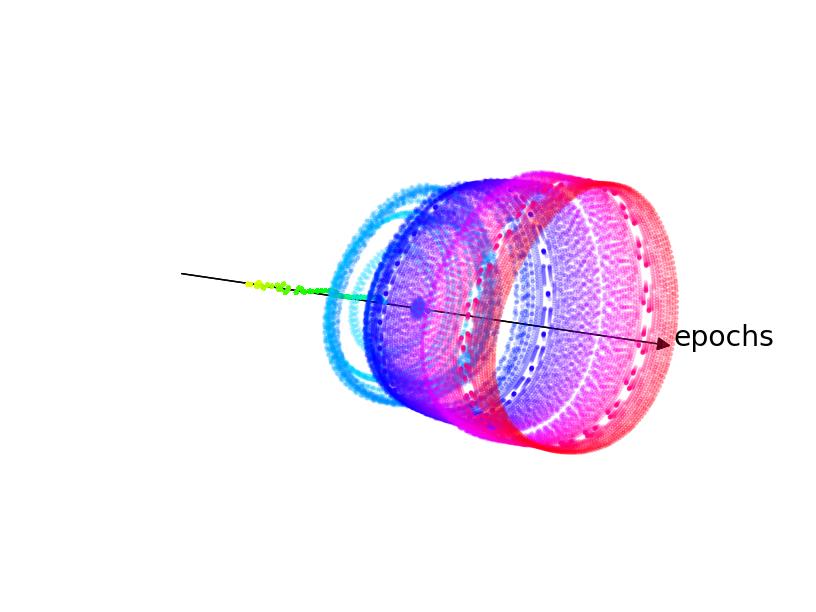}}
  \subfloat[sLSTM]{
    \includegraphics[trim={5cm 3cm 0.3cm 3cm},clip,width=0.3\textwidth]{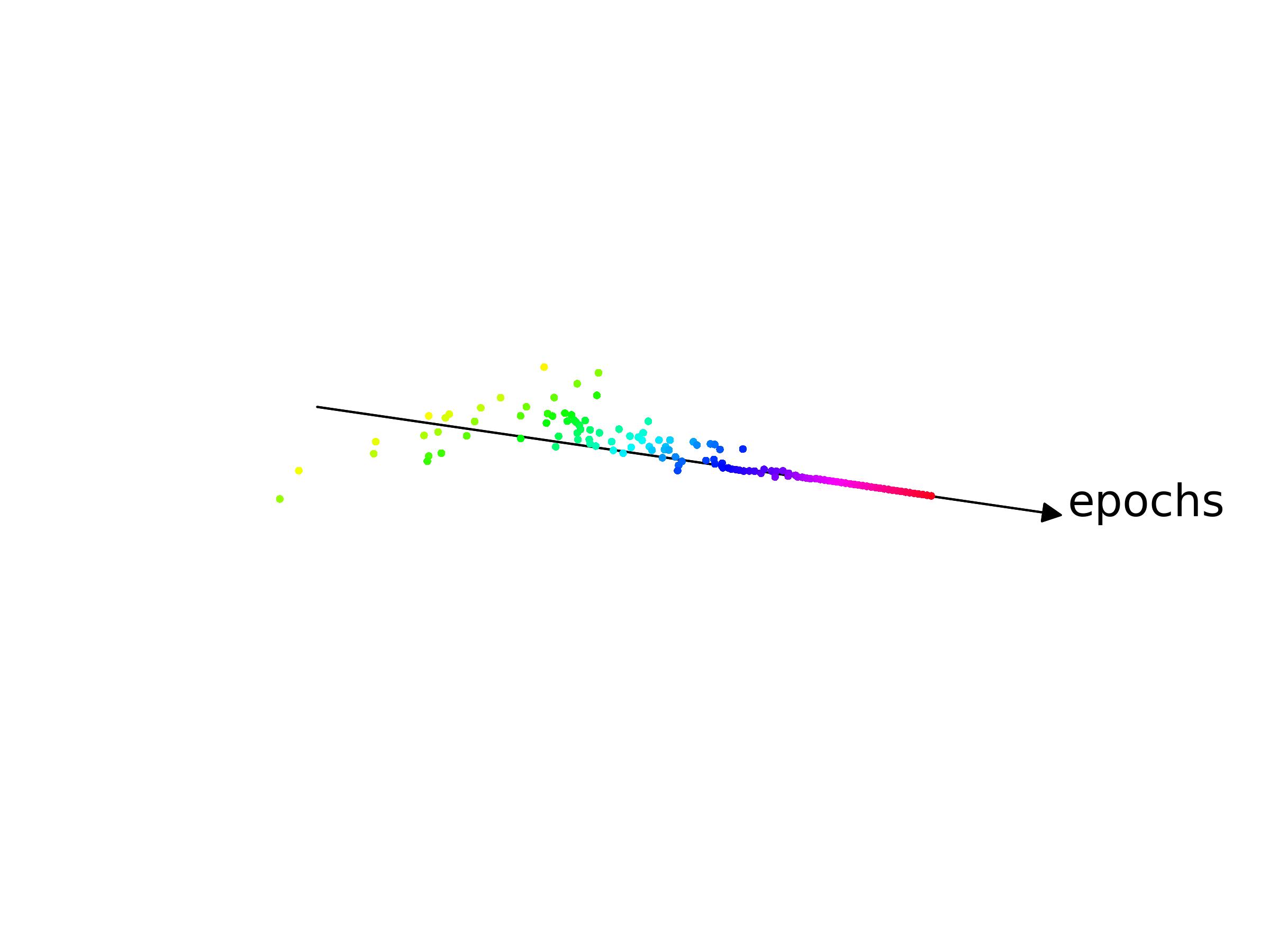}}
  \subfloat[oRNN]{
    \includegraphics[trim={4cm 3cm 0.3cm 3cm},clip,width=0.3\textwidth]{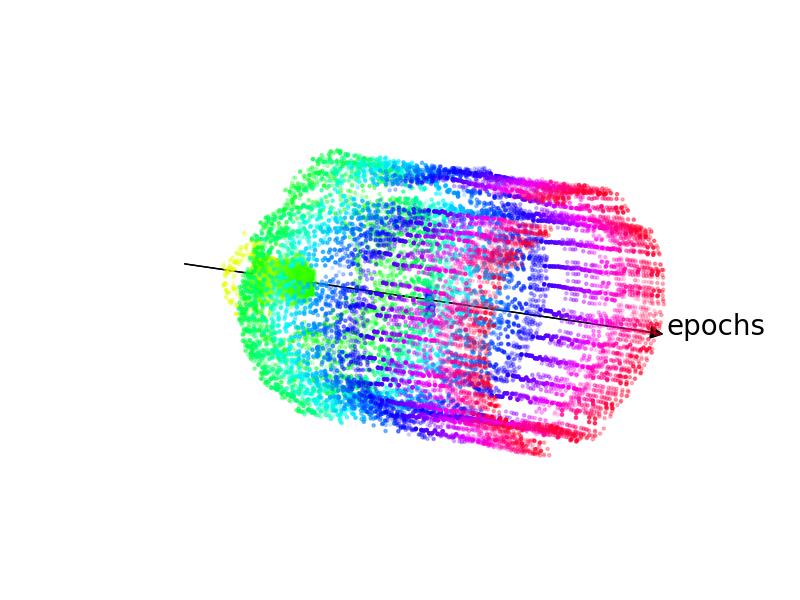}}
  \caption{\textbf{Learning oscillations.} Bifurcation diagram during training for the sine-wave generation task  showing the steady-state of the output $y_t$ and its first difference $y_t - y_{t-1}$. For each epoch, the plot shows values visited after a burn-in period of 1200 samples, which is used to remove the transient response and to produce a visualization of the system attractors. The plot shows epochs 300 to 1500, the initial 300 epochs are out of scale for the oRNN, and are omitted in all diagrams. The bifurcation diagram is for a constant input of $0.218$, and other input values yield similar plots.}
  \vspace*{-0.1cm}
  \label{fig:sinewave-orbit-diagram}
\end{figure*}

\begin{figure}[t]
  \centering
  \vspace*{-0.5cm}
  \includegraphics[width=0.5\textwidth]{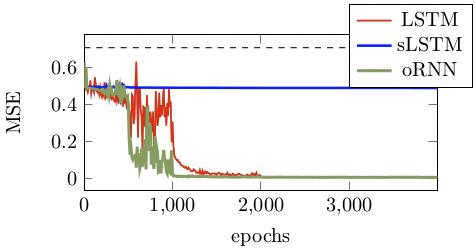}
  \caption{\textbf{Sine-wave generator training history.} Mean square error (MSE) per epoch for the sine-wave  generation task during training. The baseline performance is indicated by the dashed line. The final performance of the LSTM after 4000 epochs is $4.8*10^{-3}$, the final performance of the sLSTM is $0.49$ and the final performance of the oRNN is $4.1*10^{-3}$.}
  \vspace*{-0.6cm}
  \label{fig:sinewave-history}
\end{figure}

We study the use of RNNs for the generation of sine-waves with varying frequencies. The input is a constant, and the output is a sine-wave with unitary amplitude and the frequency specified by the constant input. This artificial task is described in~\citet{sussillo_opening_2013}. We have 100 sequences of length 400, each sequence consisting of a sine-wave with the given constant frequency.  The 100 sequences are uniformly picked in the interval $\left[\frac{\pi}{16}, \frac{\pi}{8}\right]$. We use those sequences for training.

For all models, we use the same configurations: Adam optimization algorithm~\citep{kingma_adam_2014} with initial learning rate $10^{-3}$ and default parameters. For the stable model, we decrease the learning rate by 10 at epochs $\{500, 1000, 2000\}$. We use a hidden size of 200 and a single layer for the RNN, which is followed by a linear layer for the output. The gradient is clipped when its norm exceeds 0.25.

The training history for the three models is displayed in Figure~\ref{fig:sinewave-history}. The stable LSTM fails to perform well in the sine-wave generation task. Figure~\ref{fig:sinewave-orbit-diagram} shows the bifurcation diagram for the three models during training. Here we use a two-dimensional bifurcation diagram containing both the output $y_t$ and its first difference $y_t - y_{t-1}$ to facilitate the visualization of periodic attractors.

For the LSTM, during training, a stable fixed point undergoes a bifurcation at which the fixed point stability switches, and a periodic solution arises.  The stable LSTM is constrained to stay in the region of the parameter space for which the system is contractive and, hence has a single stable attractor point (cf.~discussion on Section~\ref{sec:smoothness}). Since the system needs to sustain the oscillations, this is a significant limitation that prevents the model from performing the task well.

The bifurcation that the fixed point needs to undergo to oscillate is called Hopf (for the continuous case) or Neimark-Sacker bifurcation (for the discrete case). During this bifurcation, a single stable fixed point changes stability and becomes unstable as a pair of complex eigenvalues (of the Jacobian matrix) enters the unstable region. The orthogonal constraint in the orthogonal RNN prevents the occurrence of these bifurcations since all eigenvalues are fixed to one. Nevertheless, orthogonal RNNs still manage to solve the problem and learn the periodic attractor.

Since most real-world sequence tasks have stochastic elements, the deterministic nature of the sine-wave generation is a limitation of this task to be representative of real-world problems. Nevertheless, this example illustrates the challenges of learning steady-state behavior using RNNs. Stable models will not be able to store information in the form of attractors, unless other training mechanisms, such as teacher forcing, are used (see Figure~\ref{fig:rnn_modes}). Orthogonal models, do not have this restriction, but, even so, the orthogonal constraints can also prevent some bifurcations, and require different learning mechanisms. For instance, another type of bifurcation that this model will fail to capture is the supercritical pitch-fork bifurcation, where a single stable fixed point loses stability, and two stable fixed points are created.

\subsection{Sequence classification based on few relevant symbols}
\label{sec:symb-class}

\begin{figure*}[t]
  \vspace*{-0.4cm}
  \centering
  \subfloat[Sequence length 100]{\includegraphics[height=0.21\textwidth]{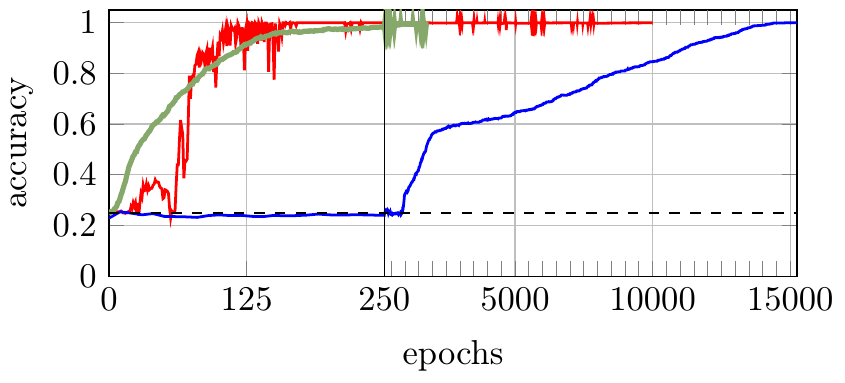}}
  \subfloat[Sequence length 200]{\includegraphics[height=0.2\textwidth]{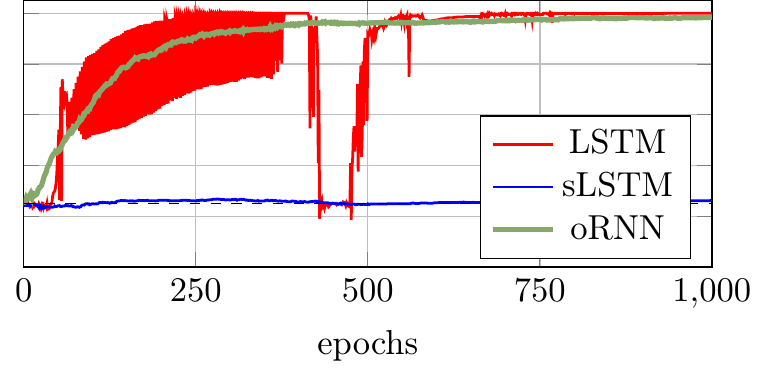}}
  \vspace*{-0.05cm}
  \caption{\textbf{Sequence classification training history.} Accuracy on validation data set for the  recurrent models trained to perform the same sequence classification task for two different sequence lengths. The baseline performance (always predicting $\{p, p\}$) is indicated by the dashed line. In (a) two x-axis scales co-exist in the same graph, one scale in [0, 250) and other in [250, 15250], with a relation 1:40 between the two scales.}
  \label{fig:history-symbol-classification}
\end{figure*}
\begin{figure*}[t]
  \vspace*{-0.4cm}
  \centering
  \subfloat[LSTM, $p\rightarrow \{p, p\}$]{\includegraphics[trim={3cm 1.5cm 0.3cm 3cm},clip, width=0.23\textwidth]{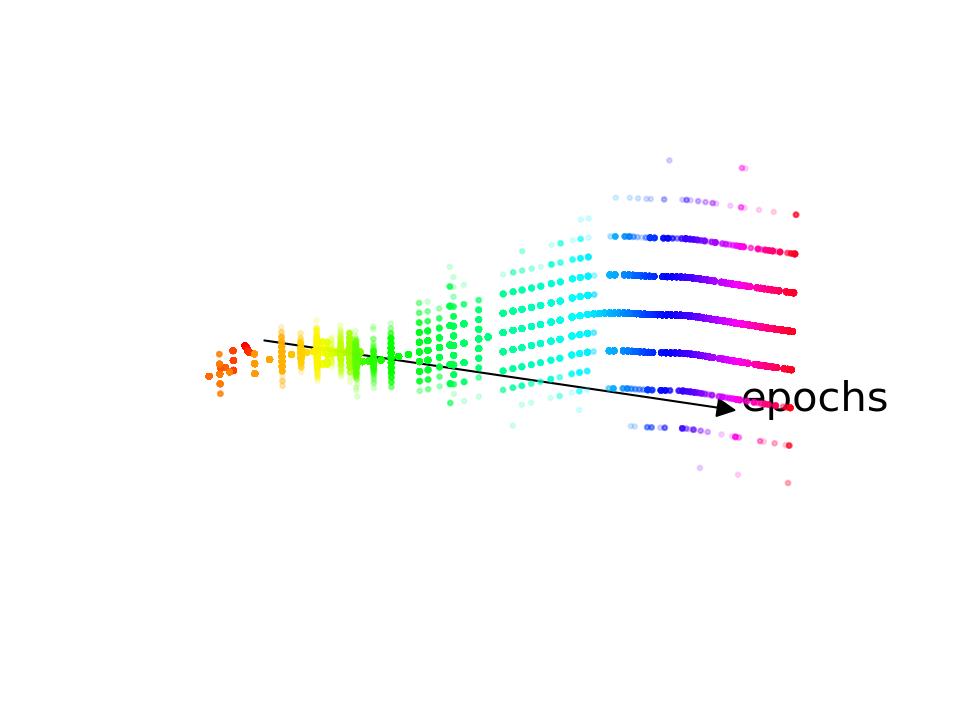}}
  \subfloat[LSTM, $b\rightarrow \{p, p\}$]{\includegraphics[trim={3cm 1.5cm 0.3cm 3cm},clip, width=0.23\textwidth]{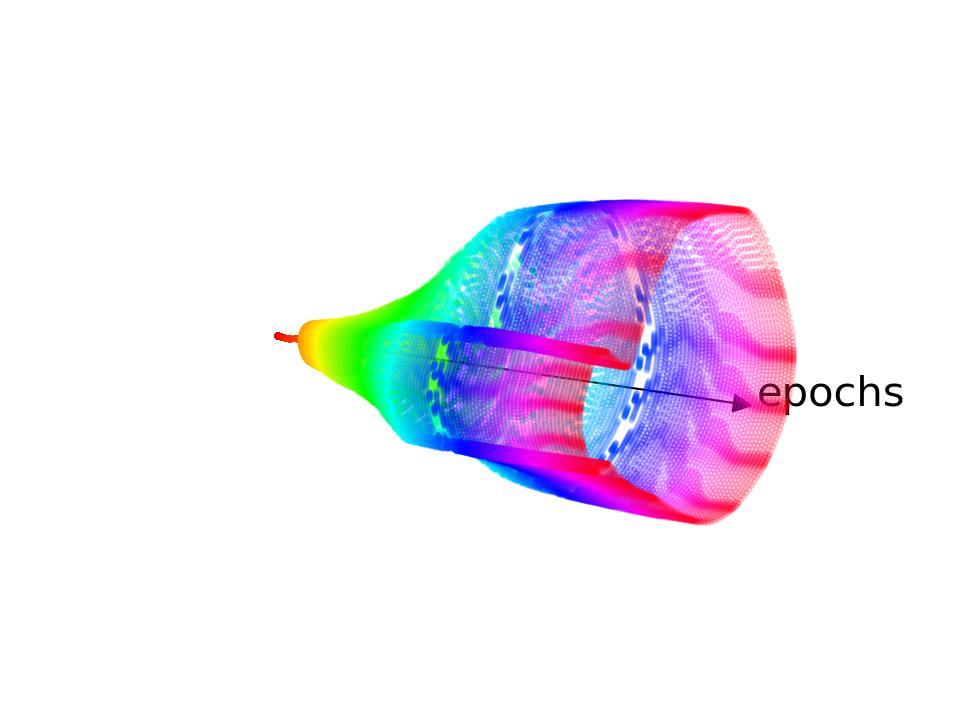}}
  \subfloat[oRNN, $p\rightarrow \{p, p\}$]{\includegraphics[trim={3cm 1.5cm 0.3cm 3cm},clip, width=0.23\textwidth]{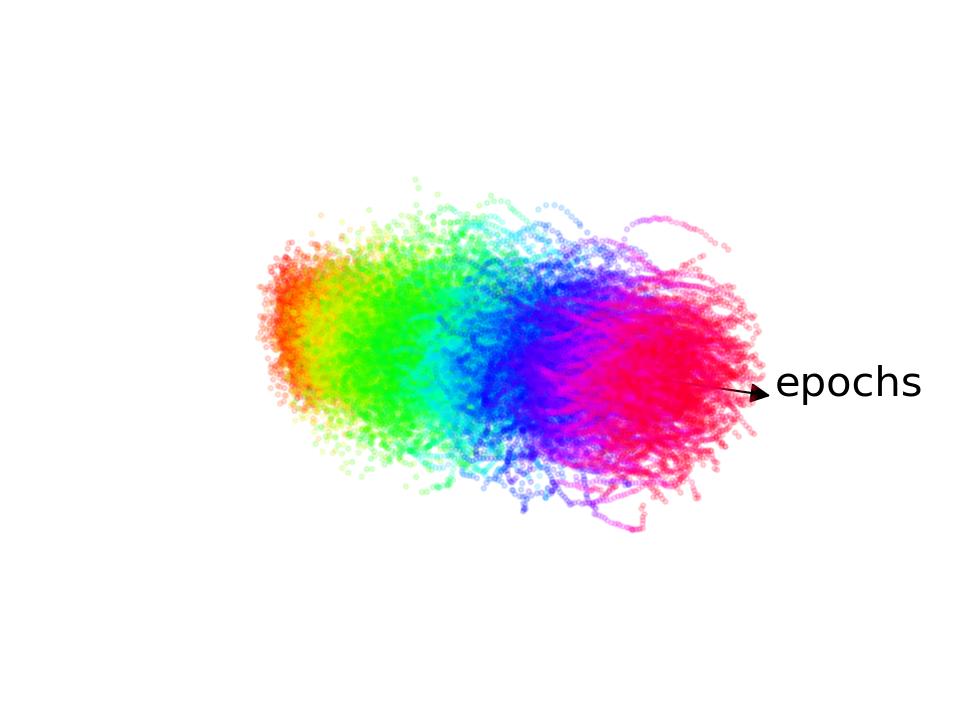}}
  \subfloat[oRNN, $b\rightarrow \{p, p\}$]{\includegraphics[trim={3cm 1.5cm 0.3cm 3cm},clip, width=0.23\textwidth]{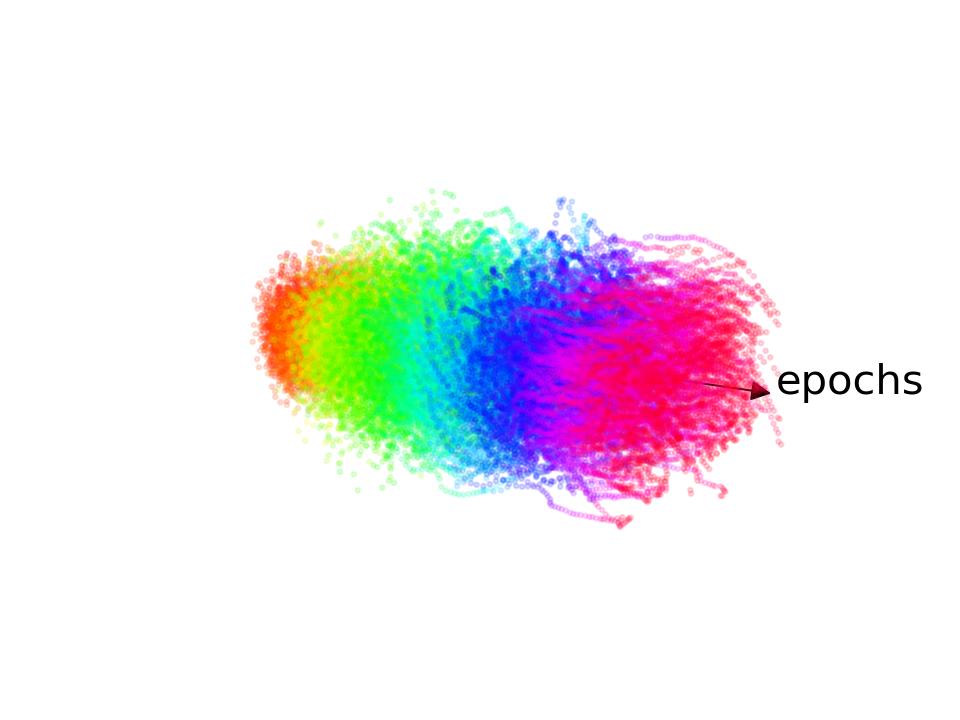}}
  \vspace*{-0.05cm}
  \caption{\textbf{Learning to classify sequences.} Bifurcation diagram for the sequence classification task for sequences of length 100. It shows the steady-state of the output $y_t$ and its first difference $y_t - y_{t-1}$.  The arrows point towards the evolution of the number of epochs, that varies from 0 to 400.
  For each epoch, the plot shows values visited after a burn-in period of 400 samples, which is used to remove the transient response and to produce a visualization of the system attractors.}
  \vspace*{-0.4cm}
  \label{fig:symbol-classification-orbit-diagram}
\end{figure*}

\begin{table}[]
  \centering
  \caption{\textbf{Sequence classification accuracy.} Comparison between the LSTM; stable LSTM (sLSTM); and orthogonal RNN (oRNN); for sequence classification with different sequence lengths ($l$). As a baseline, always predicting $\{p, p\}$ would yield an accuracy of $0.25$.}
  \vspace*{-0.2cm}
  \label{tab:comparisson}
  {\renewcommand{\arraystretch}{0.5}
    \begin{tabular}{@{}rccc@{}}
      \toprule
      $l$ &LSTM & sLSTM & oRNN \\
      \midrule
      50 & 1.00 & 1.00 & 1.000 \\
      100 & 1.00 & 1.00 & 1.000 \\
      200 & 1.00 & 0.27 & 0.999 \\
      300 & 0.25 & 0.26 & 0.995 \\
      500 & 0.27 & 0.26 & 0.970 \\
      \bottomrule
    \end{tabular}
  }
  \vspace*{-0.45cm}
\end{table}

We study the classification of sequences according to a few relevant, widely separated, symbols. This artificial task was originally described by~\citet{hochreiter_long_1997}, and it requires the neural network to retain information for long periods. For this problem, the sequence contains categorical values $\{p, q, a, b, c, d\}$. The symbols $\{a, b, c, d\}$ act as distractors and are not relevant to the task. Instead, the relevant symbols are picked from $\{p, q\}$ appearing only twice in the sequence, at positions $t_1\in [10, 18]$ and $t_2 \in [40, 49]$. The task of the neural network is to classify the sequence according to the order in which the relevant symbols occur. The four possible classes are $\{(p, p), (p, q), (q, p), (q, q)\}$. Training and validation datasets each contain  1000  sequences, and all sequences have the same length.

For all models, we use the same configurations: Adam optimization algorithm with an initial learning rate of $10^{-2}$ and default parameters. We decrease the learning rate by 10 at epochs $\{500, 1000, 2000\}$, and run the optimization for a total of 15000 epochs or until accuracy one is attained. We use a hidden size of 200 and a single layer for the RNN, which is followed by a linear layer with a single output. The gradient is clipped when its norm exceeds 0.25. The batch size is 100.

Figure~\ref{fig:history-symbol-classification} shows the training history for the different models in the sequence classification task. In Figure~\ref{fig:history-symbol-classification}(a), the sLSTM eventually manage to solve the task, the accuracy, however, increase at a very slow linear rate between epochs $[1000, 10000]$. The linear rate is quite characteristic of first-order optimization methods in a basin of attraction~\citep{nocedal_numerical_2006}, and is very slow due to the vanishing gradient problem. On the other hand, the standard LSTM converges quickly. The convergence, however, is quite dependent on the initial condition. See Supplementary Figure~\ref{fig:history-temporal-order} to see how, for a different random seed and identical setup, the model completely fails to converge.

For length 200, the LSTM converges, but very unsteadily and with many bumps. The strong dependency with the initial conditions and hard to navigate landscape is in agreement with the considerations about smoothness and multiple close local minima that might affect this model (cf.~Section~\ref{sec:smoothness}).  The orthogonal RNN avoids this hard to navigate landscape by fixing the eigenvalues equal to one, and it does have a much smoother convergence than the LSTM. It is also the model that manages to solve the problem for the longest sequences (see Table~\ref{tab:comparisson}). The stable LSTM model, on the other hand, fails to solve the task for the length of 200 or longer.

To analyze the attractors during training, we apply a constant unitary input to one of the six input channels (keeping the others at zero) and measure one of the four possible outputs. Bifurcation diagrams for the LSTM and the oRNN for all $6\times 4$ possible combinations of input/output pairs are displayed in Supplementary Figures~\ref{fig:lstm-bifurcation-all} and~\ref{fig:ornn-bifurcation-all}. Figure~\ref{fig:symbol-classification-orbit-diagram} shows a subset of the combinations. More specifically, it shows the bifurcation diagram between the input $p$ and the output $\{p, p\}$ and between the input $b$ and the same output, which we believe to be representative of all possible combinations.

\begin{figure*}[t]
  \vspace*{-0.65cm}
  \centering
  \makebox[\textwidth][c]{
  \subfloat[LSTM]{\includegraphics[trim={2cm 3cm 2.5cm 3cm},clip,width=0.17\textwidth]{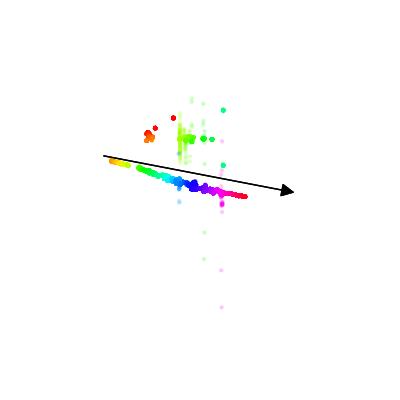}}
  \subfloat[LSTM, feedback]{\includegraphics[trim={2cm 3cm 2.5cm 3cm},clip,width=0.17\textwidth]{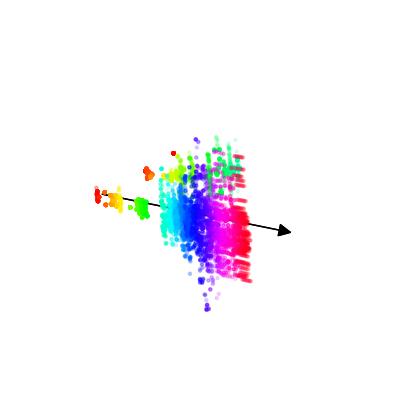}}
  \subfloat[oRNN]{\includegraphics[width=0.17\textwidth]{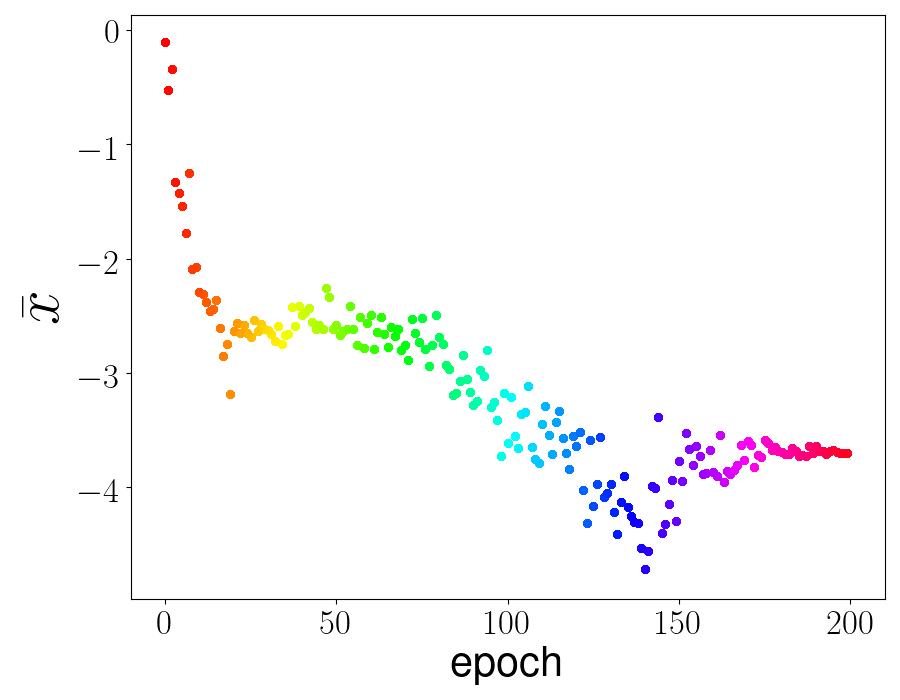}}
  \subfloat[oRNN, feedback]{\includegraphics[width=0.17\textwidth]{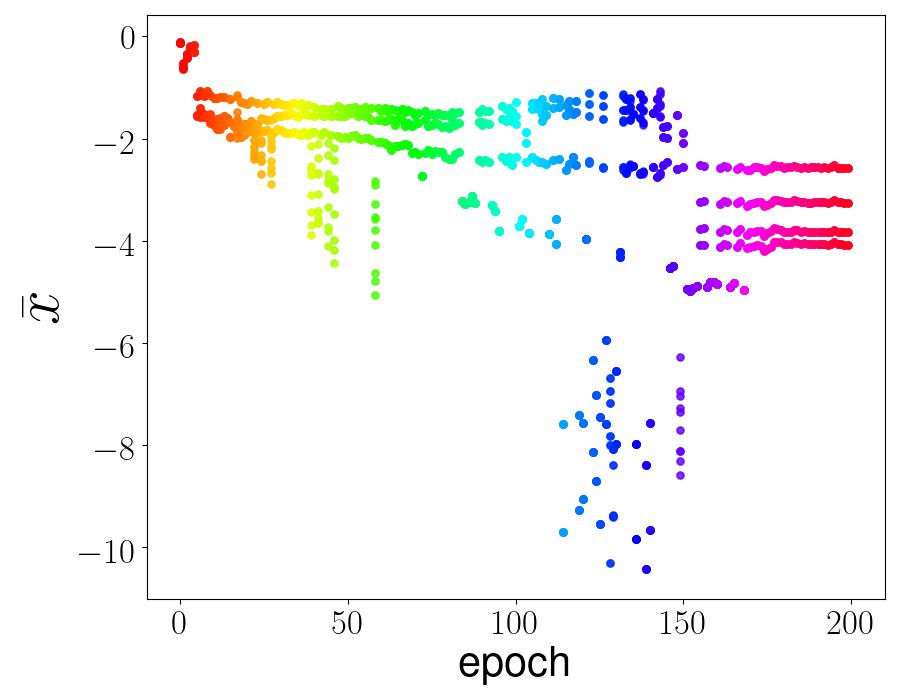}}
  \subfloat[sLSTM]{\includegraphics[width=0.17\textwidth]{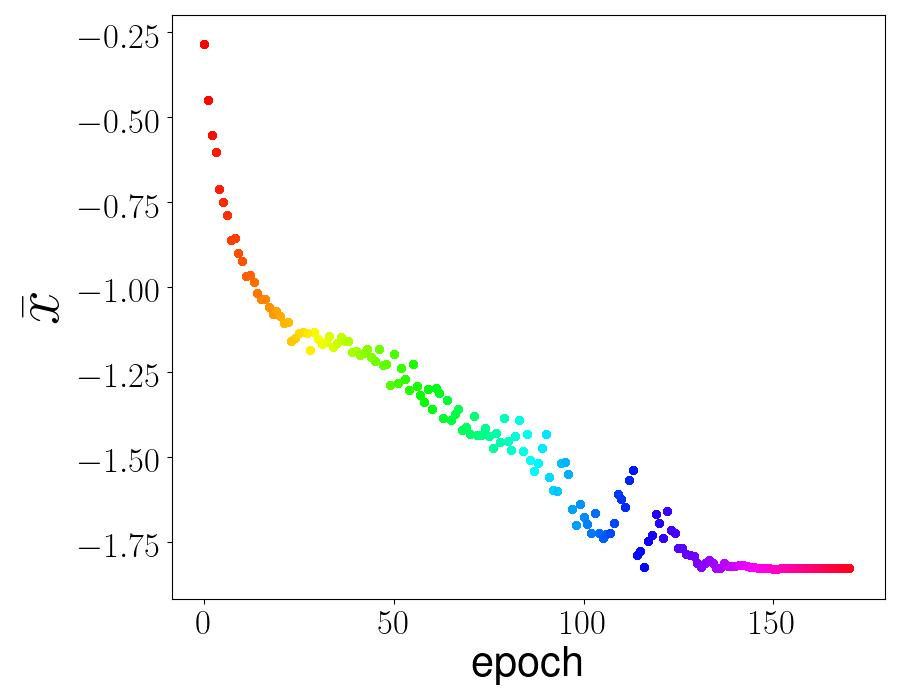}}
  \subfloat[sLSTM, feedback]{\includegraphics[width=0.17\textwidth]{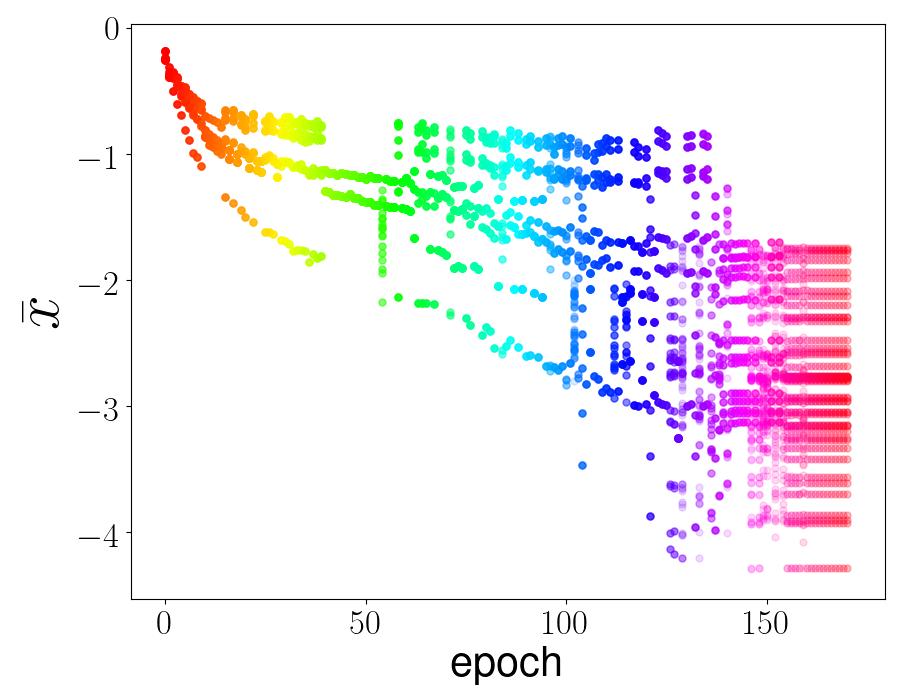}}
}
\vspace*{-0.2cm}
  \caption{\textbf{Learning to model language.} Bifurcation diagram for the world-level language model for: the LSTM, in (a)-(b); the oRNN, in (c)-(d); and, the sLSTM, in (e)-(f). In (a) and (b), we plot a 2-D bifurcation diagram showing the steady-state of the average of internal states $\bar{x}_t$ and its first difference $\bar{x}_t - \bar{x}_{t-1}$. In (c)-(f), a 1-D bifurcation of the average $\bar{x}_t$ is used, which facilitates the visualization. In (a), (c) and (e), the bifurcation diagram is obtained from constant inputs.  In (b), (d) and (f), the diagram is generated using as input the word predicted with the highest probability at the previous time instant, and using as first input to the sequence the same input as in the constant input case.}
  \label{fig:wikitext2-orbit-diagram}
  \vspace*{-0.3cm}
\end{figure*}

The LSTM goes through bifurcations, for instance: in $p\rightarrow \{p, p\}$, the number of fixed points increases, and, in $b\rightarrow \{p, p\}$, they transition to periodic behavior. We understand the ability to undergo bifurcations is useful, and having multiple fixed points helps the LSTM to solve the task for longer sequences than it would be able to only within the contractive region. The comparison between LSTM and stable LSTM in Table~\ref{tab:comparisson} seems to corroborate such a hypothesis. The periodic behavior in $b\rightarrow \{p, p\}$, on the other hand, is spurious, since $b$ and the other distractor symbols do not influence the outcome. The mechanism the orthogonal RNN uses to solve the task is quite distinct: they create a  cloud of fixed points and, without suffering any bifurcation, during training, these fixed points change position in the state space. These fixed points do not disappear during training, and we believe the last layer learns to weigh them differently and use the information to solve the task.

\subsection{Word-level language model}
\label{sec:word-level-language}

\begin{figure}[t]
  \centering
  \vspace*{-0.2cm}
  \begin{tikzpicture}
  \begin{axis}[
  height=0.25\textwidth,
  width=0.45\textwidth,
  xlabel=epochs,
  ylabel= perplexity,
  legend entries={LSTM, sLSTM, oRNN},
  xmin=0, xmax=200,
  ymin=0, ymax=540
  ]

    \addplot[
    color=red,
    no marks,
    line width=1.5pt]
    table [x=Epoch,
    y=Valid PPL,
    col sep=comma] {./data/word-level-lm/history-wikitext2-lstm.csv};\

    \addplot[
    color=blue,
    no marks,
    line width=1.5pt]
    table [x=Epoch,
    y=Valid PPL,
    col sep=comma] {./data/word-level-lm/history-wikitext2-slstm.csv};\
2
    \addplot[
    color=asparagus,
    no marks,
    line width=1.5pt]
    table [x=Epoch,
    y=Valid PPL,
    col sep=comma] {./data/word-level-lm/history-wikitext2-ornn.csv};\
  \end{axis}
\end{tikzpicture}
  \vspace*{-0.3cm}
  \caption{\textbf{Word-level language model training history.} Perplexity on validation data per epoch.}
  \vspace*{-0.5cm}
  \label{fig:wikitext2-history}
\end{figure}
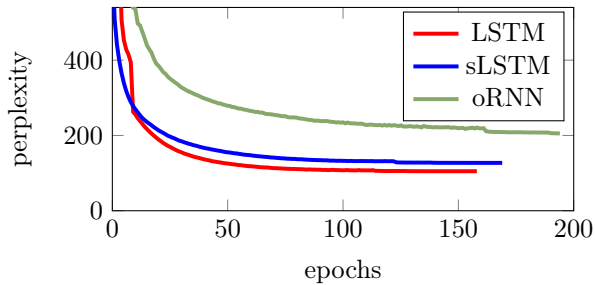

We train a language model on the openly available dataset Wikitext-2. This dataset contains 600 Wikipedia articles for training (2,088,628 tokens), 60 articles for validation (217,646 tokens), and 60 articles for testing (245,569 tokens)~\citep{Merity2017}. The dataset has a vocabulary of 33,278 distinct tokens. The goal is to predict the next token in the article given the previous tokens. The state-of-the-art result for the Wikitext-2 data set when not extending the training set is a perplexity (lower is better) on the test set of 39.14~\citep{Gong2018}.

Since our implementation of the orthogonal RNN is restricted to a single layer, we restrict all the models to a single layer. Despite this our LSTM model achieves a perplexity of 99.2 on the test set, which is close to other results reported in the literature (i.e.~99.3 obtained for a standard LSTM model in~\citet{Grave2017}). The stable LSTM has a perplexity of 118.8 (which is close to~\citet{miller_stable_2019}, where they achieve a perplexity of 113.2) and the orthogonal LSTM has a perplexity of 185.3.

We trained a model that consists of an embedding layer outputting a size 800 vector to a recurrent unit (LSTM, stable LSTM, or orthogonal RNN) with hidden size 800, followed by a softmax decoder layer that outputs probabilities for each token in the vocabulary. The last layer and the embedding have tied layers. A dropout layer is included before and after the LSTM layer with a dropout rate of 0.5. The models are trained and evaluated using a cross-entropy loss. The model is trained using ADAM with betas (0.9, 0.999). The learning rate starts at $10^{-3}$ and is divided by 10 when the cost function plateaus for more than 7 epochs. The training is stopped when the learning rate drops below $10^{-7}$. The gradient norm is clipped when it exceeds 0.25. The batch size is 100, and the gradient is not backpropagated for a sequence length larger than 80.

Figure~\ref{fig:wikitext2-orbit-diagram} shows bifurcation diagrams for this task. The problem has a high-dimensional output, and the average of all outputs, i.e., $\bar x_t$, is used to condense the information. The projection is extended into two dimensions by including the first difference for the case of the LSTM bifurcation diagram visualization. Other projections, other then the average of internal states, yield similar qualitative results and are displayed in Supplementary Figures~\ref{fig:wikitext2-lstm-orbit-diagram},~\ref{fig:wikitext2-slstm-orbit-diagram} and~\ref{fig:wikitext2-ornn-orbit-diagram}. The projection in the direction of the tokens ``is'',  ``<unk>'', and ``Valkyria'', which are representative of both frequent and infrequent tokens,  are used in this case. The bifurcation diagram for different random initial states and different constant inputs is also displayed in the supplementary material.

When fed with constant input, LSTM presents a wide range of different qualitative behavior: from converging to a single fixed point, as in Figure~\ref{fig:wikitext2-orbit-diagram}(a), to going through bifurcations that result in chaotic or periodic behavior, as in Supplementary Figure~\ref{fig:wikitext2-lstm-orbit-diagram}(b). When fed with constant input, orthogonal RNNs, on the other hand, converge to a single fixed point. This was verified for several random seeds, and we show a representative example in Figure~\ref{fig:wikitext2-orbit-diagram}(c). Qualitatively, this behavior is quite similar to the stable LSTM, which also has single fixed-point, by construction.

We also show the attractors for when the input, at instant $t$, is the word predicted with the highest probability at the previous time instant, $t-1$. This scenario corresponds to using the model for sentence generation. LSTM presents, again, chaotic attractors, which are quite desirable here, since periodic behavior or a single attractor correspond to repeating a few words or only a single word after sufficient interactions, which would not generate interesting sentences. Both orthogonal RNNs and stable LSTMs present this type of undesirable behavior. Not to say they are the same, in our experiments, sLSTM usually takes more than 600 tokens to reach this steady-state periodic behavior (after <eos> most of the time) while for the oRNN this sometimes happens before 100 tokens. Also, sLSTM converges to a richer periodic behavior, which oscillates between more points (compare Figure~\ref{fig:wikitext2-orbit-diagram} (d) and (f)). This example shows how the stable LSTM, during inference, might present richer steady-state behavior due to the feedback connections.

The training history is displayed in Figure~\ref{fig:wikitext2-history}. In the previous examples, the LSTM training was unsteady and full of jumps (and also highly sensitive to initial conditions, cf. Supplementary Figure~\ref{fig:history-temporal-order}); here, the curve is very smooth.
Chaotic and periodic attractors in the constant input scenario suggest that training the LSTM goes into the region of the parameter space that yields non-contractive models. An experiment presented by~\citet[sec. 2.1]{laurent_recurrent_2017} shows that this chaotic behavior, however, has limited influence in the output. This might help explain why the training of the LSTM model progresses so smoothly. It might also help in explaining the good performance of the stable LSTM model.

\section{Discussion}

We believe the refinements presented above might be relevant in developing new and more capable optimization and regularization methods for RNNs and to understand the limitations and strengths of the many techniques that are already in use. 

Constraining the RNN eigenvalues to be smaller than one yields exponentially decreasing  transient behavior, which in turn gives rise to vanishing gradients. When the training and inference mode are the same, this constraint also has the effect of not allowing attractors to appear. Much of the current RNN literature, however, does not distinguish between the two effects, which precludes richer analysis and an improved understanding of the training mechanism.

It is also important to draw a clear distinction between the RNN behavior during training and inference. As discussed in Section~\ref{sec:rnn-information}, attractors are needed to store information in a non-volatile way. When the training and inference mode are the same, learning attractors (other than a single fixed point) require the optimization procedure to explore regions of the parameter space where the model is non-contractive. As discussed in Section~\ref{sec:smoothness}, the smoothness properties of the objective function might be very poor in these regions, which is typically very problematic for traditional optimization procedures.

The symbol classification task studied in this paper is very similar to other artificial benchmarks used to assess the ability of recurrent models to \textit{remember} information seen many time steps before. Other examples include the adding and multiplication problems~\citep{hochreiter_long_1997} and the copy memory task~\citep{arjovsky_unitary_2016}. All these problems are useful benchmarks in assessing the model's ability to learn attractors and explore regions of the parameter space where the model is non-contractive. The results in this paper seem to reinforce the effectiveness of oRNN models in solving this class of problems.

Artificial benchmarks are usually formulated in a way that makes it hard or even impossible to use teacher forcing. State-of-the-art models for many relevant tasks, such as word language modeling, however, employ such a mechanism. Teacher forcing introduces the appealing possibility of using different modes for training and inference.
In this case, the system output is fed back as an input, allowing new dynamical behavior during inference.  Hence, even if the training is limited to regions of the parameter space where the dynamic system is contractive, it might still be possible to have attractors (and hence long-term memory) during inference. 

Different training and inference behavior might help explain why the stable model by~\citet{miller_stable_2019} is successful in many tasks. On the one hand, training the model in the unstable region is challenging, so ruling out this region of the parameter space from the optimization procedure might not be so useful. On the other hand, the model is not necessarily stable during inference since the output is fed back as an input. A similar argument might also justify the good performance of the RNN model without chaotic behavior proposed by~\citet{laurent_recurrent_2017}.

\section{Related work}
\citet{sussillo_opening_2013} and~\citet{maheswaranathan_reverse_2019} give examples analyzing attractors in an RNN. However, they focus on the situation after training, while we use the attractors for the analysis of the training procedure. 

The understanding of exploding and vanishing gradients attributed  to~\citep{hochreiter_long_1997, pascanu_difficulty_2013, bengio_learning_1994}. \citet{pascanu_difficulty_2013} formulates exploding and vanishing gradients in terms of the eigenvalues of the hidden-to-hidden weight matrix but do not consider the importance of attractors in retaining the information. \citet{bengio_learning_1994} define long-term memory in terms of single fixed points and present a trade-off between vanishing gradients and retaining information robustly in the presence of noise. We extend the analysis by contemplating chaotic and periodic attractors and, also, the possibility of different training and inference modes.

\citet{pascanu_difficulty_2013} conjectures that the walls in the cost function are due to bifurcations. In our examples, we do not observe this connection between bifurcations and jumps in performance.

The notion of using entropy for a dynamical system is not new. One classical definition is that of the Kolmogorov-Sinai entropy~\citep{sinai_notion_1959}, which is related to how uncertainty (and information) increases with time in a chaotic attractor. However, this definition cannot be applied for regions of the parameter space that do not preserve volume, and hence the definition in Section~\ref{sec:rnn-information} serves our purpose better.

\textbf{Code availability}: The code for reproducing the numerical examples is available at: \href{https://github.com/antonior92/attractors-and-smoothness-RNN}{github.com/antonior92/attractors-and-smoothness-RNN}. 

\subsubsection*{Acknowledgments}
We thank Carl Jidling and Manoel Horta Ribeiro for insightful comments and suggestions. This work has been supported by the Brazilian agencies \textit{CAPES - Coordenacão de Aperfeiçoamento de Pessoal de N\'ivel Superior} (Finance Code: 001), \textit{CNPq - Conselho Nacional de Desenvolvimento Cient\'ifico e Tecnológico} (contract number: 302079/2011-4, 200931/2018-0 and 142211/2018-4) and \textit{FAPEMIG - Fundação de Amparo \`a Pesquisa do Estado de Minas Gerais} (contract number: TEC 1217/98), by the Swedish Research Council (VR) via the projects \emph{NewLEADS -- New Directions in Learning Dynamical Systems} (contract number: 621-2016-06079) and \emph{Learning flexible models for nonlinear dynamics} (contract number: 2017-03807), and by the Swedish Foundation for Strategic Research (SSF) via the project \emph{ASSEMBLE} (contract number: RIT15-0012)


\appendix

\twocolumn[
\vspace*{-0.5cm}
\aistatstitle{Supplementary material - \mytitle}
\vspace*{-0.5cm}
\aistatsauthor{\myauthors}
\aistatsaddress{\myaddress}
]
\thispagestyle{plain}


\renewcommand{\figurename}{Supplementary Figure}
\setcounter{figure}{0}
\setcounter{equation}{0}
\renewcommand\theequation{S\arabic{equation}}

\section{Computing the derivatives}
\label{sec:comp-deriv}

Let the Jacobian matrices of $\mathbf{f}(\mathbf{x}, \mathbf{u}; \boldsymbol{\theta})$  with respect to $\mathbf{x}$ and to $\boldsymbol{\theta}$ evaluated at the point $(\mathbf{x}_t, \mathbf{u}_t; \boldsymbol{\theta})$ be denoted, respectively, as $A_t$ and $B_t$. Similarly, the Jacobian matrices of $\mathbf{g}(\mathbf{x}, \mathbf{u}_t; \boldsymbol{\theta})$ are denoted as $C_t$ and $F_t$. A direct application of the chain rule to~(\ref{eq:nlss}) gives a recursive formula for computing the derivatives of the predicted output in relation to the parameters in the interval ${1 \le t  \le N}$:
\begin{eqnarray}
    \nonumber
  D_{t+1}
  &=& A_t D_t +B_t\text{ for }D_0 = \mathbf{0}\\
  \label{eq:sensitivity_eq}
  J_t &=& C_t D_t + F_t,
\end{eqnarray}
where we denote the Jacobian matrices of $\hat{\mathbf{y}}_t$ and $\mathbf{x}_t$ with respect to $\boldsymbol{\theta}$ respectively as $J_t$ and $D_t$.

For the cost function $V$ defined as in Eq.~(\ref{eq:cost_function}), its gradient $\nabla V$ is given by:
\begin{equation}
  \label{eq:grad_objective_fun}
  \nabla V = \tfrac{1}{N}\sum_{t=1}^{N}  J_t~ l'(\hat{\mathbf{y}}_t,\mathbf{y}_t),
\end{equation}
where $l'(\hat{\mathbf{y}}_t,\mathbf{y}_t)$ denotes the gradient of the loss function with respect to its first argument, evaluated at $(\hat{\mathbf{y}}_t,\mathbf{y}_t)$.

\section{Proofs}
\label{sec:proofs}

\subsection{Entropy lower bound}

\begin{theorem}
\label{thm:entropy}
  Let $\mathbf{f}(\cdot, \mathbf{u}_t; \boldsymbol{\theta})$ in Eq.~(\ref{eq:nlss}) be a one-to-one continuous differentiable map, and let $\mathbf{f}(\mathbf{x}, \mathbf{u}; \boldsymbol{\theta})$  be  Lipschitz in ${(\mathbf{x}, \boldsymbol{\theta})}$ with constant $L_f$ on a compact and convex set $\Omega = (\Omega_{\mathbf{x}},\Omega_{\mathbf{u}}, \Omega_{\boldsymbol{\theta}})$. Then the entropy $H_t$ in Eq.~\eqref{ref:entropy_state} with $x_t \in \mathbb{R}^{N_x}$ satisfies:
\begin{equation}
\label{eq:entropy_increment_upper_bound_2}
   H_{t} + N_x \log L_f  \leq H_{t+1}.
\end{equation}
\end{theorem}

\begin{proof}
Under the assumption that $ \mathbf{f}(\cdot, \mathbf{u}_t; \boldsymbol{\theta})$ is a 1-1 continuous differentiable map (cf.~Theorems 3-13 and 3-14 by~\citet{spivak_calculus_1998}), applying the change of variable $\mathbf{x}_{t+1} = \mathbf{f}(\mathbf{x}_t, \mathbf{u}_t; \boldsymbol{\theta})$ we get:
\begin{scriptsize}
\begin{align*}
    H_t &=- \int_{\mathbf{f}(\Omega_x, \mathbf{u}_t; \boldsymbol{\theta})} p_{t+1}(\mathbf{x}_{t+1}) \log \left(p_{t+1}(\mathbf{x}_{t+1}) \left|\det \frac{\partial \mathbf{x}_{t+1}}{\partial \mathbf{x}_{t}}\right|\right) {\rm d} \mathbf{x}_{t+1}\\
    &= H_{t+1} - \int_{\mathbf{f}(\Omega_x, \mathbf{u}_t; \boldsymbol{\theta})} p_{t+1}(\mathbf{x}_{t+1}) \log \left(\left|\det \frac{\partial \mathbf{x}_{t+1}}{\partial \mathbf{x}_{t}}\right|\right) d \mathbf{x}_{t+1},
\end{align*}
\end{scriptsize}

\noindent
where $\frac{\partial \mathbf{x}_{t+1}}{\partial \mathbf{x}_{t}}$ is the Jacobian matrix of $\mathbf{f}(\cdot, \mathbf{u}_t; \boldsymbol{\theta})$. Using Hadamard's inequality:
\begin{equation}
   \log \left|\det \frac{\partial \mathbf{x}_{t+1}}{\partial \mathbf{x}_{t}}\right| \leq \sum_{i=1}^{N_x} \log \|\mathbf{v}_i\|_2,
\end{equation}

\noindent
where $\mathbf{v}_i$ is the $i$-th column of the Jacobian matrix and  $\log\|\mathbf{v}_i\|_2 \le \log\left\|\frac{\partial \mathbf{x}_{t+1}}{\partial \mathbf{x}_{t}}\right\|_2 \le \log L_f$. Hence, we have that
\begin{equation}
   H_{t} + N_x \log L_f  \leq H_{t+1}.
\end{equation}

\end{proof}

\subsection{Preliminary results}

\begin{lem}
\label{thm:lipschtz_sum_prod}
For $i = 1,\cdots, n$, let $\mathbf{f}_i$ be a Lipschitz function on~$\Omega$ with constants $L_i$. Then,
\begin{enumerate}
    \item $\sum_{i=1}^n\mathbf{f}_i$ is also a Lipschitz function on~$\Omega$ with Lipschitz constant upper bounded by $\sum_{i=1}^n L_i$;
    \item if, additionally, $\mathbf{f}_i$ are bounded by $M_i$ on~$\Omega$, then  $\prod_{i=1}^n\mathbf{f}_i$ is also a Lipschitz function on~$\Omega$ with Lipschitz constant upper bounded by $\left( \sum_{i=1}^n M_1\cdots M_{i-1}L_i M_{i+1}\cdots M_{n} \right)$.
\end{enumerate}
\end{lem}

\begin{lem}
\label{thm:condition_on_the_loss}
Let us define the properties:
\begin{enumerate}
    \item $|l(\hat{\mathbf{y}}, \mathbf{y})-l(\hat{\mathbf{z}}, \mathbf{y})| < \left(K_1 \|\mathbf{y}\| + K_2\max(\|\hat{\mathbf{y}}\|, \|\hat{\mathbf{z}}\|)\right)\|\hat{\mathbf{y}} - \hat{\mathbf{z}}\|$,
    \item $l'(\hat{\mathbf{y}}, \mathbf{y}) = \Psi(\hat{\mathbf{y}}) - K_3\mathbf{y} $,
\end{enumerate}
where $l'(\hat{\mathbf{y}},\mathbf{y})$ denotes the first derivative of the loss function with respect to its first argument, evaluated at $(\hat{\mathbf{y}},\mathbf{y})$.
There exist constants $K_1$, $K_2$, and $K_3$ and a function $\Psi$ that is Lipschitz continuous with constant $K_4$ and for which $\Psi$  such that these properties hold for both: a) $l(\hat{\mathbf{y}}, \mathbf{y}) = \|\hat{\mathbf{y}} - \mathbf{y}\|^2$; b) $l(\hat{\mathbf{y}}, \mathbf{y}) = -\mathbf{y}^T \log (\sigma(\hat{\mathbf{y}})) - (1-\mathbf{y})^T \log (1 - \sigma(\hat{\mathbf{y}}))$. In (b), the sigmoid function, $\sigma(x) = \frac{1}{1+\exp(-x)}$, and the logarithm are evaluated element-wise. We assume in (b) that the elements in $\mathbf{y}$ are either 0 or 1.

\end{lem}

\begin{proof}
For (a) and (b), property 2 follows from differentiating $l(\hat{\mathbf{y}}, \mathbf{y})$ with respect to its first argument. For (a), $\phi(\hat{\mathbf{y}}) = 2\hat{\mathbf{y}}$ and $K_3=2$; for (b),  $\phi(\hat{\mathbf{y}}) = \sigma(\hat{\mathbf{y}})$ and $K_3=1$.

For loss function (a), property 1 holds due to the following inequalities 
\begin{footnotesize}
 \begin{align*}
     \left|\|\hat{\mathbf{y}}- \mathbf{y}\|^2-\|\hat{\mathbf{z}} - \mathbf{y}\|^2\right| = \left|\|\hat{\mathbf{y}}\|^2-\|\hat{\mathbf{z}}\|^2 -2\mathbf{y}^T\left(\hat{\mathbf{y}} - \hat{\mathbf{z}}\right)\right| \leq &
     \\
     \leq 
     \left|\left(\|\hat{\mathbf{y}}\|-\|\hat{\mathbf{z}}\|\right)\left(\|\hat{\mathbf{y}}\|+\|\hat{\mathbf{z}}\|\right) -2\mathbf{y}^T\left(\hat{\mathbf{y}} - \hat{\mathbf{z}}\right)\right| \leq &
     \\
     \leq 
     \left(2\|\mathbf{y}\| + 2\max\left(\|\hat{\mathbf{y}}\|, \|\hat{\mathbf{z}}\|\right)\right)  \|\hat{\mathbf{y}} - \hat{\mathbf{z}}\|. &
\end{align*}
\end{footnotesize}

\noindent
For loss function (b), let $\hat{y}$ and $\hat{z}$ be two scalar values. Furthermore, consider, without loss of generality, that $\hat{y}\geq \hat{z}$. Then:
\begin{footnotesize}
\begin{equation}
    \label{eq:ineq_cross_entropy_1}
    0 \leq \log(\sigma(\hat{y})) - \log(\sigma(\hat{z})) = (\hat{y} - \hat{z}) - \log\left(\frac{\exp(\hat{y}) + 1}{\exp(\hat{z}) + 1} \right) \leq (\hat{y} - \hat{z})
\end{equation}
\end{footnotesize}
The first inequality follows from the fact that $\log(\sigma(\cdot))$ is a monotonically increasing function. The last inequality holds because $\log\left(\frac{\exp(\hat{y}) + 1}{\exp(\hat{z}) + 1} \right)\geq 0$. Analogously, 
\begin{equation}
\begin{aligned}
    \label{eq:ineq_cross_entropy_2}
    0 & \leq \log(1-\sigma(\hat{z})) - \log(1-\sigma(\hat{y}))\\& = (\hat{y} - \hat{z}) - \log\left(\frac{\exp(-\hat{z}) + 1}{\exp(-\hat{y}) + 1} \right)\\& \leq (\hat{y} - \hat{z}).
\end{aligned}
\end{equation}

For  $l(\mathbf{y}, \hat{\mathbf{y}})$ defined as in (b),  it follows from Eq.~(\ref{eq:ineq_cross_entropy_1}), Eq.~(\ref{eq:ineq_cross_entropy_2}), and the fact that $\mathbf{y}$ contains only values in the set $\{0, 1\}$, that:
\begin{equation}
    |l(\hat{\mathbf{y}}, \mathbf{y})-l(\hat{\mathbf{z}}, \mathbf{y})| \leq \|\hat{\mathbf{y}} - \hat{\mathbf{z}} \|_1 \leq \sqrt{N_y}\|\hat{\mathbf{y}} - \hat{\mathbf{z}}\|_2
\end{equation}
where $\|\cdot\|_1$ and $\|\cdot\|_2$ denote $l_1$ and $l_2$ norm of a vector and  $N_y$ is the number of outputs.
\end{proof}

\subsection{Proof of Theorem \ref{thm:lipschitz} (a)}

Assume two different trajectories resulting from simulating the system in Eq~(\ref{eq:nlss}) with parameters and initial conditions $(\mathbf{x}_0, \boldsymbol{\theta})$ and $(\mathbf{w}_0, \boldsymbol{\phi})$, respectively. We denote the corresponding trajectories by $\mathbf{x}$ and $\mathbf{w}$. Let us call:
\begin{equation}
    \|\Delta \hat{\mathbf{y}}_t\| = \| \mathbf{g}(\mathbf{x}_t, \mathbf{u}_t; \boldsymbol{\theta}) - \mathbf{g}(\mathbf{w}_t, \mathbf{u}_t; \boldsymbol{\phi})\| .
\end{equation}

\noindent
Since $\mathbf{f}$ and $\mathbf{g}$ are Lipschitz in ${(\mathbf{x}, \boldsymbol{\theta})}$ we have:
\begin{eqnarray*}
\|\mathbf{f}(\mathbf{x}, \mathbf{u}_t, \boldsymbol{\theta}) - \mathbf{f}(\mathbf{w},  \mathbf{u}_t, \boldsymbol{\phi})\|^2 \le L_f^2 \left(\|\mathbf{x} - \mathbf{w}\|^2 + \|\boldsymbol{\theta} - \boldsymbol{\phi}\|^2 \right), \\
\|\mathbf{g}(\mathbf{x}, \mathbf{u}_t, \boldsymbol{\theta}) - \mathbf{g}(\mathbf{w}, \mathbf{u}_t, \boldsymbol{\phi})\|^2 \le L_g^2 \left(\|\mathbf{x} - \mathbf{w}\|^2 + \|\boldsymbol{\theta} - \boldsymbol{\phi}\|^2 \right),
\end{eqnarray*}
for all $(\mathbf{x}, \mathbf{u}_t, \boldsymbol{\theta})$ and $(\mathbf{w}, \mathbf{u}_t, \boldsymbol{\phi})$ in $(\Omega_{\mathbf{x}},\Omega_{\mathbf{u}}, \Omega_{\boldsymbol{\theta}})$. Applying these relations recursively we get that:
\begin{equation*}
    \|\Delta \hat{\mathbf{y}}_t\|^2 \le L_g^2 L_f^{2t} \|\boldsymbol{x}_0 - \boldsymbol{w}_0 \|^2 + L_g^2 \left(\sum_{\ell=0}^t L_f^{2\ell}\right)\|\boldsymbol{\theta} - \boldsymbol{\phi}\|^2.
\end{equation*}

\noindent
Since $L_f$ is positive, the constant multiplying the second term in the above equation is always larger than the constant multiplying the first term. Hence, taking the square root on both sides of the above inequality and after simple manipulations, we obtain:
\begin{equation}
    \label{eq:ineq_dk}
    \|\Delta \hat{\mathbf{y}}_t\| \le  L_gS(t)\|[\boldsymbol{\theta}, \boldsymbol{x}_0]^T - [\boldsymbol{\phi}, \boldsymbol{w}_0]^T \|,
\end{equation}
where
\begin{equation}
    \label{eq:Sk}
    S(t) = \sqrt{\sum_{\ell=0}^t L_f^{2\ell}} =     
    \begin{cases}
    \sqrt{t + 1} & \text{ if }L_f = 1, \\
    \sqrt{\frac{L_f^{2t+2} - 1}{L_f^2 - 1}} & \text{ if }L_f \not= 1.
    \end{cases}
\end{equation}

Since $\Omega$ is compact and $\hat{\mathbf{y}}_t$ is a (Lipschitz) continuous function of the parameters and initial conditions, then $\hat{\mathbf{y}}_t$ is bounded in $\Omega$, i.e.~$\|\hat{\mathbf{y}}_t\| \le M(t)$. Furthermore, it follows from Eq.~(\ref{eq:ineq_dk}) and from the existence of an invariant set\footnote{There are multiple ways to guarantee the invariant set premise will hold, but a very simple way is to just choose $\mathbf{f}$ such that $\mathbf{f}(\mathbf{0}, \mathbf{u}_t; \mathbf{0}) = \textbf{0}$. In this case, $\{\mathbf{0}\}$ is an invariant set and if $\Omega_{\boldsymbol{\theta}}$ contains this point the premise is satisfied. For this specific case, one can just choose $[\boldsymbol{\phi}, \boldsymbol{w}_0] = \mathbf{0}$ and it follows from Eq.~(\ref{eq:ineq_dk}) that $\|\hat{\mathbf{y}}_t\| \le  L_gS(t)\|[\boldsymbol{\theta}, \boldsymbol{x}_0]\| = \mathcal{O}(S(t))$. The more general case, for any invariant set, follows from a similar deduction.} in $\Omega$ that $M(t) = \mathcal{O}(S(t))$.

The following inequality follows from Eq.~(\ref{eq:cost_function}), and from property~1 in  Lemma~\ref{thm:condition_on_the_loss}:
\begin{equation}
    \label{eq:ineq_V1}
    |V(\boldsymbol{\theta}, \boldsymbol{x}_0) - V(\boldsymbol{\phi}, \boldsymbol{w}_0)| \le  \tfrac{1}{N}\sum_{t=1}^{N}  (K_1 L_y + K_2 M(t)) \|\Delta \hat{\mathbf{y}}_t\|,
\end{equation}

\noindent
where ${L_y = \max_{1 \leq t \leq N} \|\mathbf{y}_t\|}$. Now, assembling Eq.~(\ref{eq:ineq_V1}) and Eq.~(\ref{eq:ineq_dk}) we obtain
\begin{equation*}
|V(\boldsymbol{\theta}, \boldsymbol{x}_0) - V(\boldsymbol{\phi}, \boldsymbol{w}_0)| \le \\
L_{V_1} \left\|[\boldsymbol{x}_0, \boldsymbol{\theta}]^T - [\boldsymbol{w}_0, \boldsymbol{\phi}]^T \right\|,
\end{equation*}

\noindent
for ${L_{V} = \left(\tfrac{L_g}{N}\sum_{t=1}^{N} (K_1L_y + K_2M(t))S(t)\right)}$. The asymptotic analysis of this expression with respect to $N$ yields Eq.~(\ref{eq:asymptotic_L}).

\subsection{Proof of Theorem \ref{thm:lipschitz} (b)}

It follows from Eq.~(\ref{eq:grad_objective_fun}), and from property~2 in Lemma~\ref{thm:condition_on_the_loss}, that:
\begin{multline}
    \label{eq:ineq_grad_V}
    \|\nabla V(\boldsymbol{\theta}, \boldsymbol{x}_0) - \nabla V(\boldsymbol{\phi}, \boldsymbol{w}_0)\| \\  \le      \tfrac{1}{N}\sum_{t=1}^{N}  K_3L_y \|\Delta J_t\|
    + \|\Delta (J_t \Psi\left(\hat{\mathbf{y}}_t)\right)\| ,
\end{multline}

\noindent
where we have used the notation $\Delta J_t$  to denote the difference between $J_t$ evaluated at $(\boldsymbol{\theta}, \boldsymbol{x}_0)$ and $(\boldsymbol{\phi}, \boldsymbol{w}_0)$. Analogously, $\Delta (J_t \Psi(\hat{\mathbf{y}}_t))$ denotes the difference between $J_t\Psi(\hat{\mathbf{y}}_t)$ evaluated at the two distinct points, where $\Psi$ is the Lipschitz continuous function with constant $K_4$ defined in Lemma~\ref{thm:condition_on_the_loss}.

From Eq.~(\ref{eq:sensitivity_eq}) it follows that:
\begin{small}
\begin{equation}
    \label{eq:close_formula}
    J_t = C_t \sum_{\ell=1}^{t} \left(\prod_{j=1}^{t-\ell} A_{t-j+1}\right) B_\ell  + F_t;
\end{equation}
\end{small}

\noindent
Since the Jacobian of $\mathbf{f}$ is Lipschitz with Lipschitz constant $L_f'$, it follows that:
\begin{eqnarray}
\|\Delta A_j\|^2 \le (L_f')^2\left(\|\mathbf{x}_j - \mathbf{w}_j\|^2 + \|\boldsymbol{\theta} - \boldsymbol{\phi}\|^2\right).
\end{eqnarray}
Using a procedure analogous to the one used to obtain Eq.~(\ref{eq:ineq_dk}), it follows that:
\begin{eqnarray}
\label{eq:ineq_aj}
\|\Delta A_j\| \le L_f' S(j)~~\|[\boldsymbol{\theta}, \boldsymbol{x}_0]^T - [\boldsymbol{\phi}, \boldsymbol{w}_0]^T \|,
\end{eqnarray}
where $S(j)$ is defined as in Eq.~(\ref{eq:Sk}). An identical formula holds for $B_j$ and a similar formula, replacing $L_f'$ with $L_g'$, holds for $C_j$ and $F_j$. 

Since $\mathbf{f}$ and $\mathbf{g}$ are Lipschitz with Lipschitz constants $L_f$ and $L_g$ it follows that $\|A_j\| \le L_f$, $\|B_j\| \le L_f$, $\|C_j\| \le  L_g$ and $\|F_j\| \le L_g$. Hence, it follows from Eq.~(\ref{eq:ineq_dk}), Eq.~(\ref{eq:close_formula}), Eq.~(\ref{eq:ineq_aj}) and the repetitive application of  Lemma~\ref{thm:lipschtz_sum_prod} that $\|\Delta J_t\|$ and $\|\Delta (J_t \Psi (\hat{\mathbf{y}}_t))\|$ 
are upper bounded by $\|[\boldsymbol{\theta}, \boldsymbol{x}_0]^T - [\boldsymbol{\phi}, \boldsymbol{w}_0]^T \|$ multiplied by the following constants:
\begin{eqnarray*}
\nonumber
L_{J, t} &=&\sum_{\ell=1}^t P(t, \ell) + L_g' S(t),\\
L_{J\hat{y}, t} &=& \sum_{\ell=1}^t Q(t, \ell)  +  T(t)S(t),
\end{eqnarray*}

\noindent
where $T(t) = K_4(L_g' M(t) + L_g^2)$ and:
\begin{eqnarray*}
    P(t, \ell) &=& L_f^{t-\ell}\left(L_gL_f'\sum_{j=\ell}^t S(j) +  L_fL_g'S(t)\right), \\
     Q(t, \ell) &=&  L_f^{t-\ell}\left(K_4 M(t)L_gL_f'\sum_{j=\ell}^t S(j) +  L_fT(t)S(t)\right).
\end{eqnarray*}

\noindent
Hence,
\begin{equation*}
    \|\nabla V(\boldsymbol{\theta}, \boldsymbol{x}_0) - \nabla V(\boldsymbol{\phi}, \boldsymbol{w}_0)\|  \le      L_{V}' \|[\boldsymbol{\theta}, \boldsymbol{x}_0]^T - [\boldsymbol{\phi}, \boldsymbol{w}_0]^T \|,
\end{equation*}

\noindent
where:
\begin{eqnarray*}
    L_{V}' = \tfrac{1}{N}\sum_{t=1}^{N} \left(K_3L_yL_{J, t} + L_{J\hat{y}, t}\right).
\end{eqnarray*}
Combining the above, the asymptotic analysis of $L_{V}'$ results in Eq.~(\ref{eq:asymptotic_L'}).

\section{Chaotic LSTM example}
\label{sec:chaotic-lstm-details}

An LSTM with zero input and without bias terms is considered:
\begin{align}
    c_t & = \underbrace{\sigma \left( W_\mathrm{hf} h_{t-1} \right)}_{\text{forget gate}} * c_{t-1} + \underbrace{\sigma \left( W_\mathrm{hi} h_{t-1} \right)}_{\text{input gate}} * \underbrace{\tanh  \left( W_\mathrm{hg} h_{t-1} \right)}_{\text{cell gate}} \\
    h_t & = \underbrace{\sigma \left( W_\mathrm{ho} h_{t-1} \right)}_{\text{output gate}} * \tanh(c_t).
\end{align}
The sigmoids $\sigma(\cdot)$ and hyperbolic tangents  $\tanh(\cdot)$ operate element-wise. The symbol $*$ indicates an element-wise product. The hidden and cell state have initial conditions $h_0 = c_0 = \begin{bmatrix} 0.5 & 0.5 \end{bmatrix}^\mathsf{T}$. The hidden state $h_t$ is also the output of the model.
The weight matrices are put equal to $W_\mathrm{hi} = \begin{bmatrix} -1 &  4 \\  -3 & -2 \end{bmatrix}$, $W_\mathrm{hf} = \begin{bmatrix} -2 &  6 \\  0 & -6 \end{bmatrix}$, $W_\mathrm{hg} = \begin{bmatrix} -1 & -6 \\  6 & -9 \end{bmatrix}$, and $W_\mathrm{ho} = \begin{bmatrix}  4 &  1 \\ -9 &  7 \end{bmatrix}$ to generate the data. The values in the weight matrices are stacked on top of each other in a parameter vector $\theta_{\text{true}}$.

Figure~\ref{fig:chaotic_LSTM_2d} shows the mean-square loss evaluated on data generated by the same LSTM model with the same initial conditions, but with different parameter values. A two-dimensional grid of parameter values $\theta(s_1, s_2)$ is used with values interpolated (and extrapolated) between $\theta_\mathrm{true}$, zero parameter values, and a randomly chosen parameter vector $\theta_\mathrm{random}$, i.e.~in this case, $\theta(s_1, s_2) = s_1 \theta_\mathrm{true} + s_2 \theta_\mathrm{random}$. Again, the region in the parameter space around $\theta_\mathrm{true}$ is intricate. There is an entire region where the cost function is intricate and has many neighboring local minima.

\begin{figure*}[h]
\section{Additional Experiments}
\label{sec:additional-experiments}
\end{figure*}

\begin{figure*}[h]
    \centering
    \includegraphics[width=0.8\textwidth]{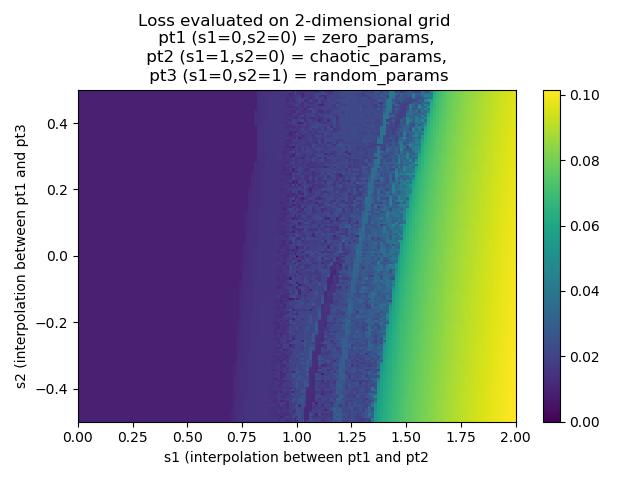}
    \caption{\textbf{Chaotic LSTM.} Mean-square cost function \textcolor{black}{(\ref{eq:cost_function})} for LSTM models with parameter vectors $\theta(s_1,s_2) = s_1 \theta_\mathrm{true} + s_2 \theta_\mathrm{random}$.}
    \label{fig:chaotic_LSTM_2d}
\end{figure*}

\begin{figure*}[htp]
  \centering
  \subfloat[good performance solution]{\includegraphics[height=0.21\textwidth]{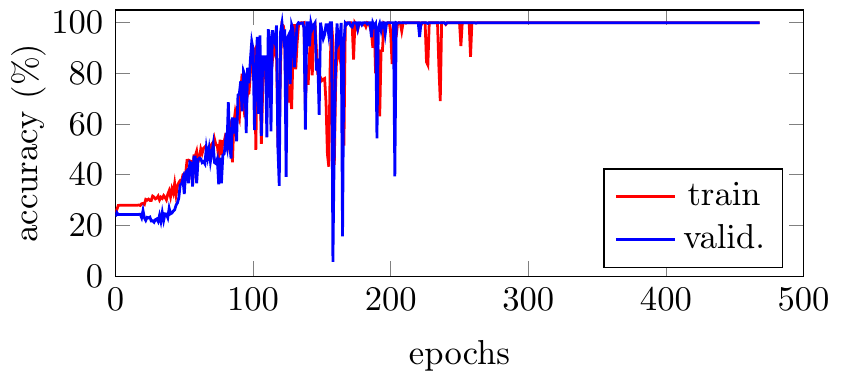}}
  \subfloat[bad performance solution]{\includegraphics[height=0.21\textwidth]{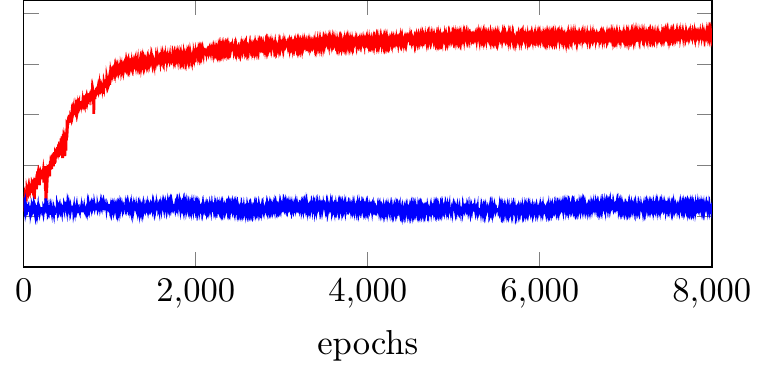}}
  \caption{\textbf{The effect of the initial conditions.} Accuracy on training and validation data on the \textbf{symbol classification task} for the \textbf{LSTM} model in identical scenarios but with different initial random parameter initialization (from different random seeds). In (a), the optimization procedure abruptly finds a solution that has good accuracy on both training and validation; while, in (b), the convergence is slow and steady towards a solution that has good accuracy on the training set (above $90\%$) but is no better than random guessing on the validation set.}
  \label{fig:history-temporal-order}
\end{figure*}

\begin{figure*}[htp]
  \centering
  \begin{tabular}{rcccc}
    & $\{p, p\}$ & $\{p, q\}$ & $\{p, p\}$  & $\{q, q\}$ \\
    $p$ & \includegraphics[width=0.2\textwidth]{./img/temporal-order-orbit-lstm/orbit-input0-output0-sequence0} & \includegraphics[width=0.2\textwidth]{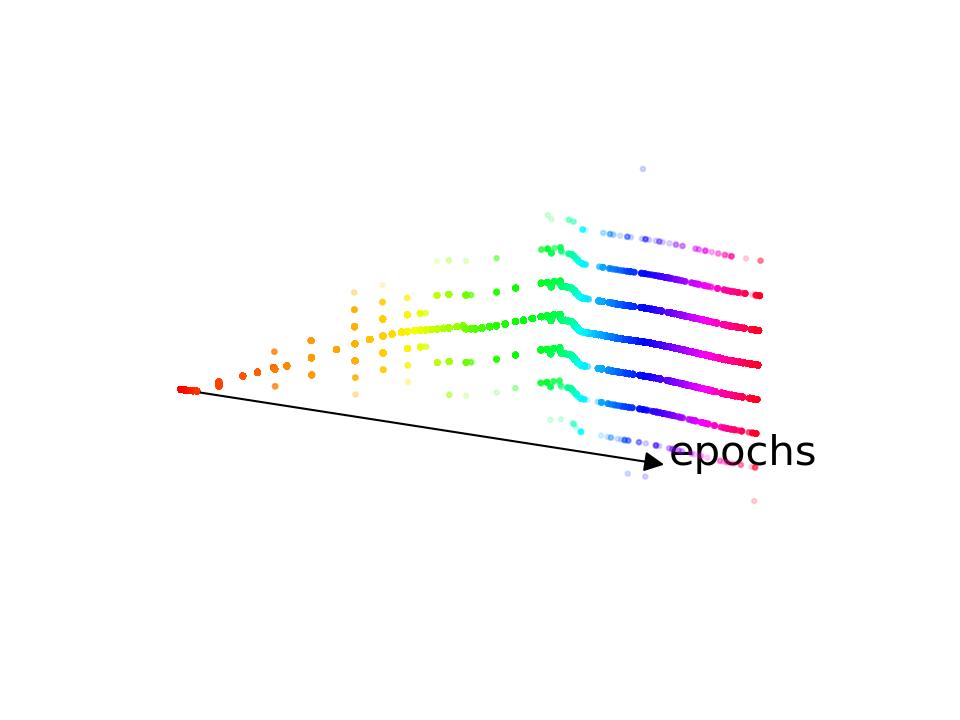} &\includegraphics[width=0.2\textwidth]{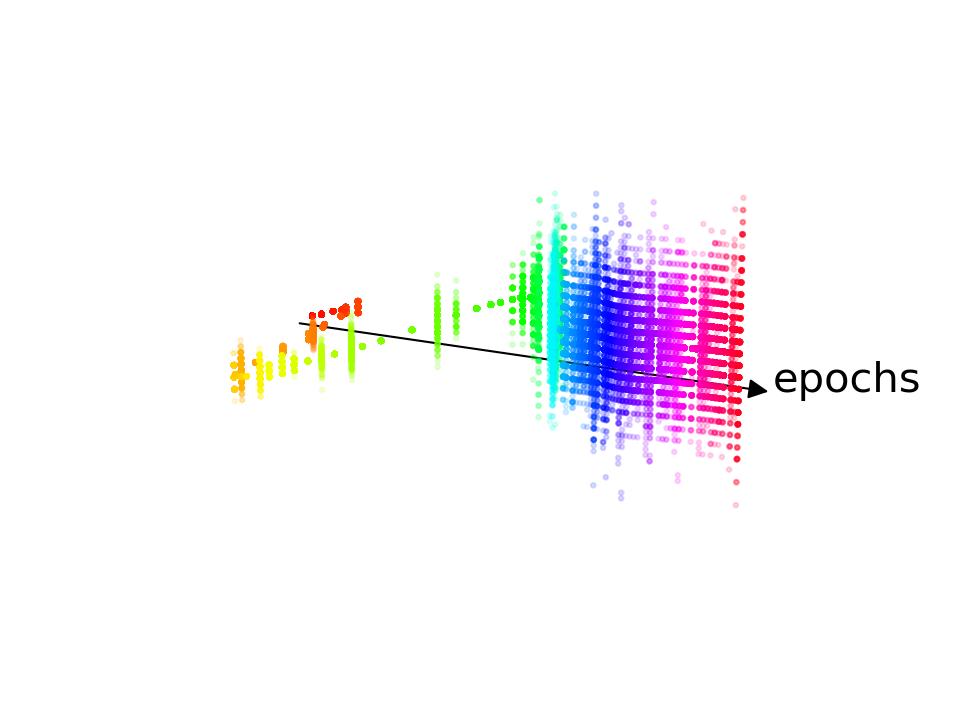} & \includegraphics[width=0.2\textwidth]{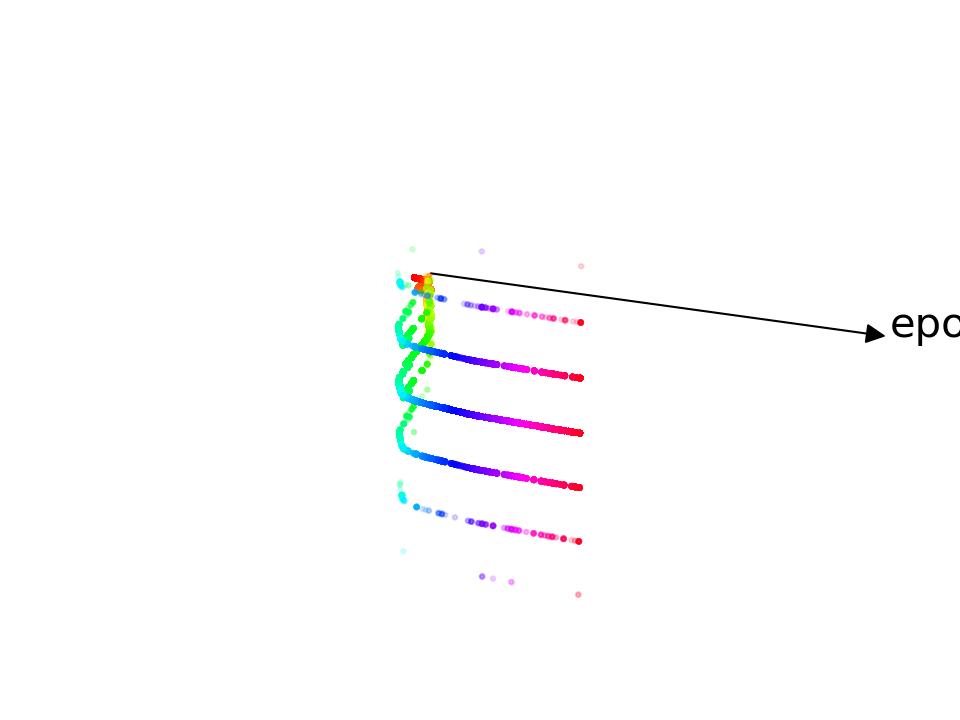}\\
    $q$ & \includegraphics[width=0.2\textwidth]{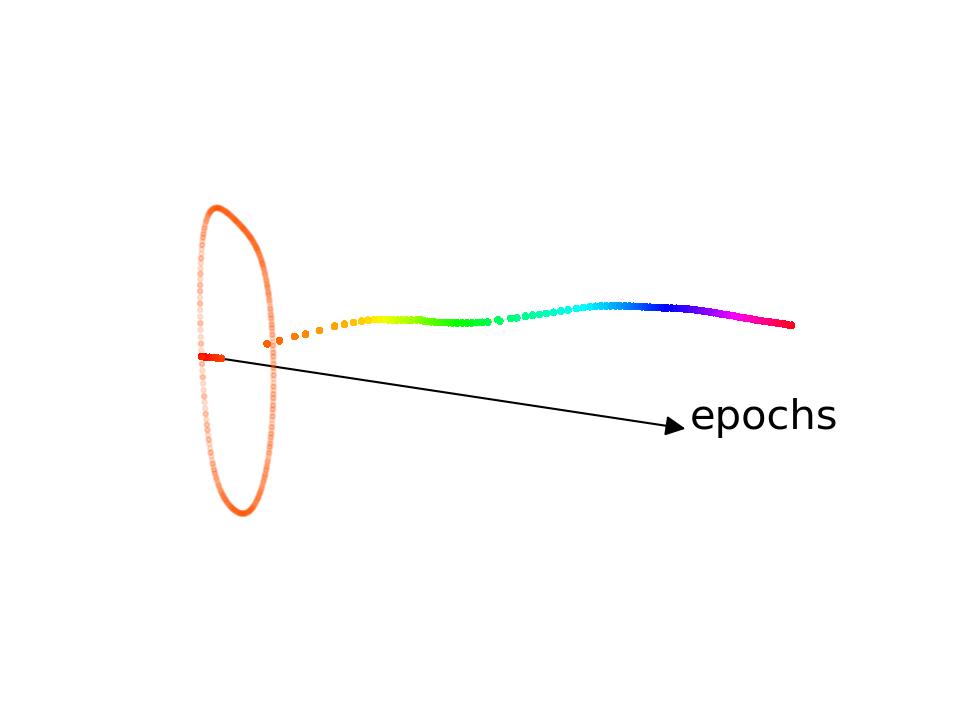} & \includegraphics[width=0.2\textwidth]{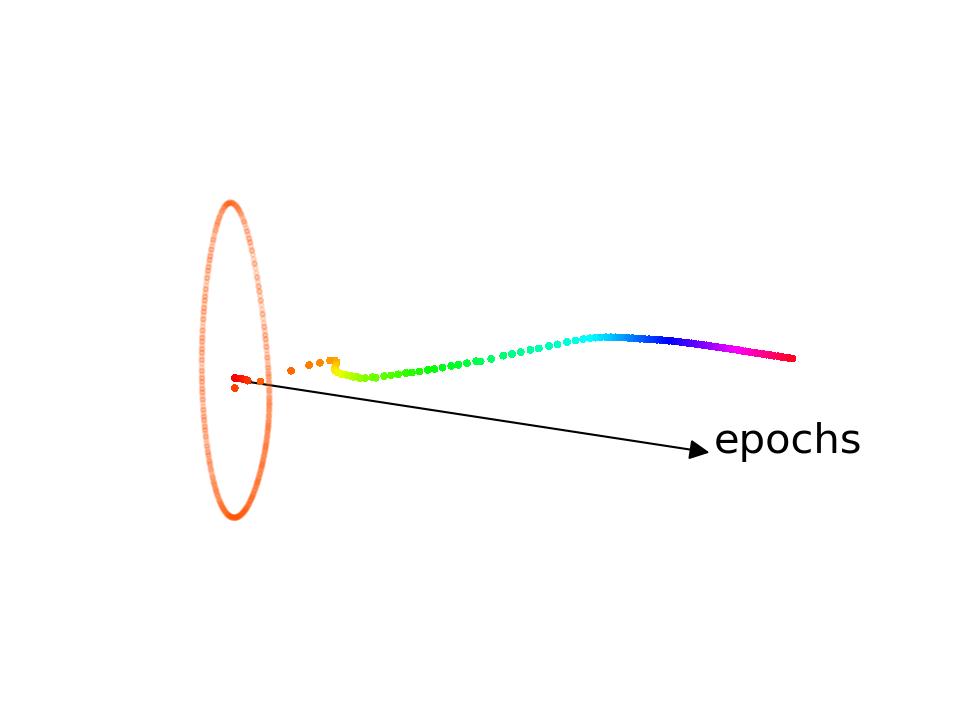} &\includegraphics[width=0.2\textwidth]{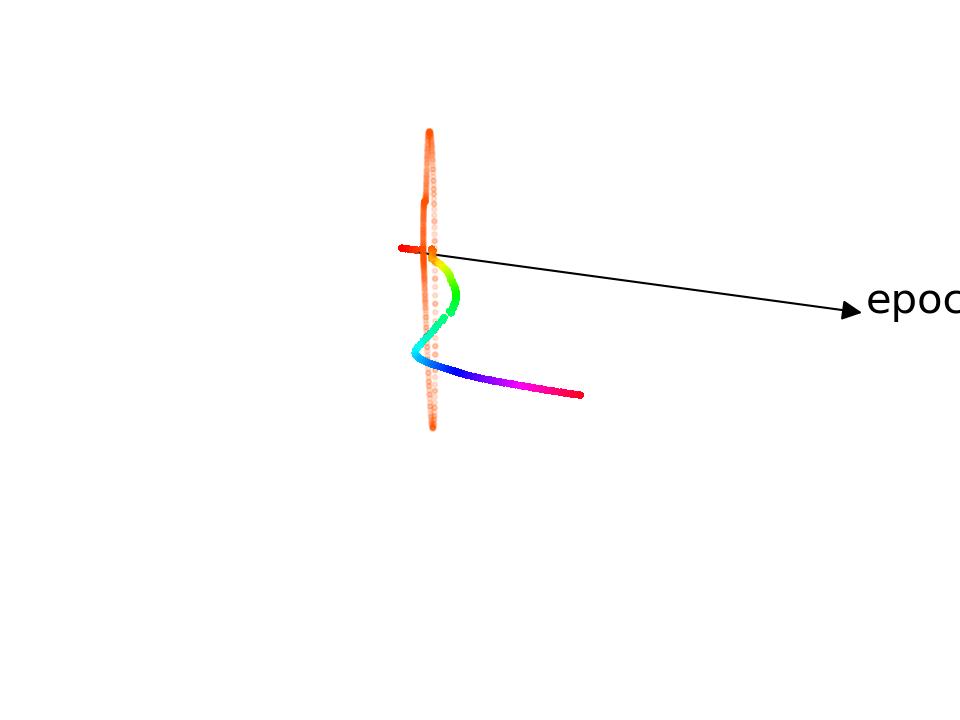} & \includegraphics[width=0.2\textwidth]{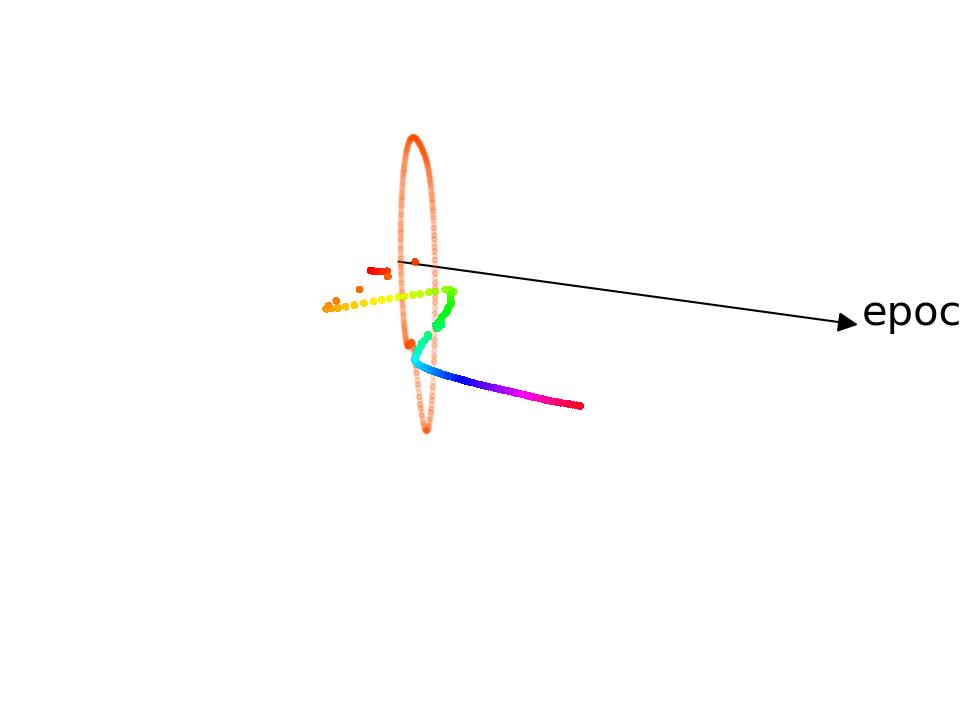}\\
    $a$ & \includegraphics[width=0.2\textwidth]{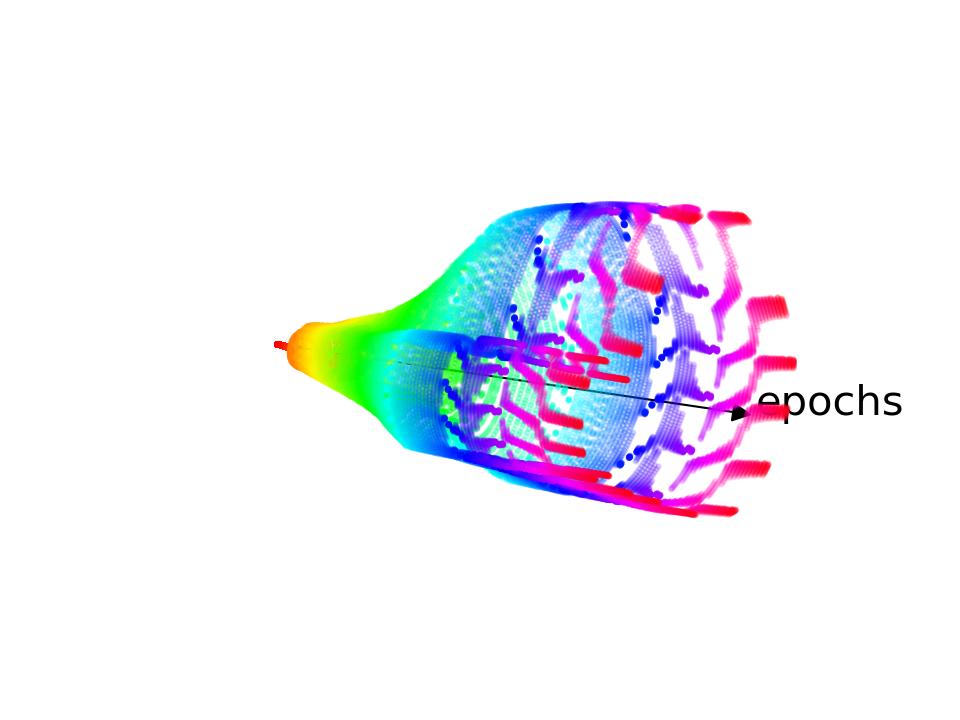} & \includegraphics[width=0.2\textwidth]{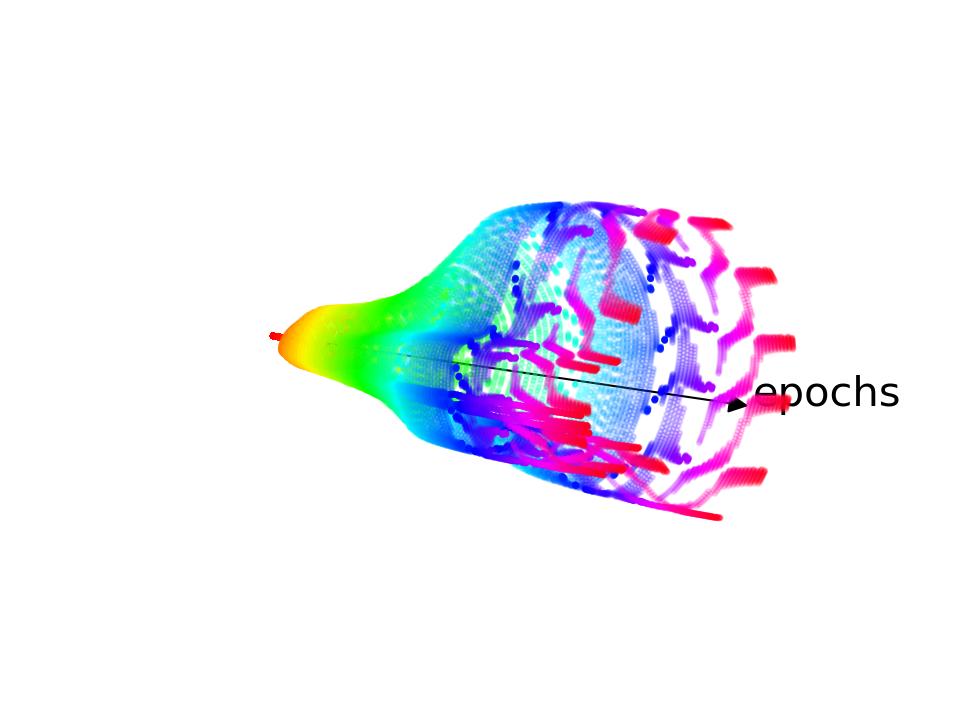} &\includegraphics[width=0.2\textwidth]{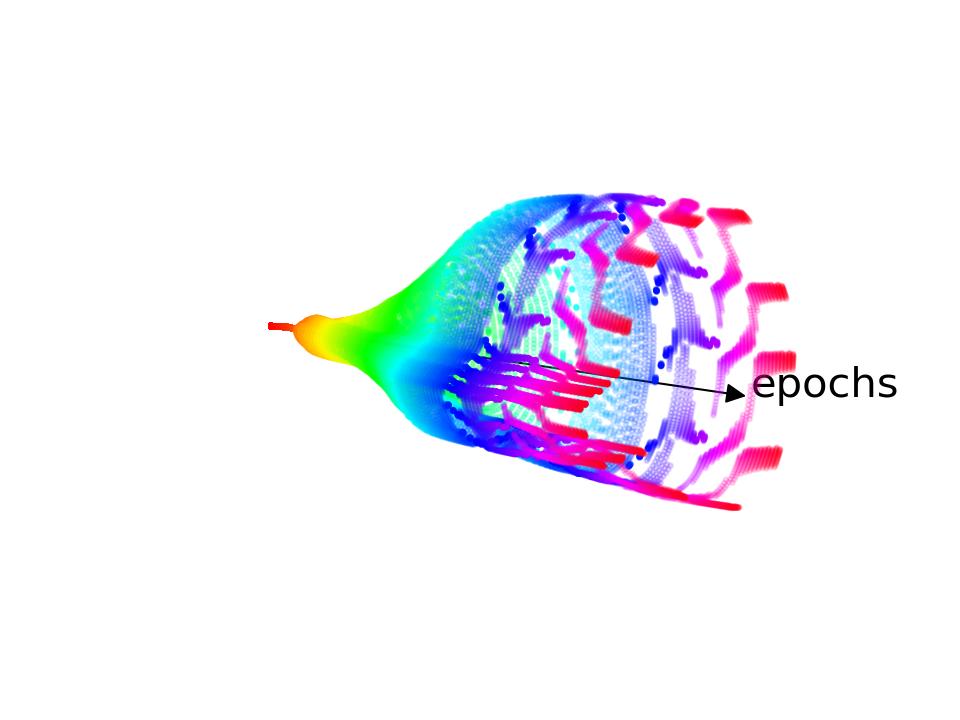} & \includegraphics[width=0.2\textwidth]{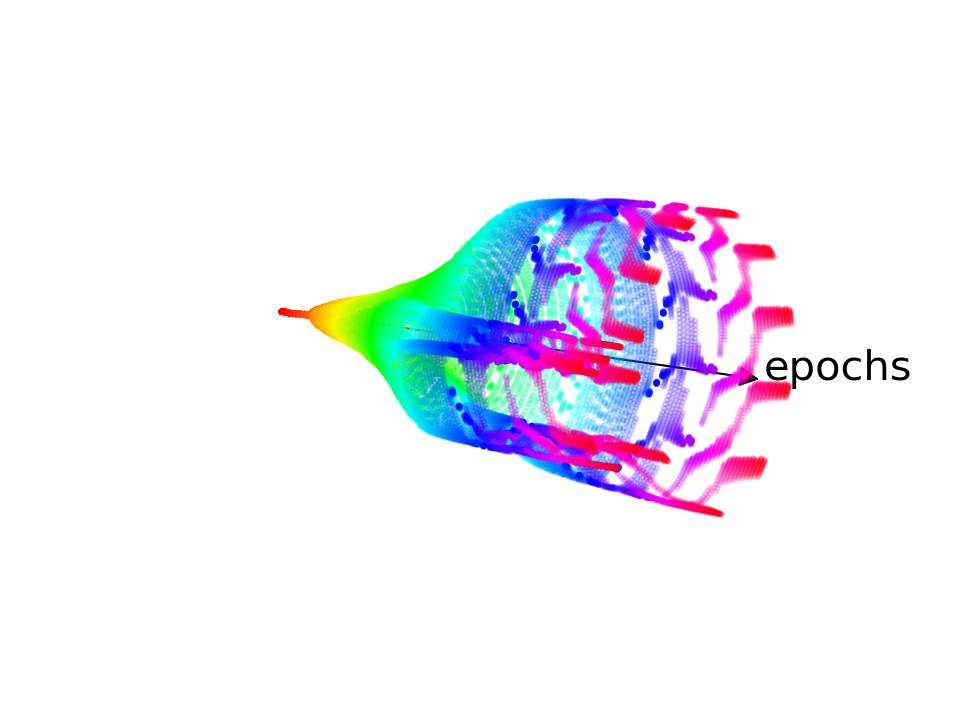}\\
    $b$ & \includegraphics[width=0.2\textwidth]{./img/temporal-order-orbit-lstm/orbit-input3-output0-sequence0} & \includegraphics[width=0.2\textwidth]{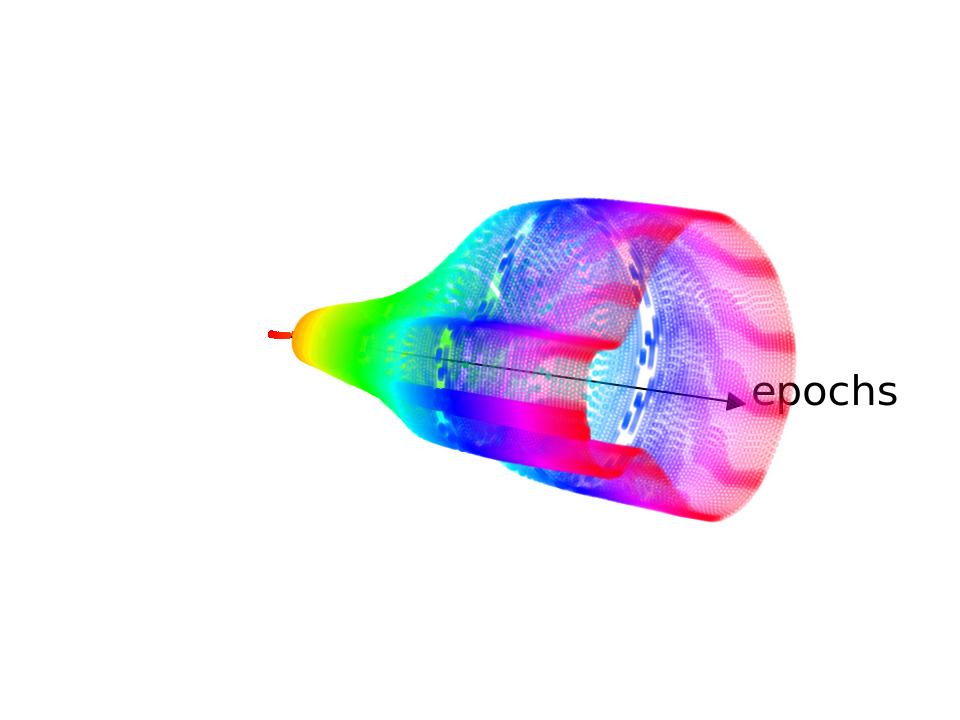} &\includegraphics[width=0.2\textwidth]{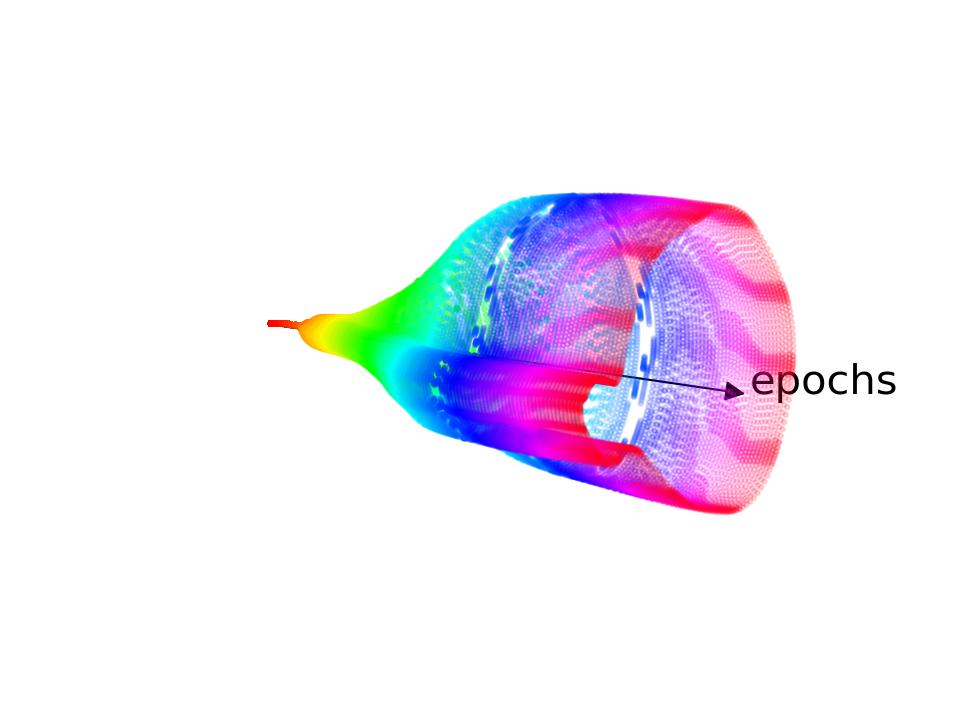} & \includegraphics[width=0.2\textwidth]{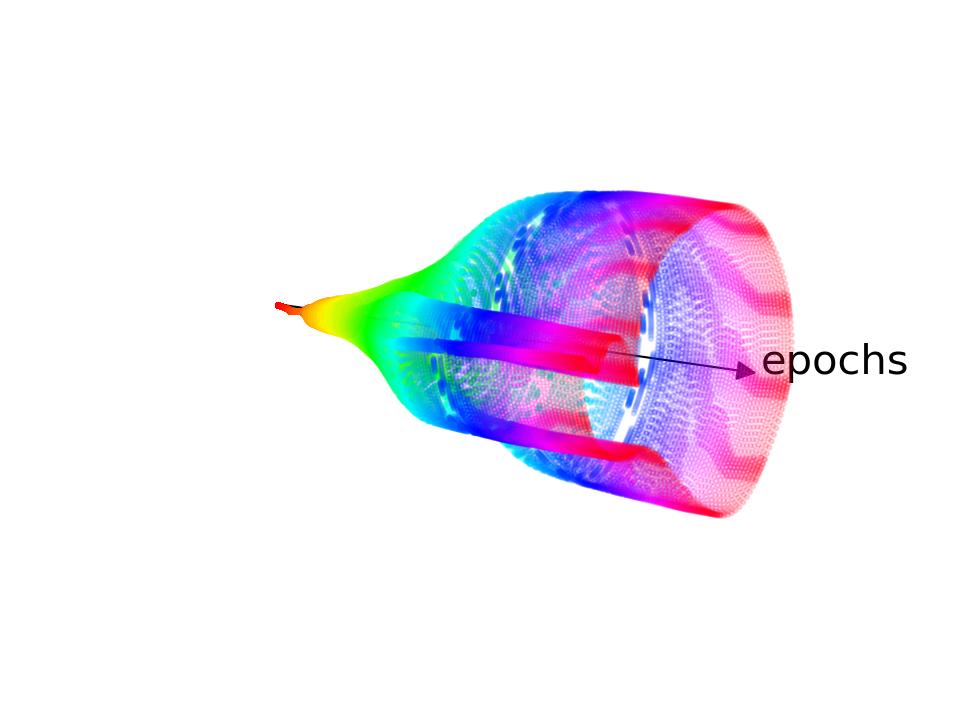}\\
    $c$ & \includegraphics[width=0.2\textwidth]{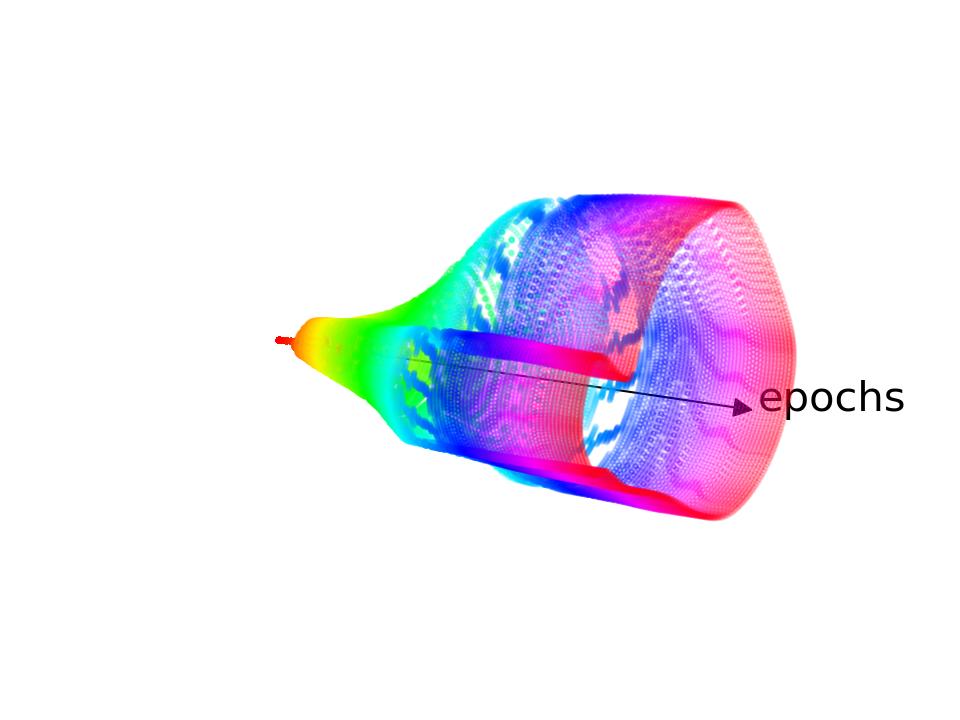} & \includegraphics[width=0.2\textwidth]{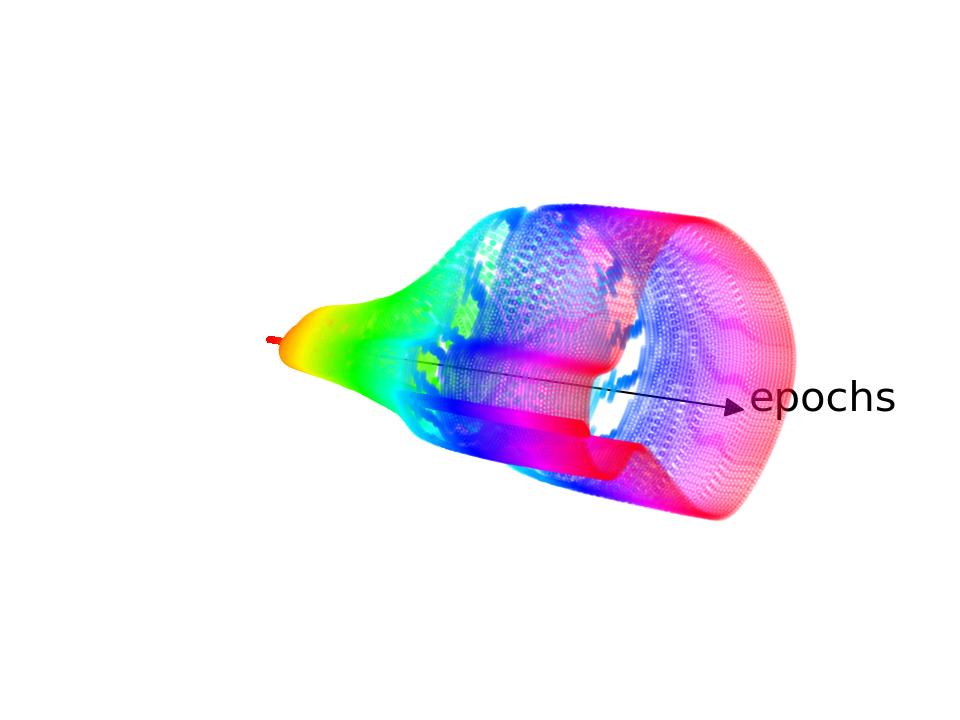} &\includegraphics[width=0.2\textwidth]{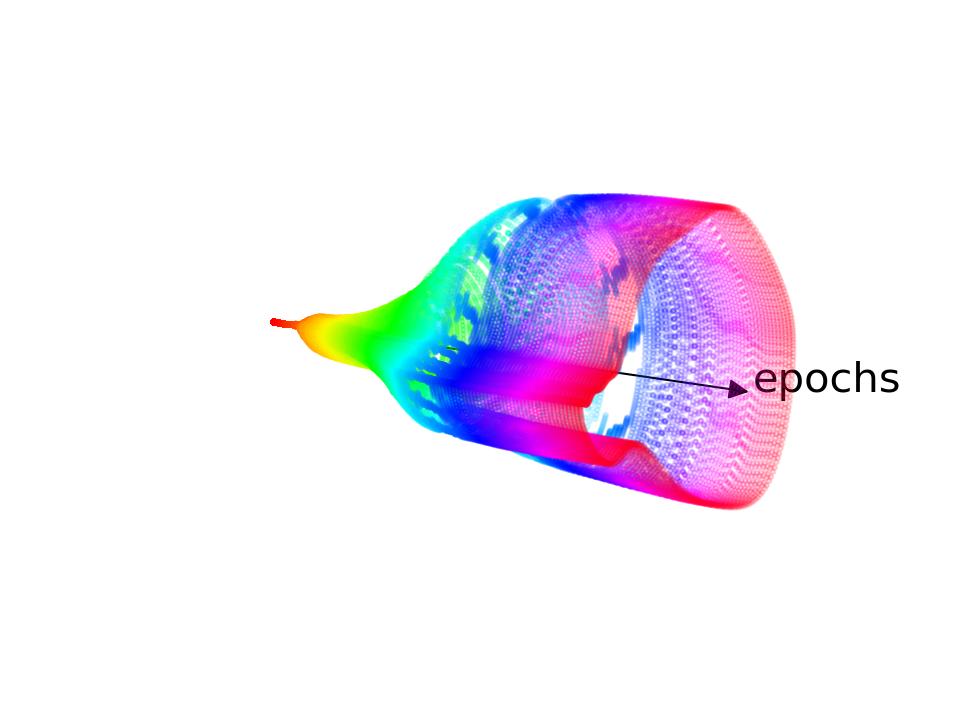} & \includegraphics[width=0.2\textwidth]{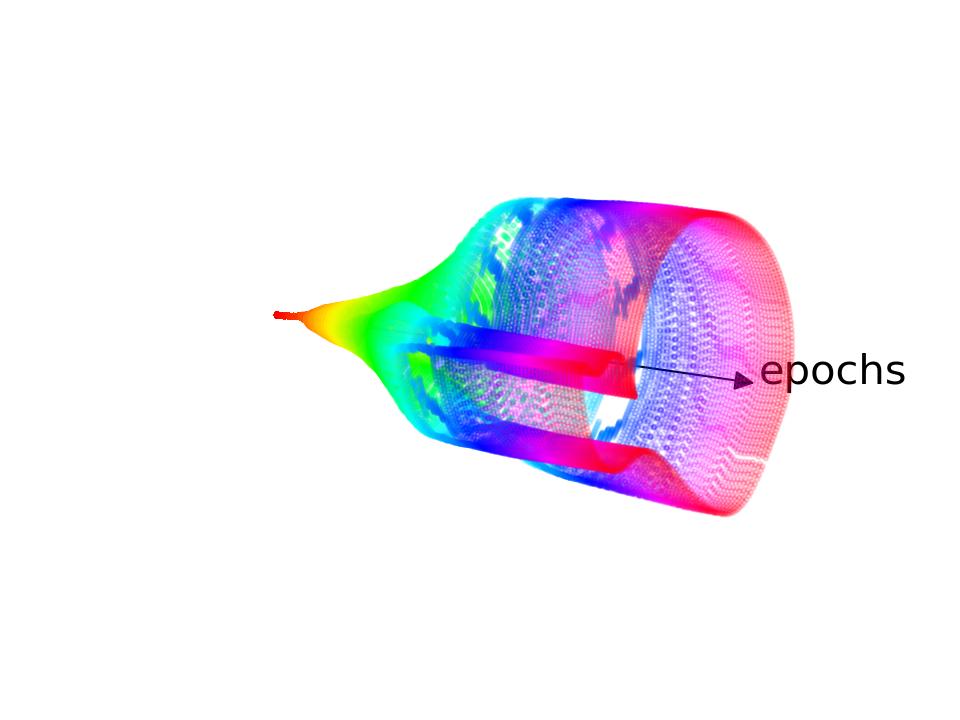}\\
    $d$ & \includegraphics[width=0.2\textwidth]{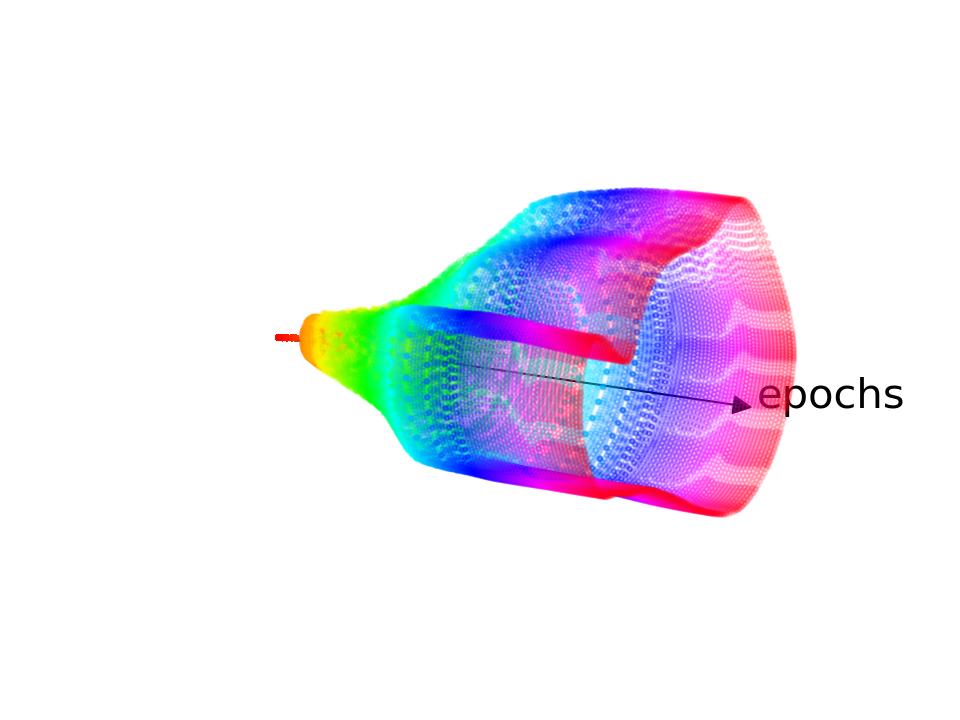} & \includegraphics[width=0.2\textwidth]{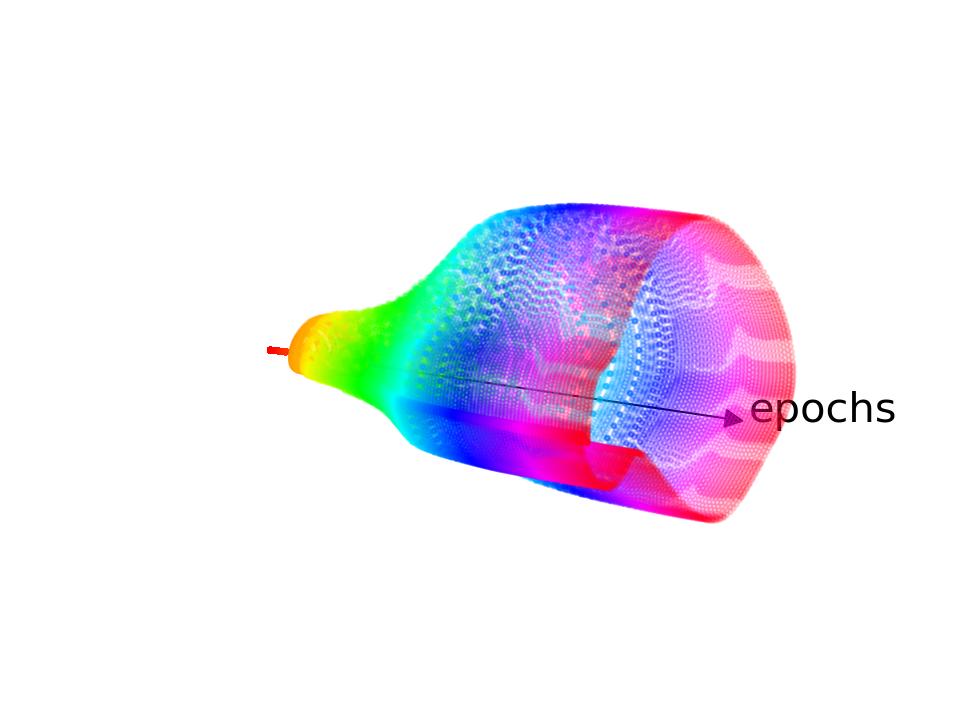} &\includegraphics[width=0.2\textwidth]{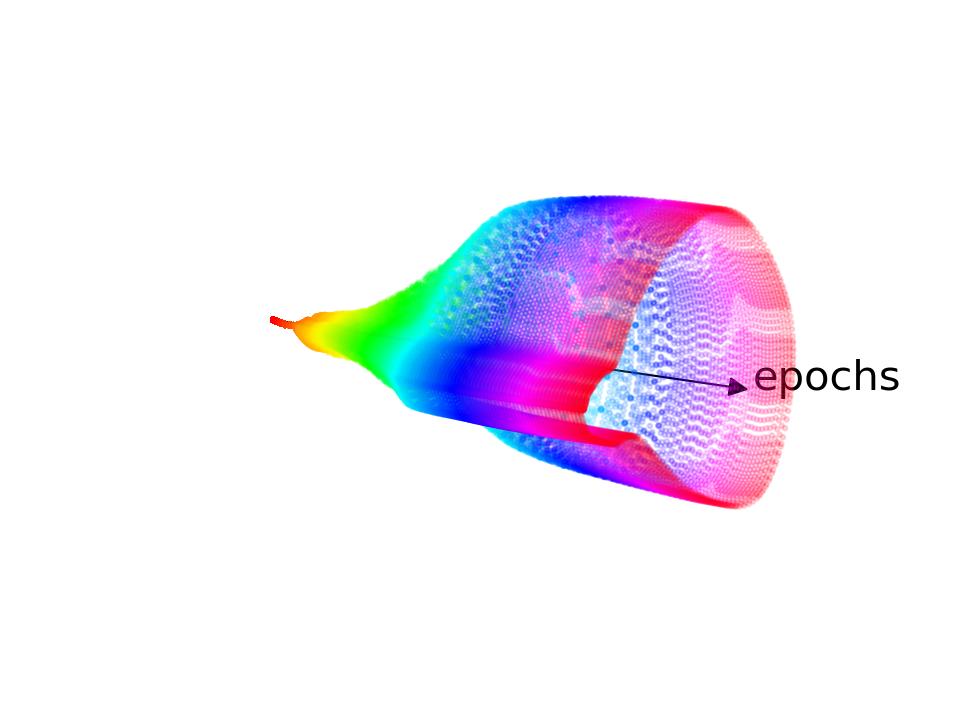} & \includegraphics[width=0.2\textwidth]{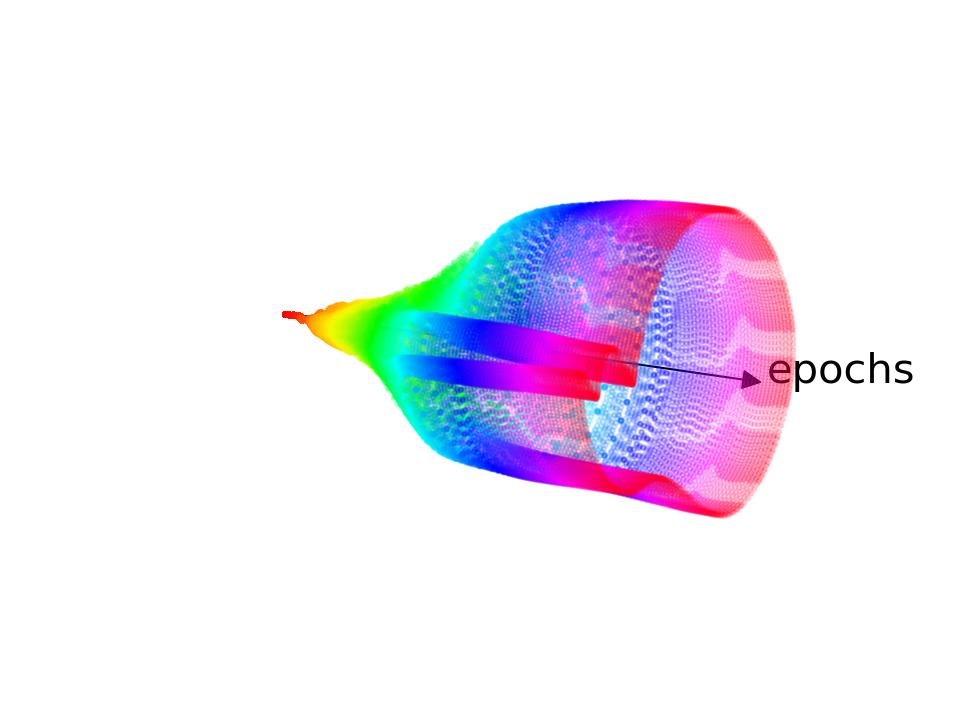}\\
  \end{tabular}
  \caption{\textbf{The LSTM learns to classify sequences.} Bifurcation diagram for the \textit{LSTM} model in \textit{symbol classification task} for sequences of length 100.  It shows the steady-state of the output $y_t$ and its first difference $y_t - y_{t-1}$.  The arrows point towards the evolution of the number of epochs, that vary from 0 to 400.}
  \label{fig:lstm-bifurcation-all}
\end{figure*}

\begin{figure*}[htp]
  \centering
   \begin{tabular}{lcccc}
    & $\{p, p\}$ & $\{p, q\}$ & $\{q, p\}$  & $\{q, q\}$ \\
    $p$ & \includegraphics[width=0.2\textwidth]{./img/temporal-order-orbit-ornn/orbit-input0-output0-sequence0} & \includegraphics[width=0.2\textwidth]{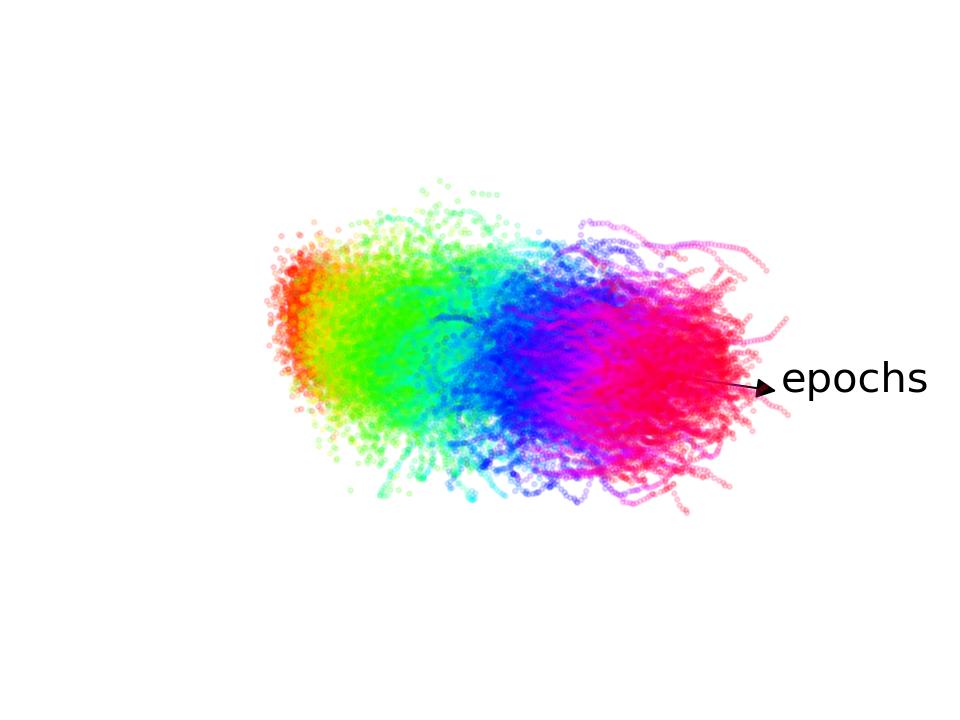} &\includegraphics[width=0.2\textwidth]{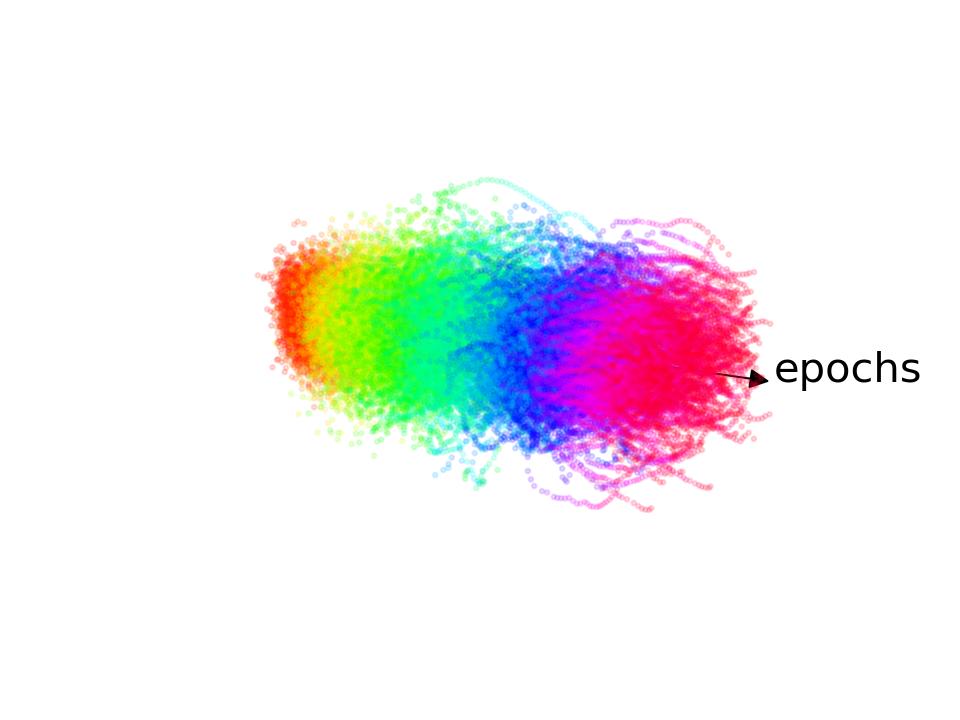} & \includegraphics[width=0.2\textwidth]{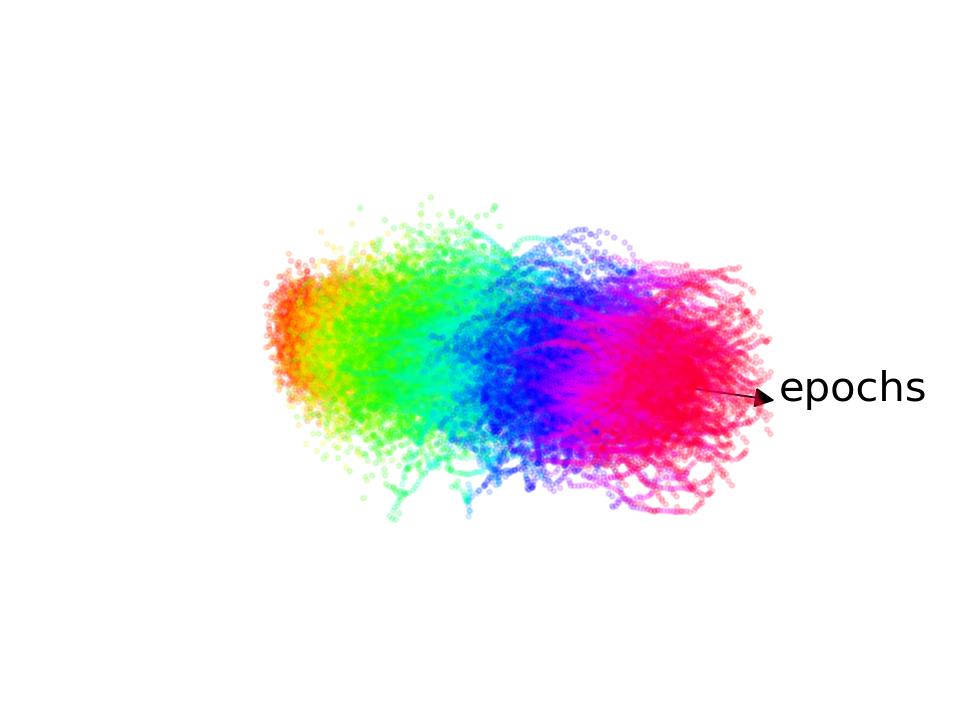}\\
    $q$ & \includegraphics[width=0.2\textwidth]{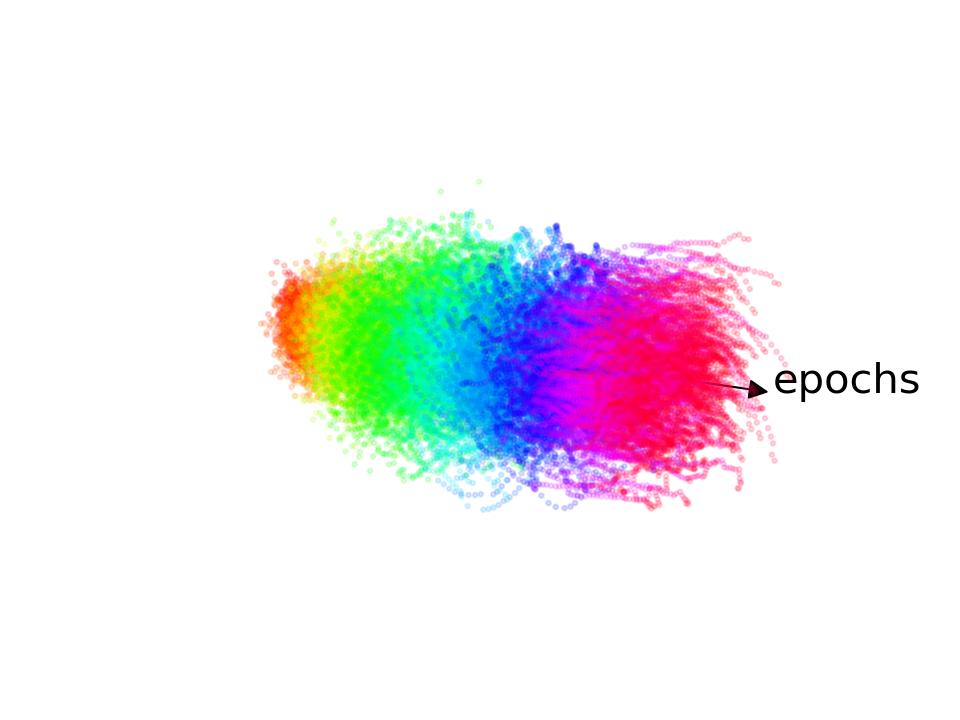} & \includegraphics[width=0.2\textwidth]{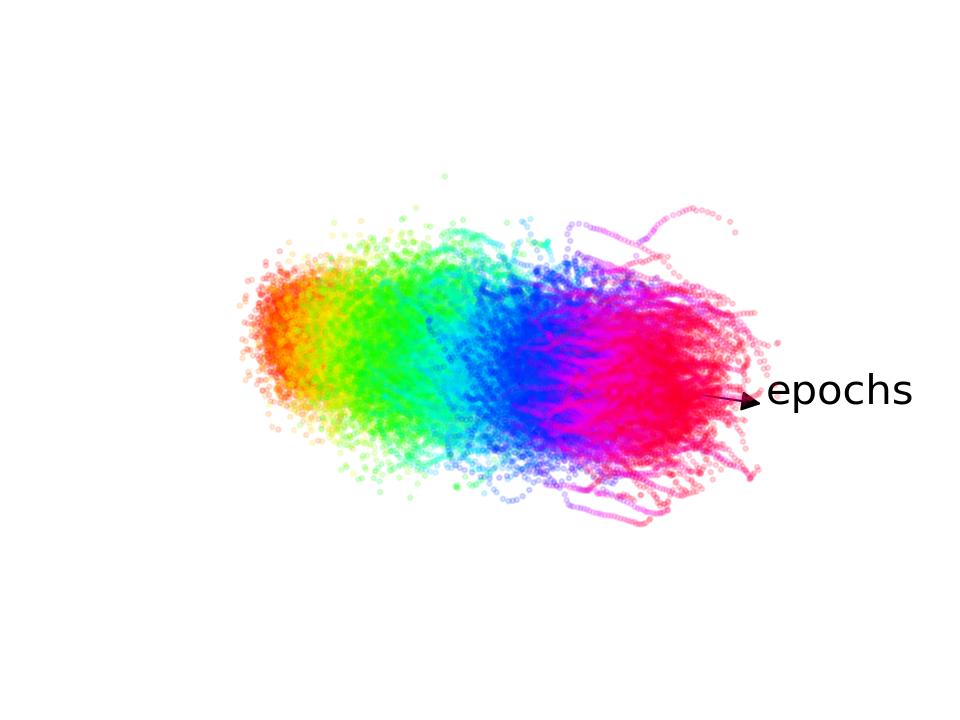} &\includegraphics[width=0.2\textwidth]{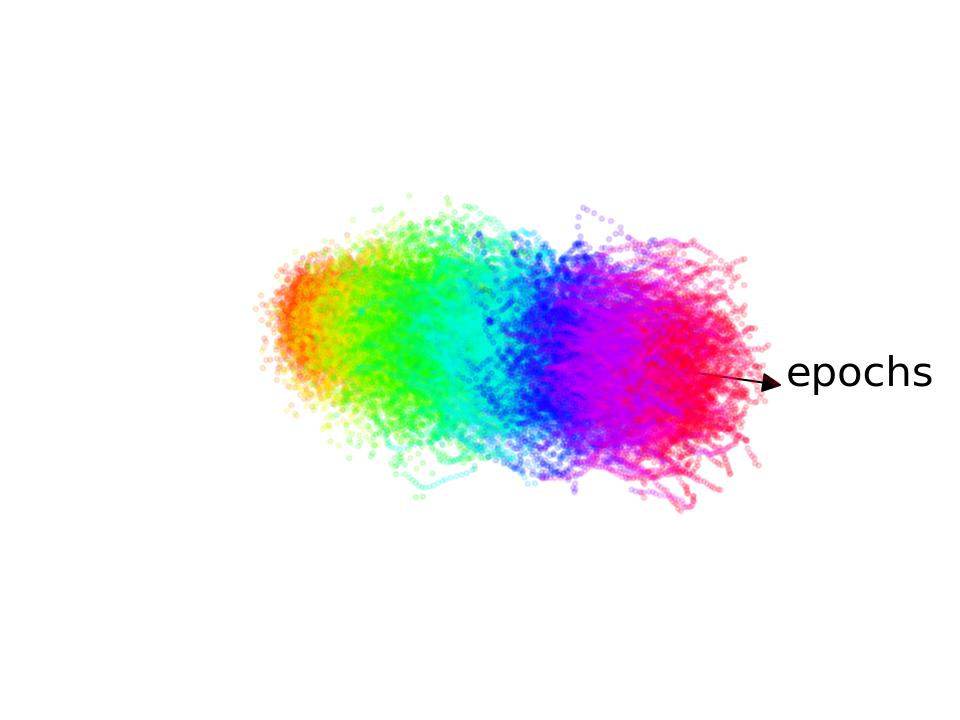} & \includegraphics[width=0.2\textwidth]{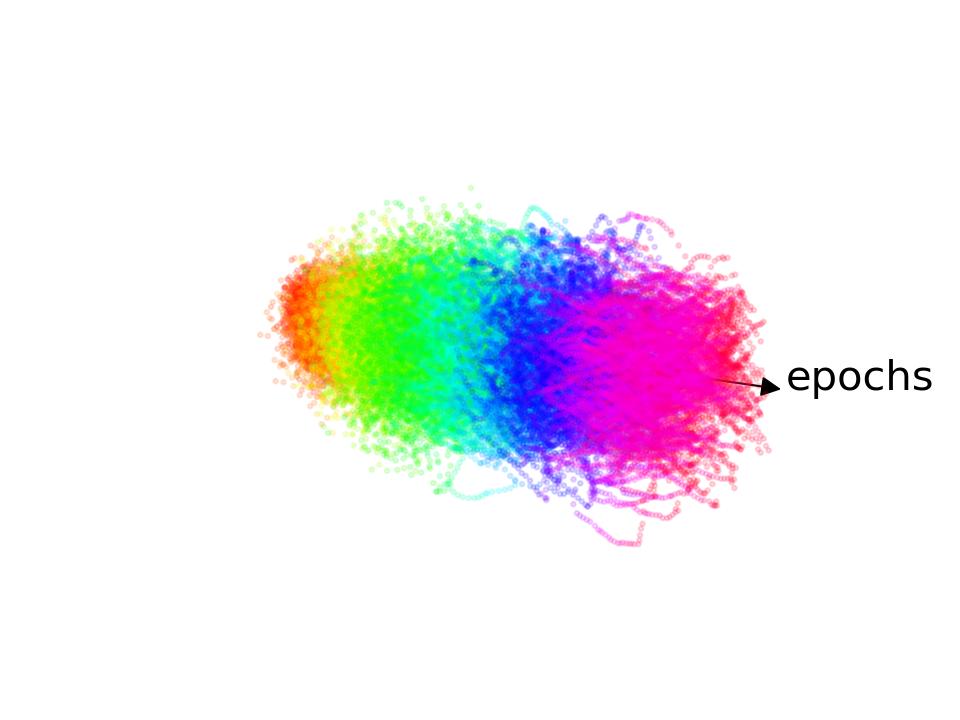}\\
    $a$ & \includegraphics[width=0.2\textwidth]{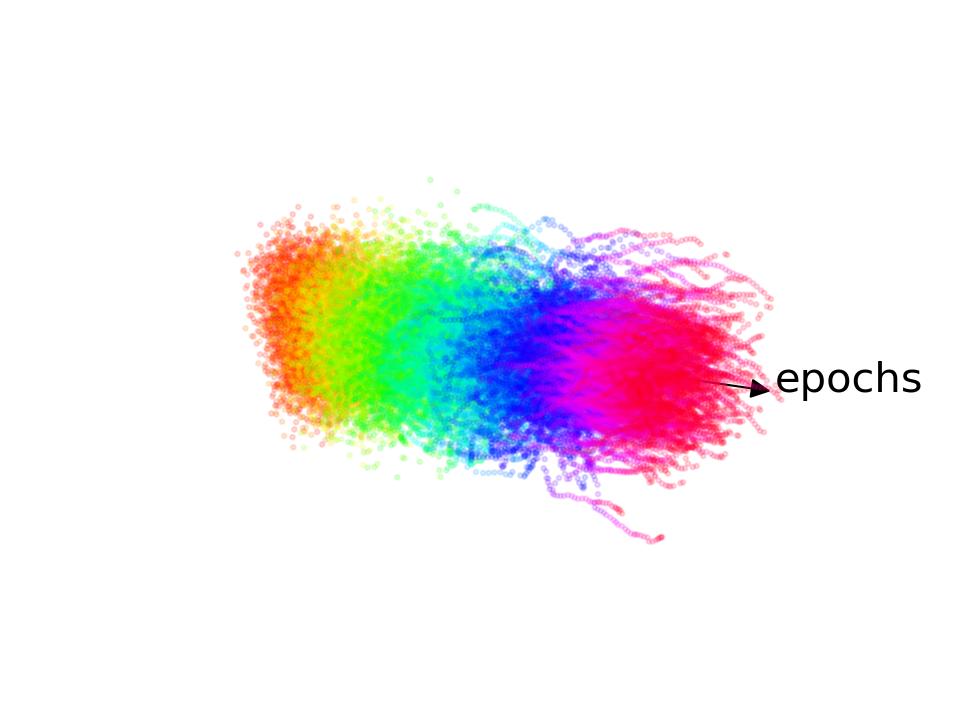} & \includegraphics[width=0.2\textwidth]{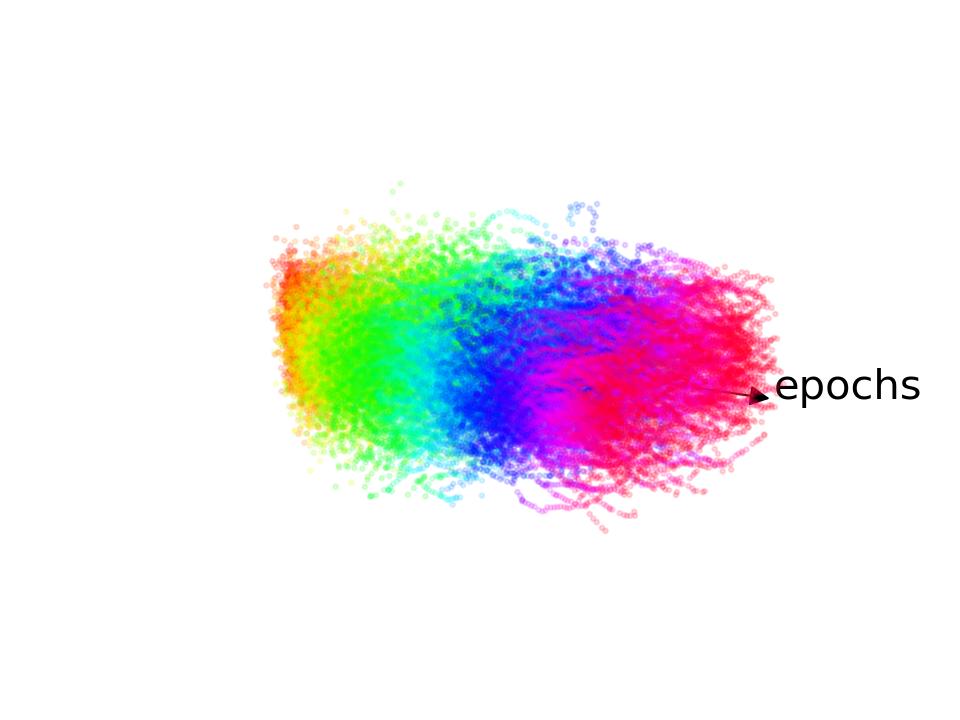} &\includegraphics[width=0.2\textwidth]{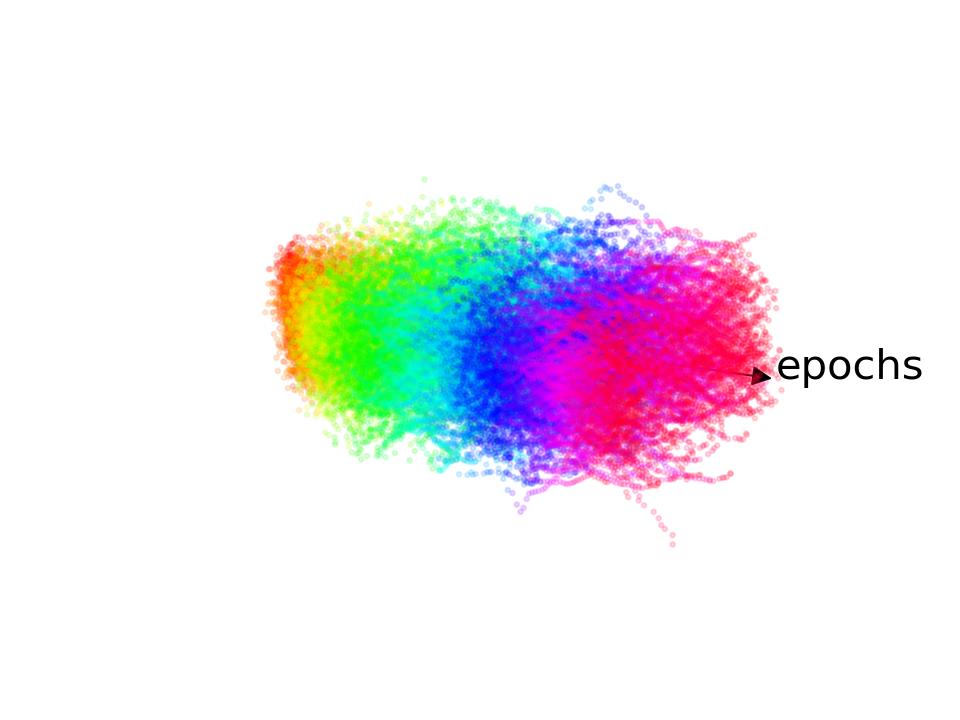} & \includegraphics[width=0.2\textwidth]{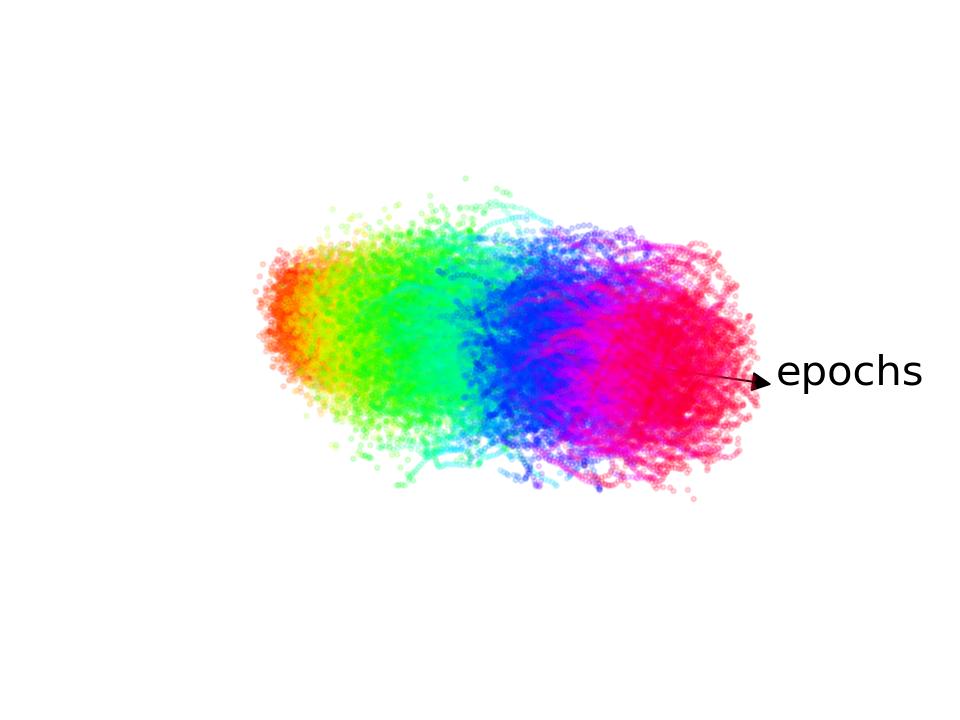}\\
    $b$ & \includegraphics[width=0.2\textwidth]{./img/temporal-order-orbit-ornn/orbit-input3-output0-sequence0} & \includegraphics[width=0.2\textwidth]{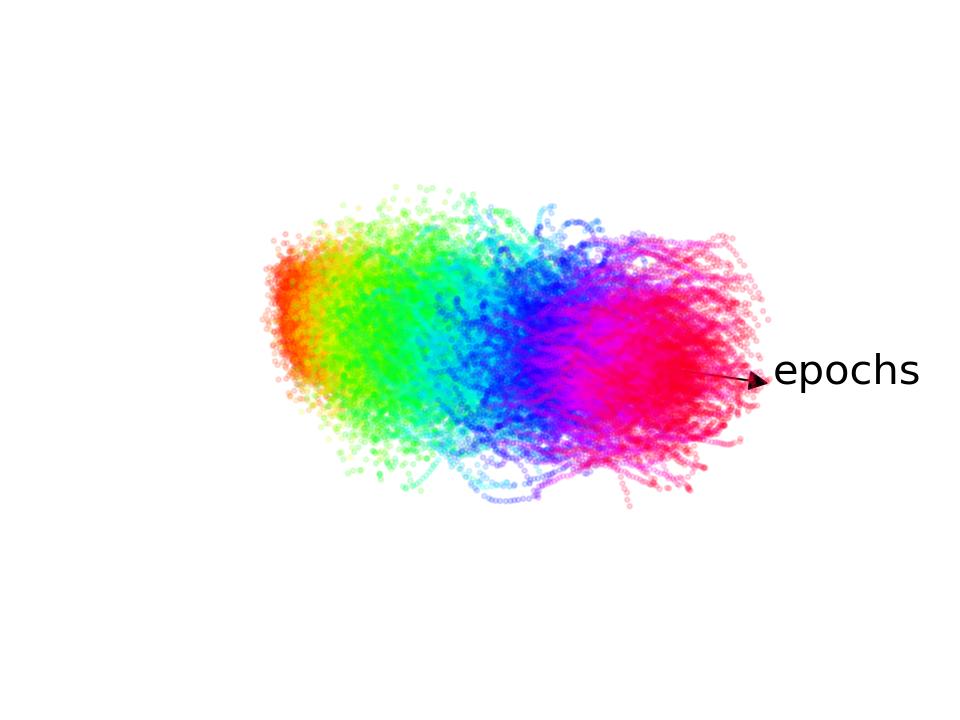} &\includegraphics[width=0.2\textwidth]{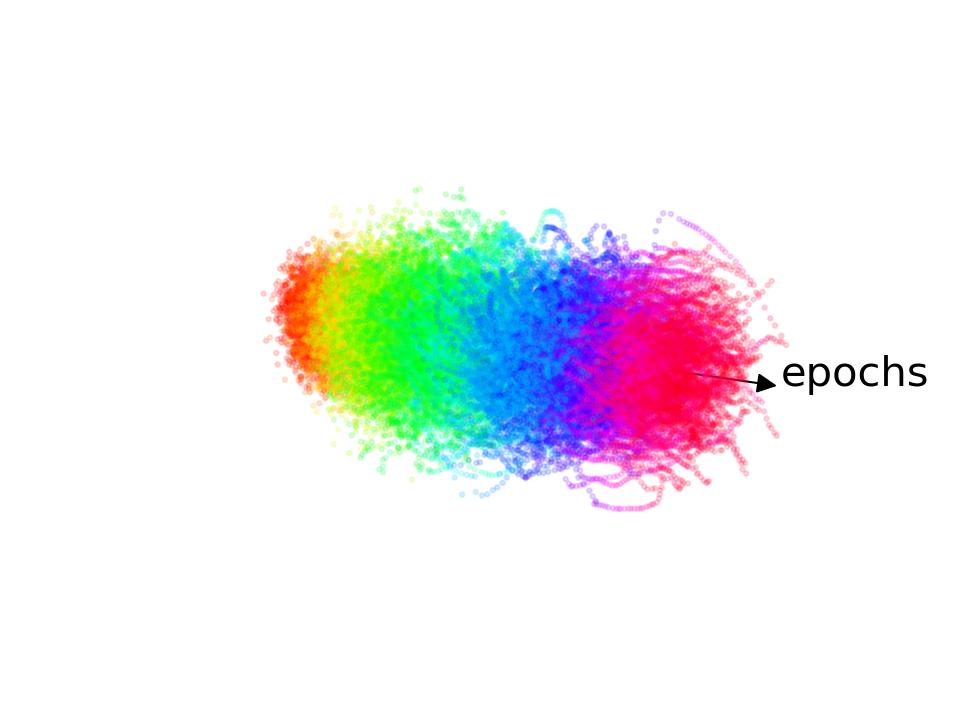} & \includegraphics[width=0.2\textwidth]{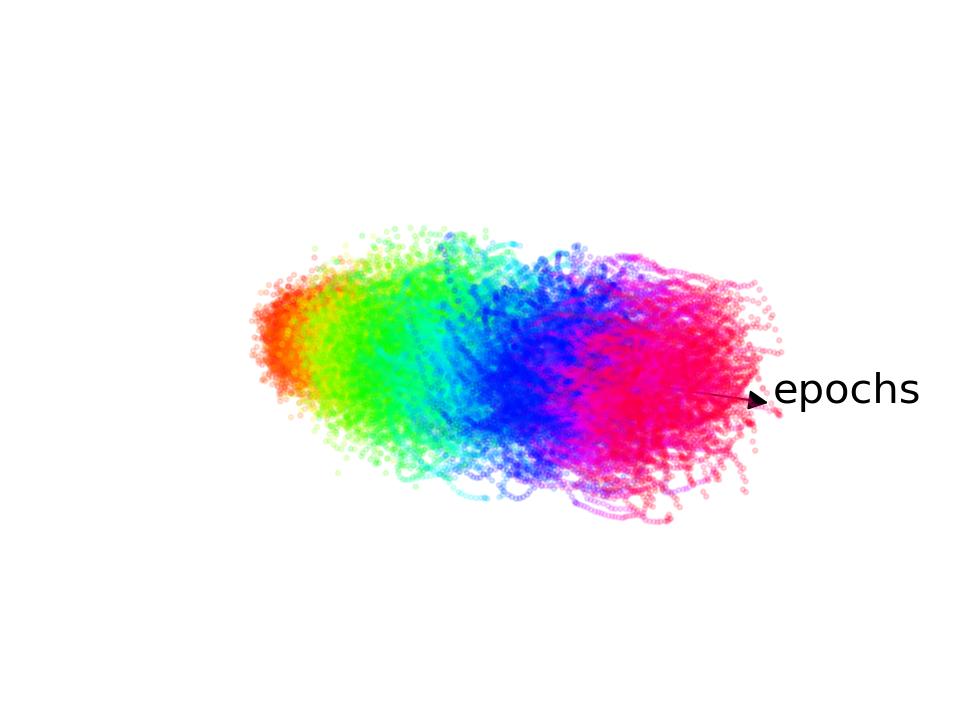}\\
    $c$ & \includegraphics[width=0.2\textwidth]{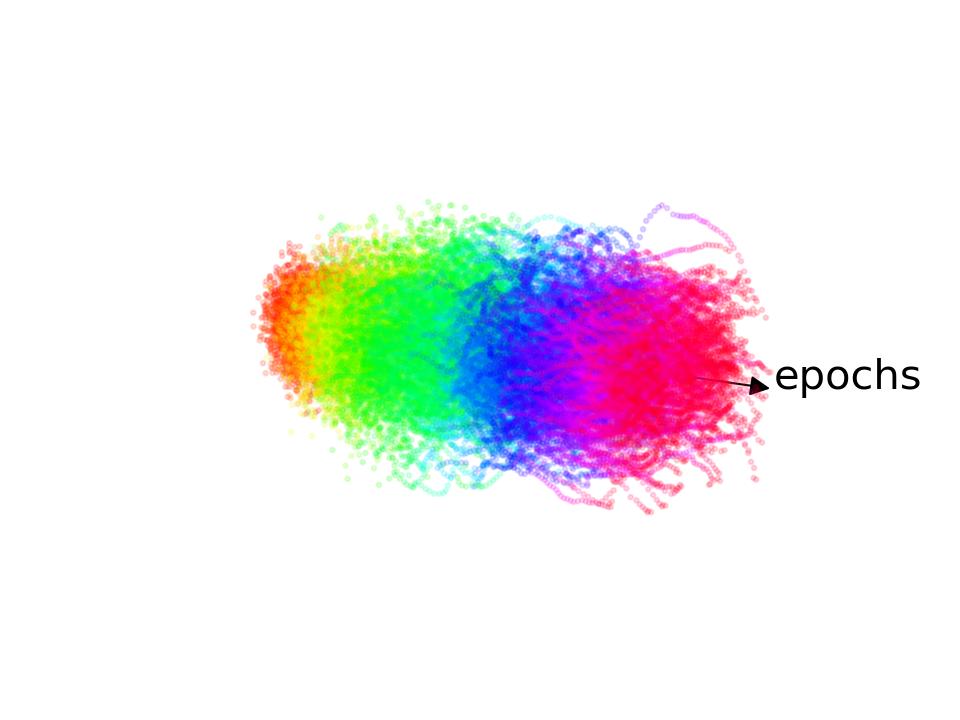} & \includegraphics[width=0.2\textwidth]{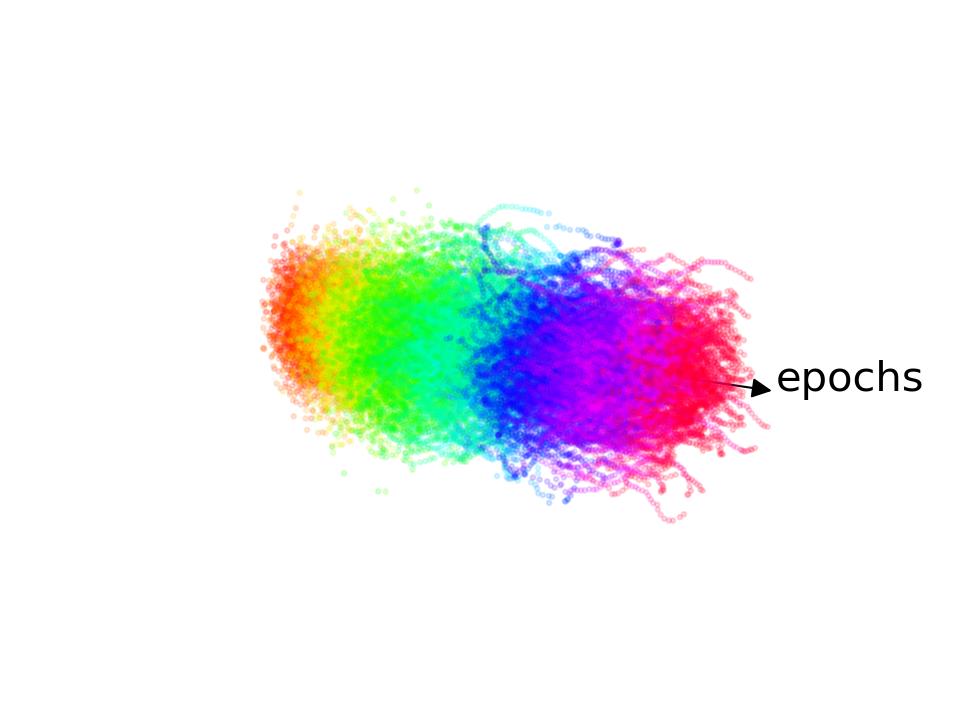} &\includegraphics[width=0.2\textwidth]{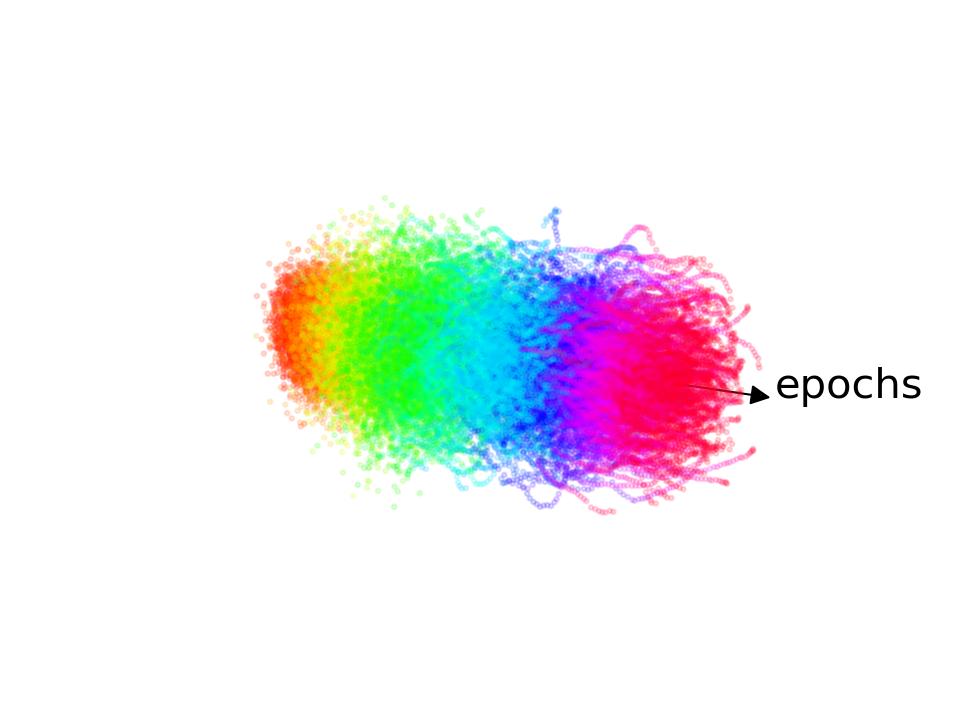} & \includegraphics[width=0.2\textwidth]{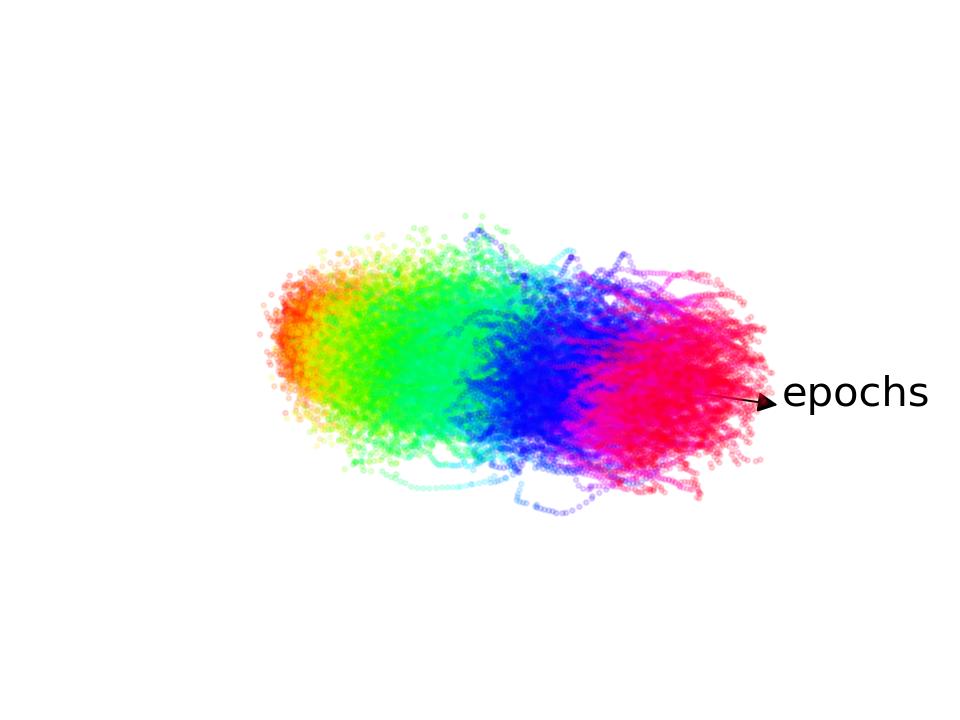}\\
    $d$ & \includegraphics[width=0.2\textwidth]{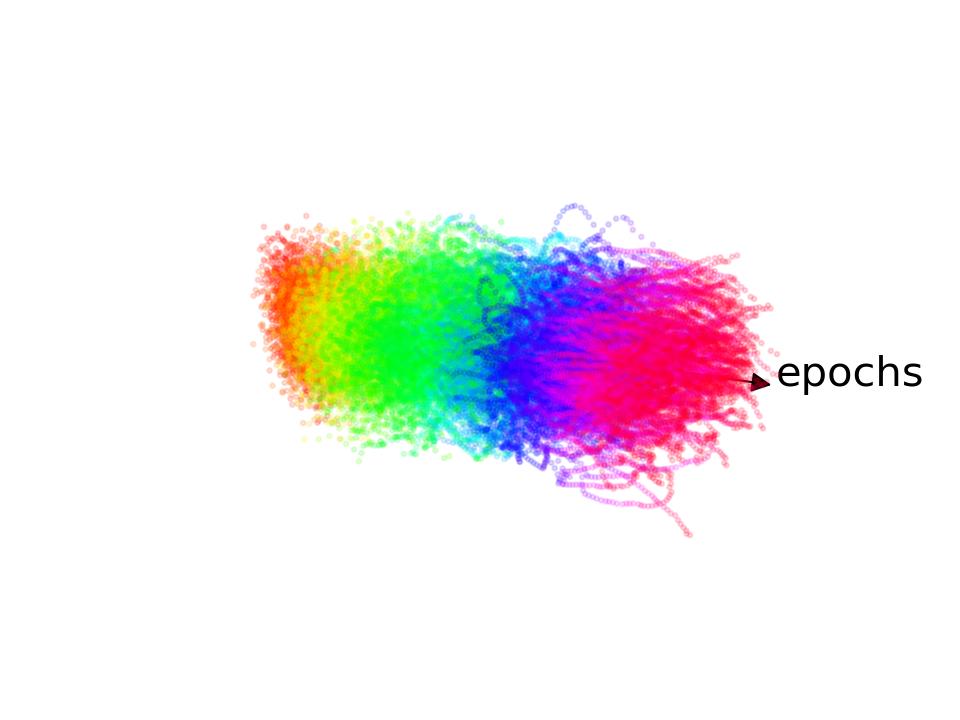} & \includegraphics[width=0.2\textwidth]{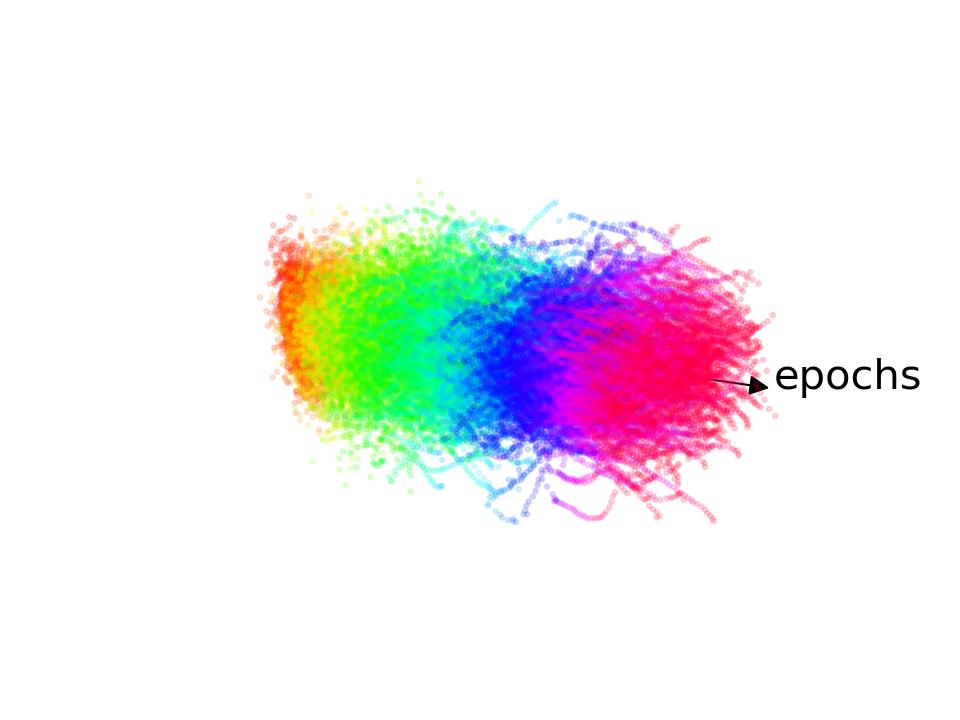} &\includegraphics[width=0.2\textwidth]{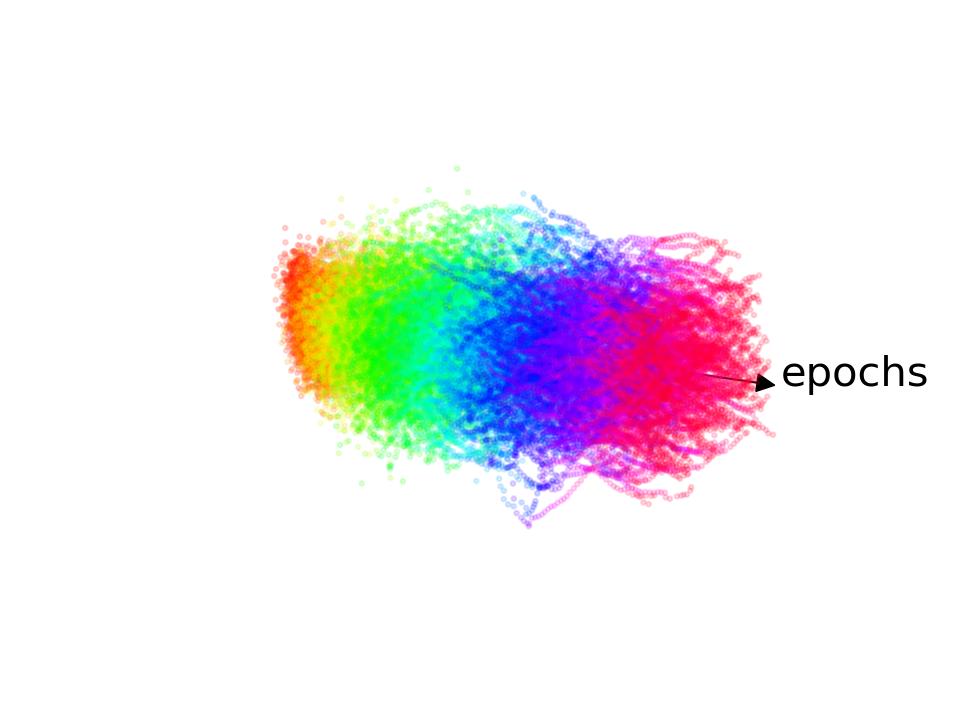} & \includegraphics[width=0.2\textwidth]{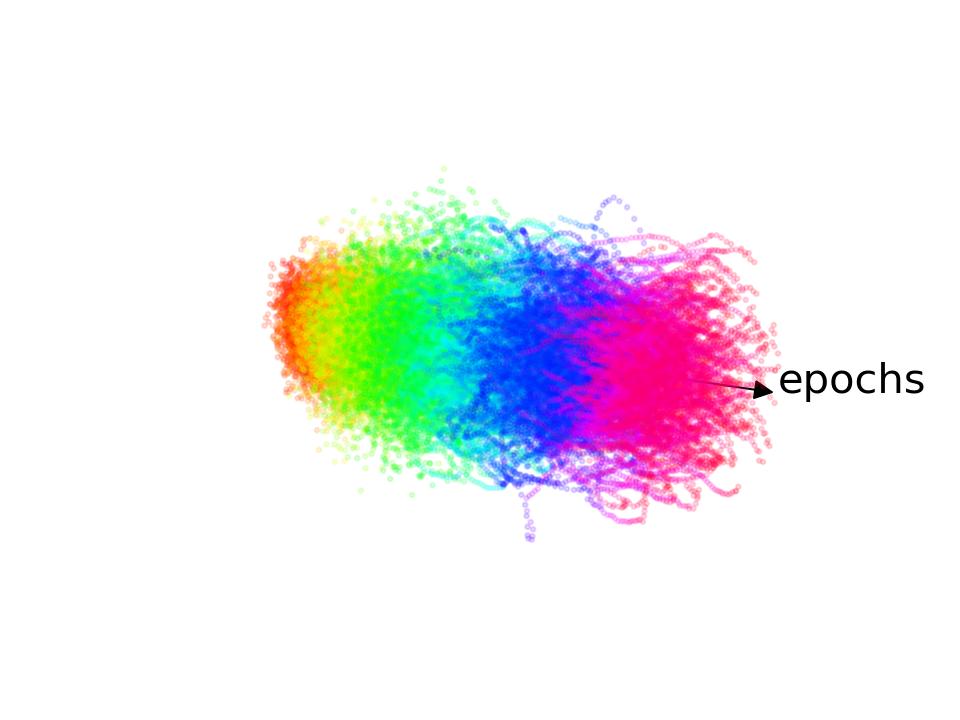}\\
  \end{tabular}
  \caption{\textbf{The oRNN learns to classify sequences.} Bifurcation diagram for the  \textit{oRNN}  model in \textit{symbol classification task} for sequences of length 100.  It shows the steady-state of the output $y_t$ and its first difference $y_t - y_{t-1}$.  The arrows point towards the evolution of the number of epochs, that vary from 0 to 400.}
  \label{fig:ornn-bifurcation-all}
\end{figure*}

\begin{figure*}[htp]
  \centering
  \subfloat[average]{\includegraphics[width=0.25\textwidth]{./img/word-level-lm/lstm-lm-1}}
  \subfloat[average]{\includegraphics[width=0.25\textwidth]{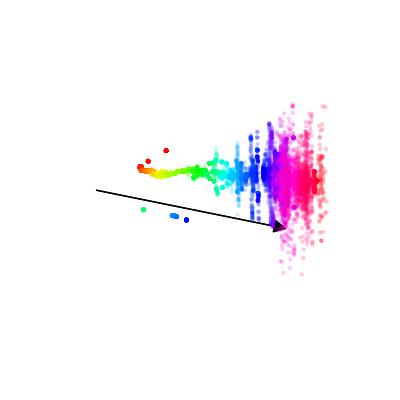}}
   \subfloat[average, with feedback]{\includegraphics[width=0.25\textwidth]{./img/word-level-lm/lstm-lm-2}}
   \subfloat[average, with feedback]{\includegraphics[width=0.25\textwidth]{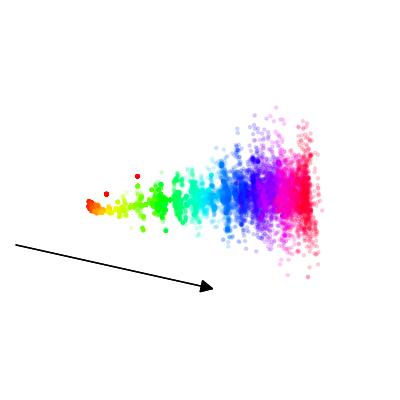}}\\
  \subfloat[``is'']{\includegraphics[width=0.25\textwidth]{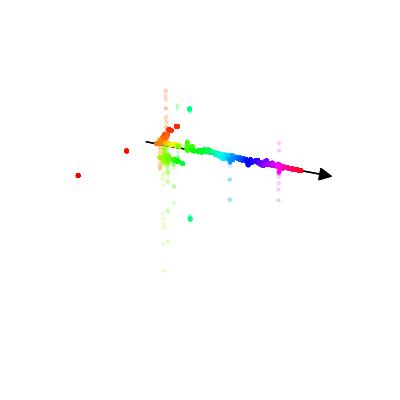}}
   \subfloat[``is'']{\includegraphics[width=0.25\textwidth]{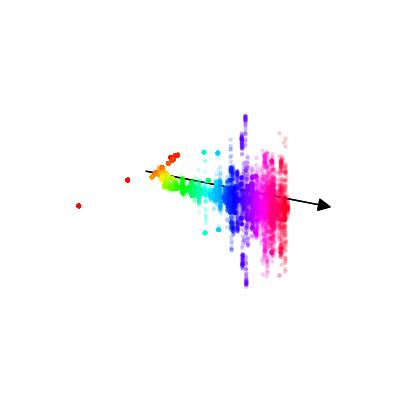}}
   \subfloat[``is'', with feedback]{\includegraphics[width=0.25\textwidth]{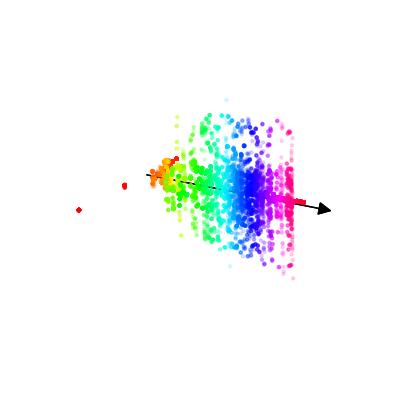}}
   \subfloat[``is'', with feedback]{\includegraphics[width=0.25\textwidth]{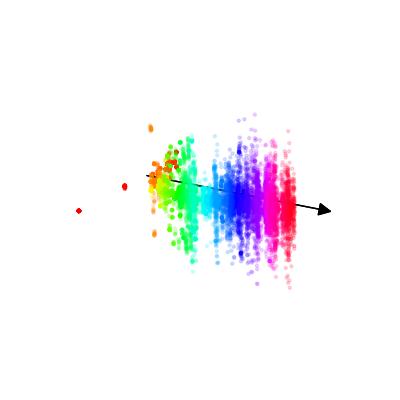}}\\
   \subfloat[``Valkyria'']{\includegraphics[width=0.25\textwidth]{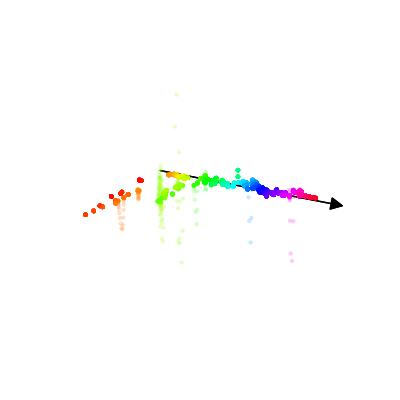}}
   \subfloat[``Valkyria'']{\includegraphics[width=0.25\textwidth]{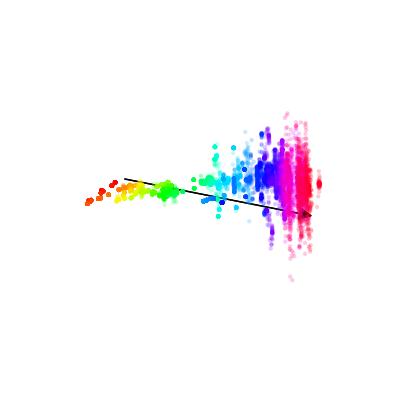}}
   \subfloat[``Valkyria'', with feedback]{\includegraphics[width=0.25\textwidth]{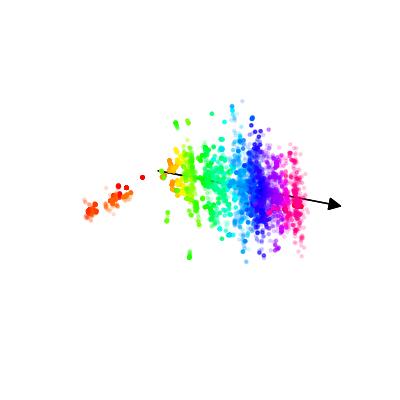}}
   \subfloat[``Valkyria'', with feedback]{\includegraphics[width=0.25\textwidth]{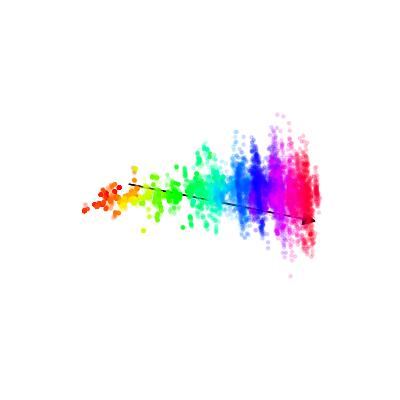}}
   \caption{\textbf{The LSTM mechanism for learning a language model.} Bifurcation diagram for the \textit{LSTM world-level language model}. For each epoch, the plot shows values visited by the projections of the internal state $p(\mathbf{x}_t)$ and its first difference $p(\mathbf{x}_t) - p(\mathbf{x}_{t-1})$  after a burnout period of 1500 samples. This burnout period  is used to remove the transient response and yields a visualization of the system attractors, per epoch. The arrow point towards the evolution of the number of epochs, that varies from 0 to 150. In (a) and (b), we have two different realizations of the bifurcation diagram obtained from constant inputs. In (c) and (d), the diagram is generated using as input the word predicted with the highest probability at the previous time instant, and using as first input to the sequence the same input as in (a) and (b), respectively. The subplots (a) to (d) use the average of internal states as projections, i.e. $p(\mathbf{x}_t) = \bar{x}_t$.    The second row, (e) to (h), and third row, (i) to (l),  show the exact same experiments but for the projections in the direction of the tokens ``is'' and ``Valkyria'', respectively.}
  \label{fig:wikitext2-lstm-orbit-diagram}
\end{figure*}

\begin{figure*}[htp]
  \centering
  \subfloat[average]{\includegraphics[width=0.25\textwidth]{./img/word-level-lm/slstm-lm-1}}
  \subfloat[``is'']{\includegraphics[width=0.25\textwidth]{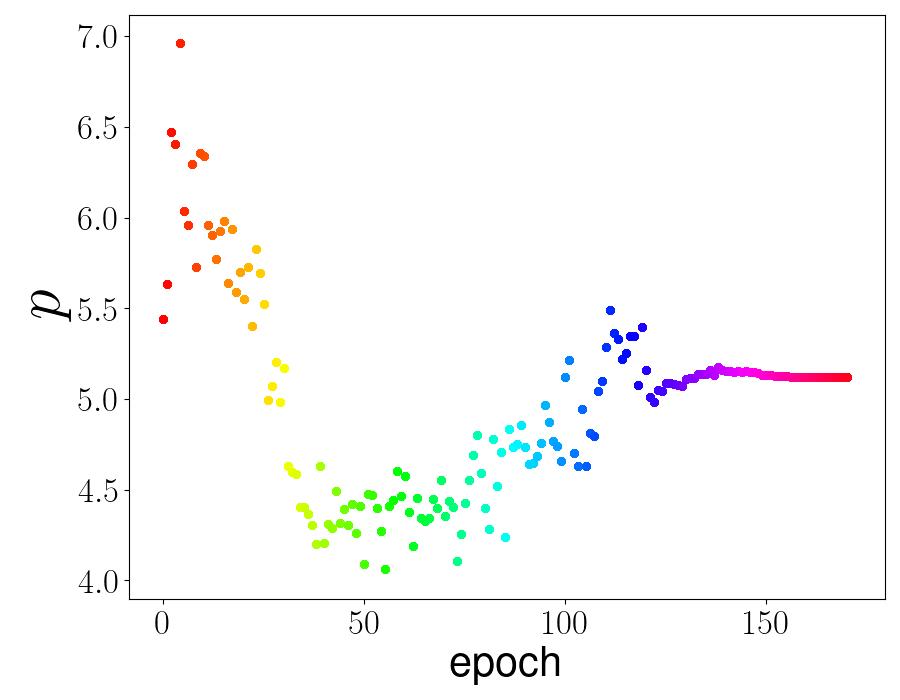}}
  \subfloat[``Valkyria'']{\includegraphics[width=0.25\textwidth]{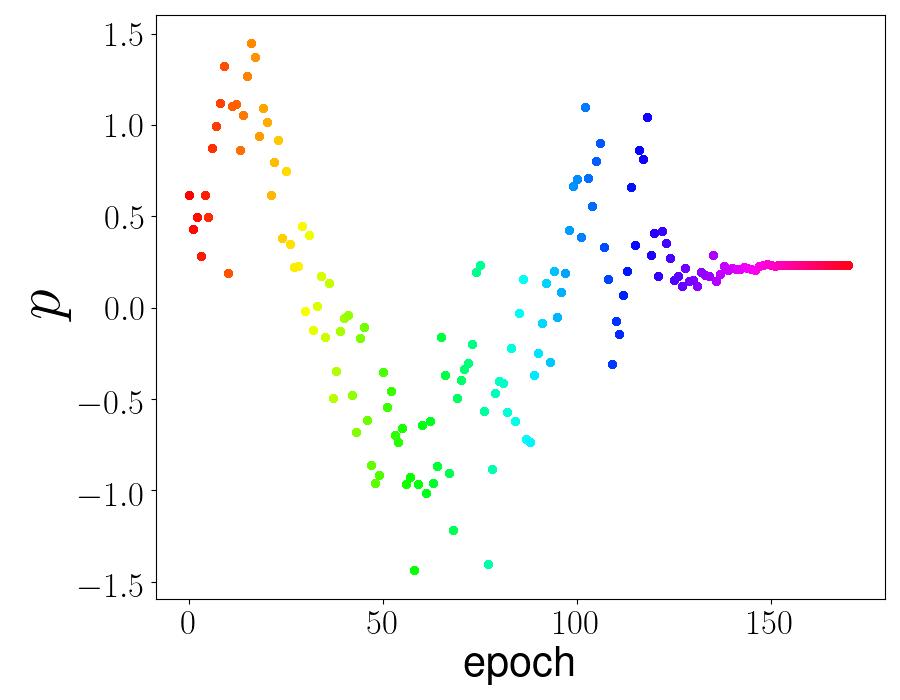}}
  \subfloat[``<unk>'']{\includegraphics[width=0.25\textwidth]{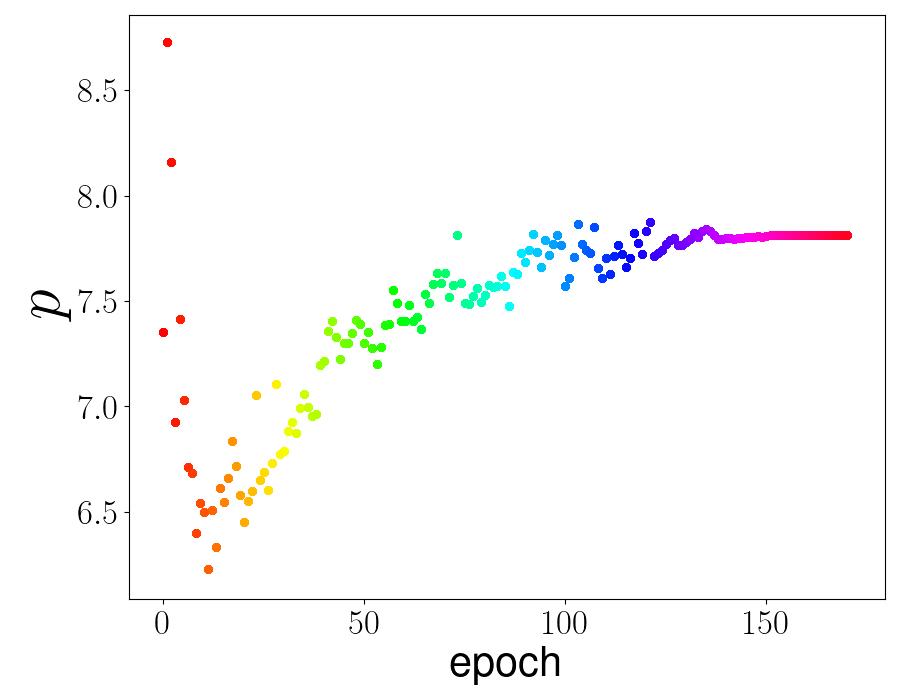}}\\
  \subfloat[average, with feedback]{\includegraphics[width=0.25\textwidth]{./img/word-level-lm/slstm-lm-5}}
  \subfloat[``is'', with feedback]{\includegraphics[width=0.25\textwidth]{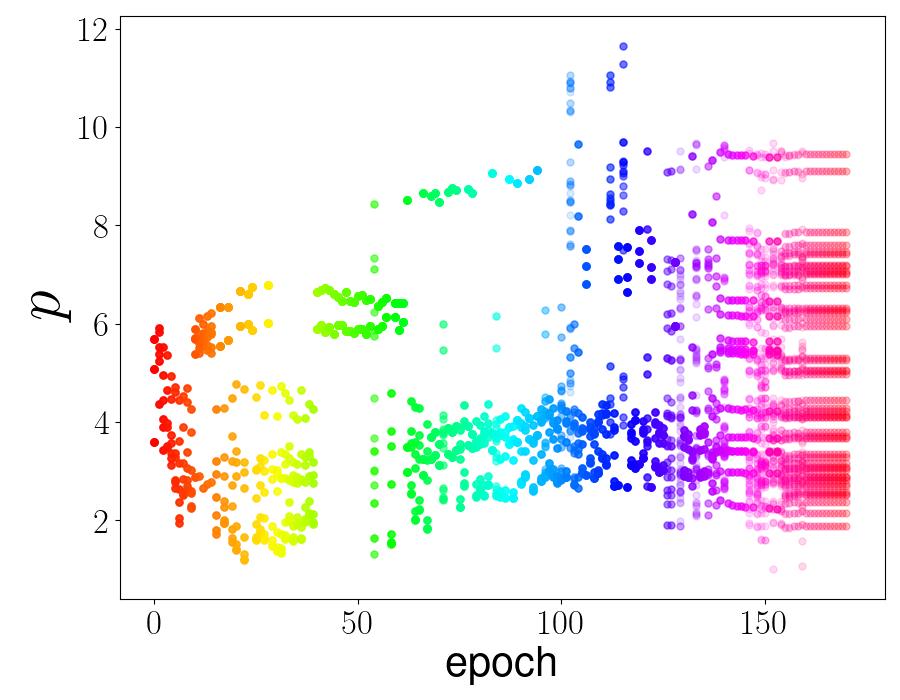}}
  \subfloat[``Valkyria'', with feedback]{\includegraphics[width=0.25\textwidth]{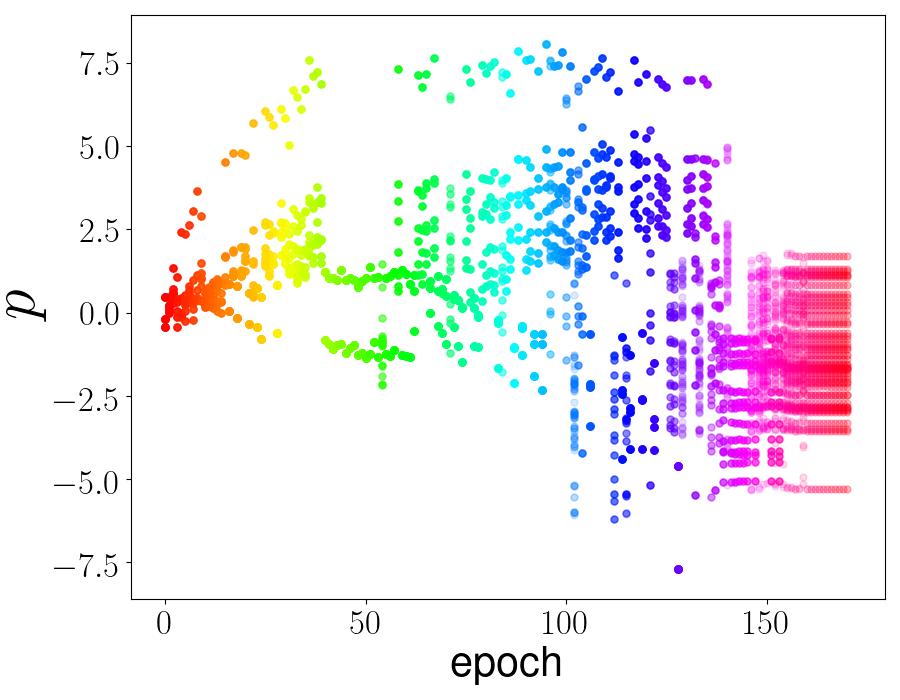}}
  \subfloat[``<unk>'', with feedback]{\includegraphics[width=0.25\textwidth]{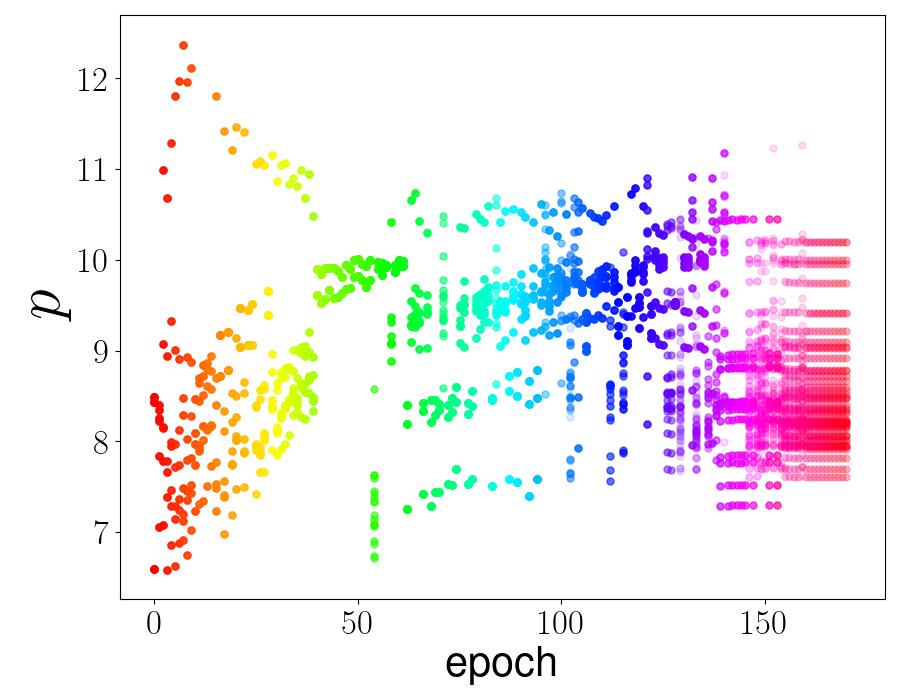}}
  \caption{\textbf{The sLSTM mechanism for learning a language model.} Bifurcation diagram for the \textit{stable LSTM world-level language model}. For each epoch, the plot shows values visited by the projections of the internal state $p(\mathbf{x}_t)$ after a burnout period of 1500 samples.  This burnout period  is used to remove the transient response and yields a visualization of the system attractors, per epoch. In the displays (a) to (d), the diagram is  generated for the same constant input. In (e) to (h), the diagram is generated using as input the word predicted with the highest probability at the previous time instant, and using as first input to the sequence the same input as in the first row. The projections are: the average of internal states as projections, i.e. $p(\mathbf{x}_t) = \bar{x}_t$; and,  projections into the direction of the tokens ``is'', ``Valkyria'' and ``<unk>''.}
  \label{fig:wikitext2-slstm-orbit-diagram}
\end{figure*}

\begin{figure*}[htp]
  \centering
  \subfloat[average]{\includegraphics[width=0.25\textwidth]{./img/word-level-lm/ornn-lm-1}}
  \subfloat[``is'']{\includegraphics[width=0.25\textwidth]{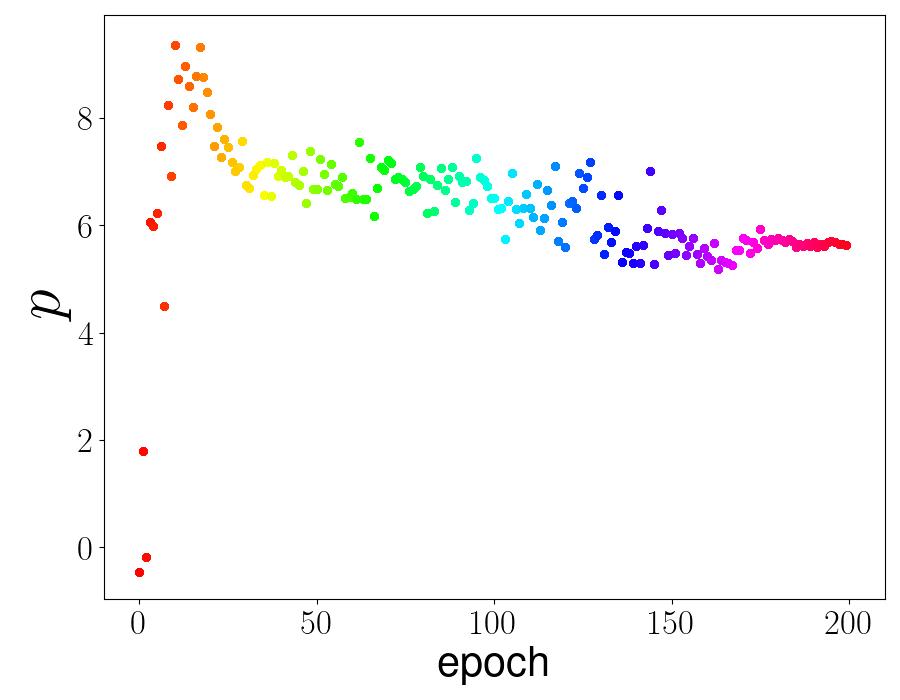}}
  \subfloat[``Valkyria'']{\includegraphics[width=0.25\textwidth]{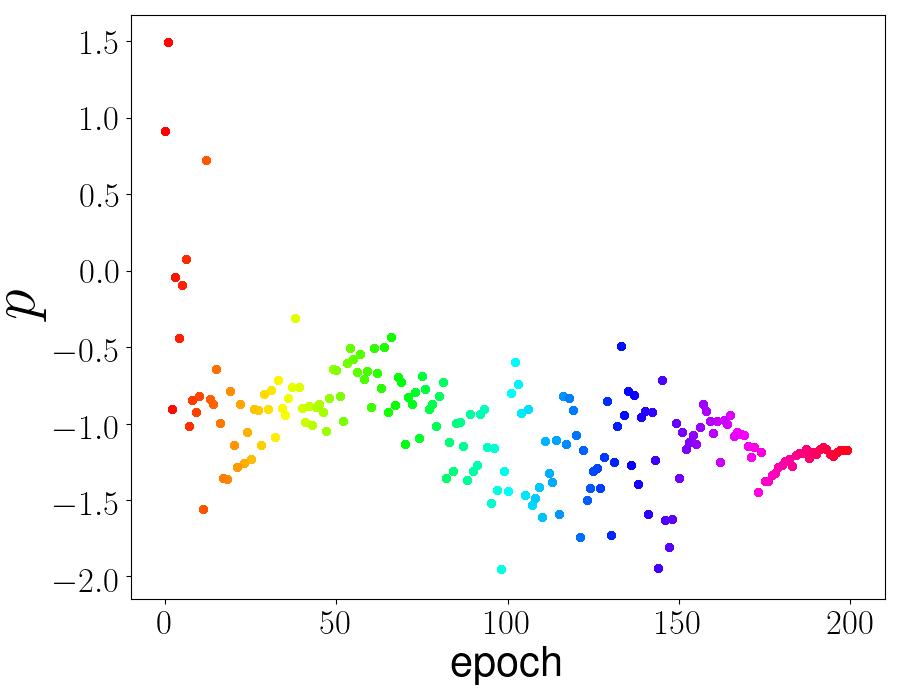}}
  \subfloat[``<unk>'']{\includegraphics[width=0.25\textwidth]{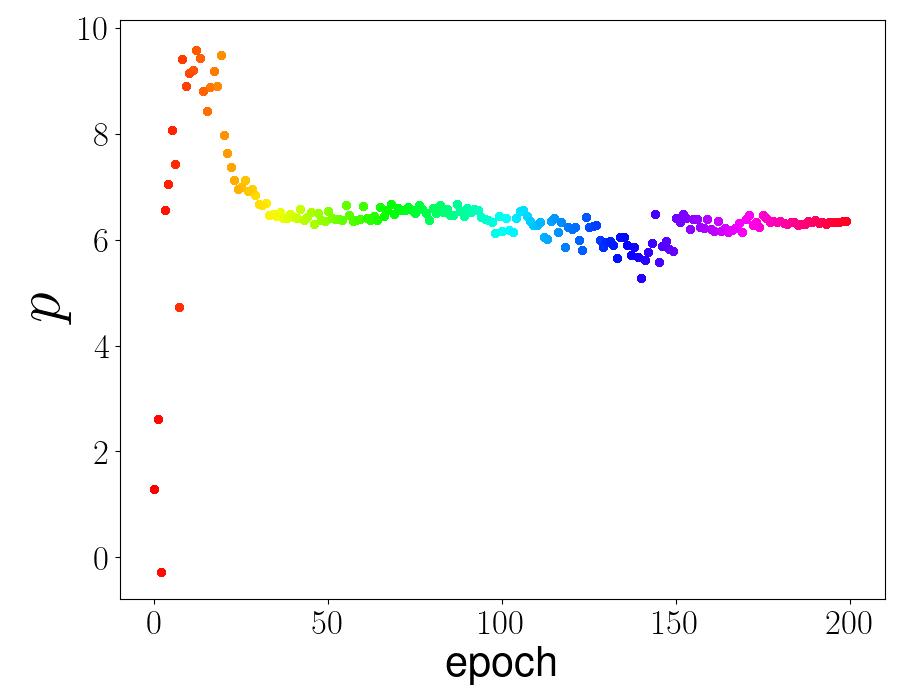}}\\
  \subfloat[average, with feedback]{\includegraphics[width=0.25\textwidth]{./img/word-level-lm/ornn-lm-5}}
  \subfloat[``is'', with feedback]{\includegraphics[width=0.25\textwidth]{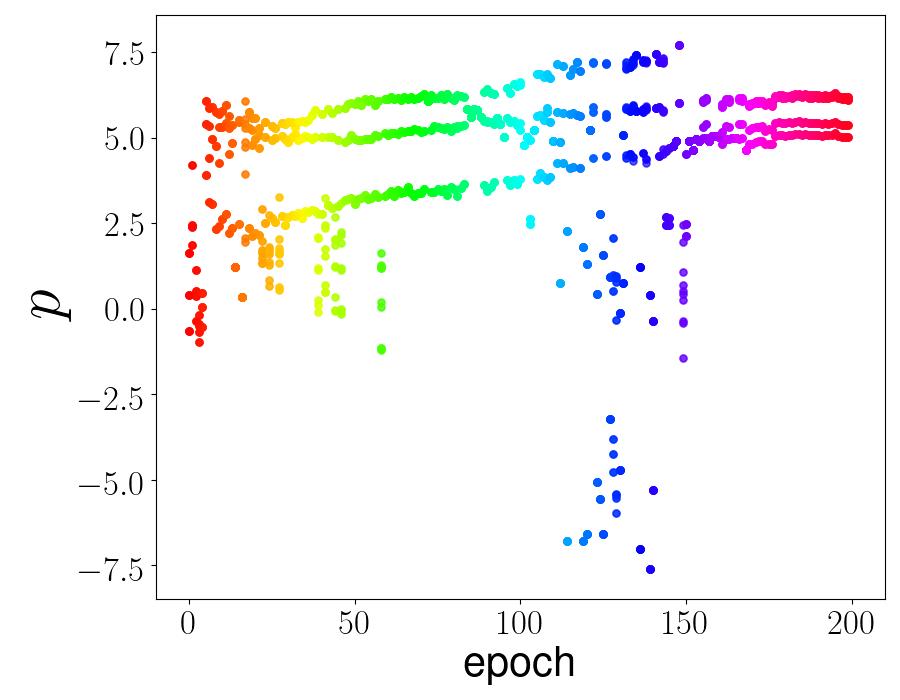}}
  \subfloat[``Valkyria'', with feedback]{\includegraphics[width=0.25\textwidth]{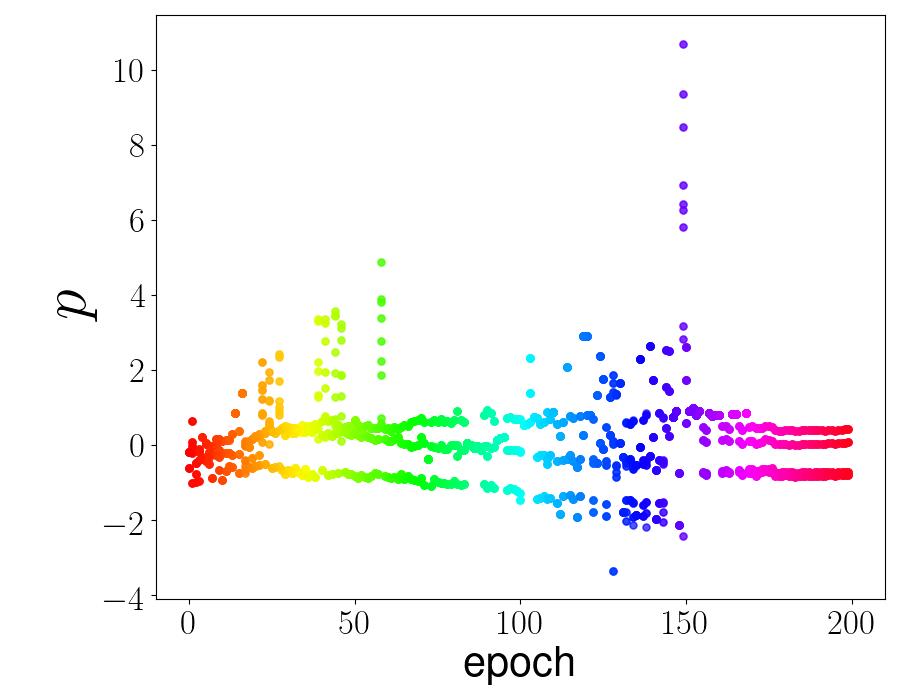}}
  \subfloat[``<unk>'', with feedback]{\includegraphics[width=0.25\textwidth]{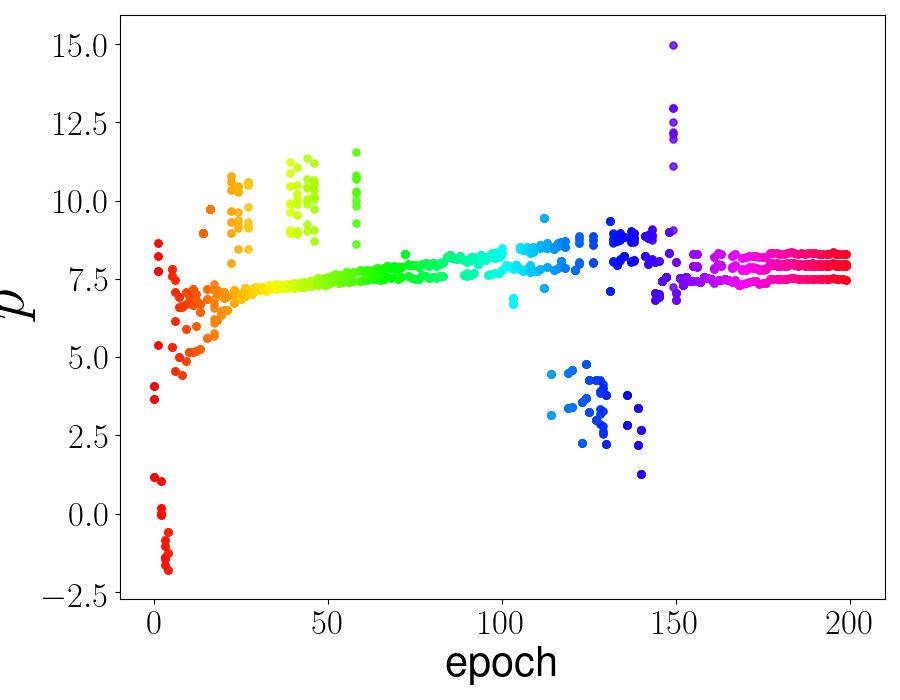}}
  \caption{\textbf{The oRNN mechanism for learning a language model.} Bifurcation diagram for the \textit{orthogonal RNN world-level language model}. For each epoch, the plot shows values visited by the projections of the internal state $p(\mathbf{x}_t)$ after a burnout period of 1500 samples.  This burnout period  is used to remove the transient response and yields a visualization of the system attractors, per epoch. In the displays (a) to (d), the diagram is  generated for the same constant input. In (e) to (h), the diagram is generated using as input the word predicted with the highest probability at the previous time instant, and using as first input to the sequence the same input as in the first row. The projections are: the average of internal states as projections, i.e. $p(\mathbf{x}_t) = \bar{x}_t$; and,  projections into the direction of the tokens ``is'', ``Valkyria'' and ``<unk>''.}
  \label{fig:wikitext2-ornn-orbit-diagram}
\end{figure*}


\subsubsection*{References}
\small

\begin{thebibliography}{}
\bibitem[Arjovsky et~al., 2016]{arjovsky_unitary_2016}
Arjovsky, M., Shah, A., and Bengio, Y. (2016).
\newblock Unitary evolution recurrent neural networks.
\newblock In {\em Proceedings of the 33rd International Conference on
  International Conference on Machine Learning (ICML)}, pages 1120--1128.

\bibitem[Bai et~al., 2018]{bai_empirical_2018}
Bai, S., Kolter, J.~Z., and Koltun, V. (2018).
\newblock An {{Empirical Evaluation}} of {{Generic Convolutional}} and
  {{Recurrent Networks}} for {{Sequence Modeling}}.
\newblock page~14.

\bibitem[Bengio et~al., 1993]{bengio_problem_1993}
Bengio, Y., Frasconi, P., and Simard, P. (1993).
\newblock The problem of learning long-term dependencies in recurrent networks.
\newblock In {\em {{IEEE International Conference}} on {{Neural Networks}}},
  pages 1183--1188.

\bibitem[Bengio et~al., 1994]{bengio_learning_1994}
Bengio, Y., Simard, P., and Frasconi, P. (1994).
\newblock Learning long-term dependencies with gradient descent is difficult.
\newblock {\em IEEE Transactions on Neural Networks}, 5(2):157--166.

\bibitem[Chandar et~al., 2019]{chandar_nonsaturating_2019}
Chandar, S., Sankar, C., Vorontsov, E., Kahou, S.~E., and Bengio, Y. (2019).
\newblock Towards {{Non}}-{{Saturating Recurrent Units}} for {{Modelling
  Long}}-{{Term Dependencies}}.
\newblock {\em Proceedings of the AAAI Conference on Artificial Intelligence},
  pages 3280--3287.

\bibitem[Cho et~al., 2014]{cho_properties_2014}
Cho, K., {van Merrienboer}, B., Bahdanau, D., and Bengio, Y. (2014).
\newblock On the {{Properties}} of {{Neural Machine Translation}}:
  {{Encoder}}-{{Decoder Approaches}}.
\newblock {\em arXiv:1409.1259}.

\bibitem[Dauphin et~al., 2017]{dauphin_language_2017}
Dauphin, Y.~N., Fan, A., Auli, M., and Grangier, D. (2017).
\newblock Language modeling with gated convolutional networks.
\newblock In {\em Proceedings of the 34th {{International Conference}} on
  {{Machine Learning}} (ICML)}, pages 933--941.

\bibitem[Doya, 1993]{doya_bifurcations_1993}
Doya, K. (1993).
\newblock Bifurcations of {{Recurrent Neural Networks}} in {{Gradient Descent
  Learning}}.

\bibitem[Edouard~Grave, 2017]{Grave2017}
Edouard~Grave, Armand~Joulin, N.~U. (2017).
\newblock Improving neural language models with a continuous cache.
\newblock In {\em Proceedings of the 5th International Conference on Learning Representations (ICML)}.

\bibitem[Gehring et~al., 2017]{gehring_convolutional_2017}
Gehring, J., Auli, M., Grangier, D., and Dauphin, Y. (2017).
\newblock A {{Convolutional Encoder Model}} for {{Neural Machine Translation}}.
\newblock In {\em Proceedings of the 55th {{Annual Meeting}} of the
  {{Association}} for {{Computational Linguistics}}}, pages 123--135.

\bibitem[Gong et~al., 2018]{Gong2018}
Gong, C., He, D., Tan, X., Qin, T., Wang, L., and Liu, T.-Y. (2018).
\newblock Frage: Frequency-agnostic word representation.
\newblock In {\em Advances in Neural Information Processing Systems 31}.

\bibitem[Helfrich et~al., 2018]{helfrich_orthogonal_2018}
Helfrich, K., Willmott, D., and Ye, Q. (2018).
\newblock Orthogonal recurrent neural networks with scaled {{Cayley}}
  transform.
\newblock In {\em Proceedings of the 35th {{International Conference}} on
  {{Machine Learning}} (ICML)}, pages 1969--1978.

\bibitem[Hochreiter and Schmidhuber, 1997]{hochreiter_long_1997}
Hochreiter, S. and Schmidhuber, J. (1997).
\newblock Long short-term memory.
\newblock {\em Neural Computation}, 9(8):1735--1780.

\bibitem[Jing et~al., 2017]{jing_tunable_2017}
Jing, L., Shen, Y., Dubcek, T., Peurifoy, J., Skirlo, S., LeCun, Y., Tegmark,
  M., and Soljacic, M. (2017).
\newblock Tunable {{Efficient Unitary Neural Networks}} ({{EUNN}}) and their
  application to {{RNNs}}.
\newblock {\em Proceedings of the 34th International Conference on Machine
  Learning (ICML)}.

\bibitem[Kalchbrenner et~al., 2016]{kalchbrenner_neural_2016}
Kalchbrenner, N., Espeholt, L., Simonyan, K., van~den Oord, A., Graves, A., and
  Kavukcuoglu, K. (2016).
\newblock Neural {{Machine Translation}} in {{Linear Time}}.
\newblock {\em arXiv:1610.10099}.

\bibitem[Kanuparthi et~al., 2019]{kanuparthi_hdetach_2019}
Kanuparthi, B., Arpit, D., Kerg, G., Ke, N.~R., Mitliagkas, I., and Bengio, Y.
  (2019).
\newblock H-{{DETACH}}: {{Modifying}} the {{LSTM Gradient Towards Better
  Optimization}}.
\newblock {\em Proceedings of the 7th International Conference for Learning
  Representations (ICLR)}.

\bibitem[Kerg et~al., 2019]{kerg_nonnormal_2019}
Kerg, G., Goyette, K., Touzel, M.~P., Gidel, G., Vorontsov, E., Bengio, Y., and
  Lajoie, G. (2019).
\newblock Non-normal {{Recurrent Neural Network}} ({{nnRNN}}): Learning long
  time dependencies while improving expressivity with transient dynamics.
  \newblock In {\em Advances in Neural Information Processing Systems 32}.

\bibitem[Khalil, 2002]{khalil_nonlinear_2002}
Khalil, H.~K. (2002).
\newblock {\em Nonlinear Systems}.
\newblock {Upper Saddle River}, third edition.

\bibitem[Kingma and Ba, 2014]{kingma_adam_2014}
Kingma, D.~P. and Ba, J. (2014).
\newblock Adam: {{A Method}} for {{Stochastic Optimization}}.
\newblock In {\em Proceedings of the 3rd {{International Conference}} for
  {{Learning Representations}} ({{ICLR}})}.

\bibitem[Laurent and {von Brecht},
  2017]{laurent_recurrent_2017}
Laurent, T. and {von Brecht}, J. (2017).
\newblock A recurrent neural network without chaos.
\newblock  In {\em Proceedings of the 5th {{International Conference}} for
  {{Learning Representations}} ({{ICLR}})}.

\bibitem[{Lezcano-Casado}, 2019]{lezcano-casado_trivializations_2019}
{Lezcano-Casado}, M. (2019).
\newblock Trivializations for {{Gradient}}-{{Based Optimization}} on
  {{Manifolds}}.
\newblock In {\em Advances in Neural Information Processing Systems 32}.

\bibitem[{Lezcano-Casado} and {Mart{\'i}nez-Rubio},
  2019]{lezcano-casado_cheap_2019}
{Lezcano-Casado}, M. and {Mart{\'i}nez-Rubio}, D. (2019).
\newblock Cheap {{Orthogonal Constraints}} in {{Neural Networks}}: {{A Simple
  Parametrization}} of the {{Orthogonal}} and {{Unitary Group}}.
\newblock In {\em Proceedings of the 36th International {{Conference}} on {{Machine Learning}} (ICML)}, pages
  3794--3803.

\bibitem[Maduranga et~al., 2019]{maduranga_complex_2019}
Maduranga, K.~D., Helfrich, K.~E., and Ye, Q. (2019).
\newblock Complex unitary recurrent neural networks using scaled cayley
  transform.
\newblock In {\em Proceedings of the {{AAAI Conference}} on {{Artificial
  Intelligence}}}, volume~33, pages 4528--4535.

\bibitem[Maheswaranathan et~al.,
  2019]{maheswaranathan_reverse_2019}
Maheswaranathan, N., Williams, A., Golub, M.~D., Ganguli, S., and Sussillo, D.
  (2019).
\newblock Reverse engineering recurrent networks for sentiment classification
  reveals line attractor dynamics.
\newblock In {\em Advances in Neural Information Processing Systems 32}, pages 15670--15679.

\bibitem[Merity et~al., 2017]{Merity2017}
Merity, S., Xiong, C., Bradbury, J., and Socher, R. (2017).
\newblock Pointer sentinel mixture models.
\newblock In {\em Proceedings of the 5th {{International Conference}} for
  {{Learning Representations}} ({{ICLR}})}.

\bibitem[Mhammedi et~al., 2017]{mhammedi_efficient_2017}
Mhammedi, Z., Hellicar, A., Rahman, A., and Bailey, J. (2017).
\newblock Efficient orthogonal parametrisation of recurrent neural networks
  using {Householder} reflections.
\newblock In {\em Proceedings of the 34th {{International Conference}} on
  {{Machine Learning}} (ICML)}, pages 2401--2409.

\bibitem[Miller and Hardt, 2019]{miller_stable_2019}
Miller, J. and Hardt, M. (2019).
\newblock Stable {{Recurrent Models}}.
\newblock  In {\em Proceedings of the 7th {{International Conference}} for
  {{Learning Representations}} ({{ICLR}})}.

\bibitem[Nesterov, 1998]{nesterov_introductory_1998}
Nesterov, Y. (1998).
\newblock {\em Introductory {{Lectures On Convex Programming}}}.
\newblock {Springer Science \& Business Media}.

\bibitem[Nocedal and Wright, 2006]{nocedal_numerical_2006}
Nocedal, J. and Wright, S.~J. (2006).
\newblock {\em Numerical Optimization}.
\newblock Springer Series in Operations Research. {Springer}, {New York}, 2nd
  edition.

\bibitem[Pascanu et~al., 2013]{pascanu_difficulty_2013}
Pascanu, R., Mikolov, T., and Bengio, Y. (2013).
\newblock On the {{Difficulty}} of {{Training Recurrent Neural Networks}}.
\newblock In {\em Proceedings of the 30th {{International Conference}} on
  {{International Conference}} on {{Machine Learning}} (ICML)}, pages 1310--1318.
va
\bibitem[Peters et~al., 2018]{peters_deep_2018}
Peters, M.~E., Neumann, M., Iyyer, M., Gardner, M., Clark, C., Lee, K., and
  Zettlemoyer, L. (2018).
\newblock Deep contextualized word representations.
\newblock In {\em Proceedings of the Conference of the North American Chapter of the Association for Computational Linguistics -  Human Language Technologies (NAACL-HLT)}.

\bibitem[Radford et~al., 2018]{radford_improving_2018}
Radford, A., Narasimhan, K., Salimans, T., and Sutskever, I. (2018).
\newblock Improving {{Language Understanding}} by {{Generative
  Pre}}-{{Training}}.

\bibitem[Radford et~al., 2019]{radford_language_2019}
Radford, A., Wu, J., Child, R., Luan, D., Amodei, D., and Sutskever, I. (2019).
\newblock Language {{Models}} are {{Unsupervised Multitask Learners}}.

\bibitem[Ribeiro et~al., 2019]{ribeiro_smoothness_2019}
Ribeiro, A.~H., Tiels, K., Umenberger, J., Sch{\"o}n, T.~B., and Aguirre, L.~A.
  (2019).
\newblock On the {{Smoothness}} of {{Nonlinear System Identification}}.
\newblock {\em Provisionally accepted at Automatica}.

\bibitem[Rudin, 1964]{rudin_principles_1964}
Rudin, W. (1964).
\newblock {\em Principles of Mathematical Analysis}.
\newblock International Series in Pure and Applied Mathematics. {McGraw-Hill}.

\bibitem[Sinai, 1959]{sinai_notion_1959}
Sinai, Y.~G. (1959).
\newblock On the notion of entropy of a dynamical system.
\newblock In {\em Dokl. {{Akad}}. {{Nauk}}. {{SSSR}}}, volume 124, page 768.

\bibitem[Spivak, 1998]{spivak_calculus_1998}
Spivak, M. (1998).
\newblock {\em Calculus on Manifolds: A Modern Approach to Classical Theorems
  of Advanced Calculus}.
\newblock Mathematics Monograph Series. {Perseus Books}, {Cambridge, Mass}.

\bibitem[Sussillo and Barak, 2013]{sussillo_opening_2013}
Sussillo, D. and Barak, O. (2013).
\newblock Opening the {{Black Box}}: {{Low}}-{{Dimensional Dynamics}} in
  {{High}}-{{Dimensional Recurrent Neural Networks}}.
\newblock {\em Neural Computation}, 25(3):626--649.

\bibitem[van~den Oord et~al., 2016]{oord_wavenet_2016}
van~den Oord, A., Dieleman, S., Zen, H., Simonyan, K., Vinyals, O., Graves, A.,
  Kalchbrenner, N., Senior, A., and Kavukcuoglu, K. (2016).
\newblock {{WaveNet}}: {{A Generative Model}} for {{Raw Audio}}.
\newblock {\em arXiv:1609.03499}.

\bibitem[Vaswani et~al., 2017]{vaswani_attention_2017}
Vaswani, A., Shazeer, N., Parmar, N., Uszkoreit, J., Jones, L., Gomez, A.~N.,
  Kaiser, {\L}., and Polosukhin, I. (2017).
\newblock Attention is {{All}} you {{Need}}.
\newblock {\em Advances in {{Neural Information Processing Systems}} 30}, pages 5998--6008.

\bibitem[Vorontsov et~al., 2017]{vorontsov_orthogonality_2017}
Vorontsov, E., Trabelsi, C., Kadoury, S., and Pal, C. (2017).
\newblock On orthogonality and learning recurrent networks with long term
  dependencies.
\newblock In {\em Proceedings of the 34th International Conference on Machine Learning (ICML)}, pages 3570--3578.

\bibitem[Williams and Zipser, 1989]{williams_experimental_1989}
Williams, R.~J. and Zipser, D. (1989).
\newblock Experimental analysis of the real-time recurrent learning algorithm.
\newblock {\em Connection Science}, 1(1):87--111.

\bibitem[Wisdom et~al., 2016]{wisdom_fullcapacity_2016}
Wisdom, S., Powers, T., Hershey, J., Le~Roux, J., and Atlas, L. (2016).
\newblock Full-{{Capacity Unitary Recurrent Neural Networks}}.
\newblock {\em Advances in {{Neural Information Processing Systems}} 29},
  pages 4880--4888.

\bibitem[Zhang et~al., 2018]{zhang_stabilizing_2018}
Zhang, J., Lei, Q., and Dhillon, I.~S. (2018).
\newblock Stabilizing {{Gradients}} for {{Deep Neural Networks}} via
  {{Efficient SVD Parameterization}}.
\newblock {\em Proceedings of the 35th International Conference on Machine
  Learning (ICML)}

\end{thebibliography}
\end{document}